\documentclass[aos]{imsart_MODIFIED}

\RequirePackage{amsthm,amsmath,amsfonts,amssymb}
\RequirePackage[numbers]{natbib}

\usepackage{hyperref}
\usepackage{url}
\usepackage[draft]{fixme}
\usepackage{xcolor}

\usepackage{tabularx, makecell, booktabs}

\usepackage[textwidth=2cm, textsize=footnotesize,disable]{todonotes}

\startlocaldefs

\usepackage[USenglish]{babel}
\usepackage{amsmath}
\usepackage{amssymb}
\usepackage{amsthm,bm}
\usepackage{thmtools}
\usepackage{mathtools}
\usepackage{verbatim}
\usepackage[shortlabels]{enumitem}
\usepackage{tikz}
\usepackage{tikz-qtree}
\usetikzlibrary{trees,positioning,shapes}
\usetikzlibrary{arrows.meta}
\usepackage{pgfplots}
\usepackage{booktabs}

\usepackage{chngcntr}
\usepackage{stmaryrd}
\usepackage{dsfont}  %
\usepackage{scalerel,stackengine}  %

\usepackage{ifthen}

\makeatletter

\newcommand{\listofappendices}{%
  \begingroup
  \renewcommand{\contentsname}{\appendixtocname}
  \let\@oldstarttoc\@starttoc
  \def\@starttoc##1{\@oldstarttoc{app}}
  \tableofcontents%
  \endgroup
}

\makeatother

\usepackage{hyperref}
\usepackage[capitalise,nameinlink]{cleveref}

\usepackage{mleftright,xparse}%

\pgfplotsset{compat=1.11}

\newlength\Origarrayrulewidth

\NewDocumentCommand\xDeclarePairedDelimiter{mmm}
{%
	\NewDocumentCommand#1{som}{%
		\IfNoValueTF{##2}
		{\IfBooleanTF{##1}{#2##3#3}{\mleft#2##3\mright#3}}
		{\mathopen{##2#2}##3\mathclose{##2#3}}%
	}%
}

\theoremstyle{plain}
\newtheorem{theorem}{Theorem}

\newtheorem{lemma}[theorem]{Lemma}
\newtheorem{proposition}[theorem]{Proposition}

\newtheorem{corollary}[theorem]{Corollary}

\theoremstyle{definition}
\newtheorem{defenv}[theorem]{Definition}

\newtheorem{remenv}[theorem]{Remark}
\newtheorem{exenv}[theorem]{Example}

\newtheorem*{assumption}{Assumption}
\makeatletter
\newcommand{\asslabel}[2]{%
  \phantomsection%
  \def\@currentlabelname{#2}%
  \label{#1}%
}
\makeatother
\newcommand{\assref}[1]{Assumption~\nameref{#1}}
\newcommand{\assrefshort}[1]{\nameref{#1}}

\theoremstyle{definition}
\newtheorem{remark}[theorem]{Remark}
\newtheorem{definition}[theorem]{Definition}
\newtheorem{example}[theorem]{Example}

\newcommand{\bbN}{\mathbb{N}}

\newcommand{\bbR}{\mathbb{R}}

\newcommand{\calL}{\mathcal{L}}

\newcommand{\calR}{\mathcal{R}}

\newcommand{\sborel}{\frk B}
\DeclareMathOperator{\spann}{span}

\newcommand{\equalDef}{\coloneqq}

\DeclareMathOperator{\id}{id}

\DeclareMathOperator*{\argmin}{argmin}

\allowdisplaybreaks

\newcommand{\Eqref}[1]{Eq.~\eqref{#1}}

\stackMath
\newcommand\rwidehat[1]{%
\savestack{\tmpbox}{\stretchto{%
  \scaleto{%
    \scalerel*[\widthof{\ensuremath{#1}}]{\kern.1pt\mathchar"0362\kern.1pt}%
    {\rule{0ex}{\textheight}}%
  }{\textheight}%
}{2.4ex}}%
\stackon[-6.9pt]{#1}{\tmpbox}%
}
\setlist[enumerate]{nosep}
\setlist[itemize]{nosep}

\makeatletter
\newcommand\ackname{Acknowledgements}
\if@titlepage
   \newenvironment{acknowledgements}{%
       \titlepage
       \null\vfil
       \@beginparpenalty\@lowpenalty
       \begin{center}%
         \bfseries \ackname
         \@endparpenalty\@M
       \end{center}}%
      {\par\vfil\null\endtitlepage}
\else
   
\fi
\makeatother

\newcommand{\scal}[2]{\left\langle #1,#2\right\rangle}			%
\renewcommand{\t}[1]{\widetilde{#1}}

\newcommand{\ca}[1]{\mathcal{#1}}

\newcommand{\R}{\mathbb{R}}
\newcommand{\Rd}{\R^d}

\renewcommand*{\d}[1]{\, \operatorname{d}\! #1}
\newcommand{\E}{\mathbb E}

\newcommand{\N}{\mathbb{N}}

\newcommand{\norm}[1]{\|#1\|}
\newcommand{\abs}[1]{|#1|}

\newcommand{\loss}{L}
\newcommand{\risk}[3]{\mathcal{R}_{#1,#2}(#3)}

\newcommand {\const}{C_{\mathrm{mul}}}
\newcommand {\constt}{C_{\mathrm{add}}}

\newcommand{\RDB}[1]{\mathcal R^*_{#1,D}}

\newcommand{\Fspace}{F}

\newcommand{\mdash}{M'}

\newcommand{\aminsq}{\norm{A^{-1}}^{-2}}
\newcommand{\amininv}{\norm{A^{-1}}}
\newcommand{\Mbeta}{M}
\newcommand{\MMbeta}{M'}

\newcommand{\firstt}[1]{\color{red} #1 \color{black}}
\newcommand{\secondd}[1]{\color{blue} #1 \color{black}}
\newcommand{\approvedd}[1]{\color{purple} #1 \color{black}}
\definecolor{fhcolor}{rgb}{0.523, 0.235, 0.625}
\newcommand{\fh}[1]{\textcolor{fhcolor}{(Fanghui: #1)}}
\definecolor{mscolor}{rgb}{0.323, 0.6, 0.155}
\newcommand{\ms}[1]{\textcolor{mscolor}{(Maxx: #1)}}
\definecolor{iscolor}{rgb}{0.8, 0.2, 0.155}
\newcommand{\is}[1]{\textcolor{iscolor}{(Ingos: #1)}}
\renewcommand{\ms}[1]{}
\renewcommand{\fh}[1]{}
\renewcommand{\is}[1]{}
\renewcommand{\firstt}[1]{#1}
\renewcommand{\secondd}[1]{#1}
\renewcommand{\approvedd}[1]{#1}

\newcommand{\f}[1]{f_{D,#1}}
\newcommand{\BM}[3]{D_{#1}(#2,#3)}
\newcommand{\ffun}[1]{f_{#1}}
\newcommand{\fdfun}[1]{f_{D,#1}}

\makeatletter
\newcommand{\ifcdot}[3]{%
  \def\@tempa{#1}\def\@tempb{\cdot}%
  \ifx\@tempa\@tempb
    #2%
  \else
    #3%
  \fi
}
\makeatother

\newcommand{\riskk}[1]{%
  \ifcdot{#1}%
    {\calR_{\loss,D}}%
    {\risk\loss D{#1}}%
}
\newcommand{\riskkk}[1]{%
  \ifcdot{#1}%
    {\calR}%
    {\calR({#1})}%
}

\newcommand{\myqquad}{\qquad\qquad}
\newcommand{\mycdot}{\,\cdot\,}
\newcommand{\intd}{\, \mathrm{d}}

\newcommand{\worse}[1]{\textcolor{red}{#1}}

\newcommand{\atob}[2]{\emph{#1)} $\Rightarrow$ \emph{#2)}.} 

\renewcommand{\a}{\alpha}
\renewcommand{\b}{\beta}
\newcommand{\g}{\gamma}
\newcommand{\e}{\varepsilon}
\newcommand{\lb}{\lambda}
\newcommand{\p}{\varphi}
\newcommand{\vt}{\vartheta}
\newcommand{\Om}{\Omega}
\renewcommand{\t}{\tau}
\newcommand{\U}{\Upsilon}

\newcommand{\eul}{\mathrm{e}}

\newcommand{\D}{\mathrm{D}}
\renewcommand{\P}{\mathrm{P}}
\newcommand{\PX}{\P_X}
\newcommand{\sA}{\ca A}

\newcommand{\cmul}{C_{\mathrm{mul}}}
\newcommand{\cadd}{C_{\mathrm{add}}}
\newcommand{\cemb}{C_{\mathrm{emb}}}

\newcommand{\frk}[1]{\mathfrak{#1}}

\DeclareMathOperator{\aco}{aco}
\DeclareMathOperator{\vcdim}{VC-dim}

\newcommand{\sL}[2]{{\ca L_{#1}(#2)}}
\newcommand{\Lx}[2] {{L}_{#1}(#2)}
\newcommand{\sLx}[2]{{\ca L_{#1}(#2)}}

\newcommand{\besov}[4]{B^{#1}_{#2,#3}(#4)}
\newcommand{\triebel}[4]{F^{#1}_{#2,#3}(#4)}
\newcommand{\besovtriebel}[4]{A^{#1}_{#2,#3}(#4)}
\newcommand{\sobol}[3]{W^{#1}_{#2}(#3)}
\newcommand{\bessel}[3]{H^{#1}_{#2}(#3)}
\newcommand{\zygmund}[2]{\ca C^{#1}(#2)}
\newcommand{\cbuc}[1]{C_{\mathrm{buc}}(#1)}
\newcommand{\cmbuc}[2]{C_{\mathrm{buc}}^{#1}(#2)}

\DeclareMathOperator{\ran}{ran}

\newcommand{\inorm}[1]{\Vert #1 \Vert_\infty}
\newcommand{\snorm}[1]{\Vert #1 \Vert}

\newcommand{\radd}[1]{\mathrm{Rad}_{D}(#1,n)}

\newcommand{\denseblo}{\frk D}

\newcommand{\RP}[2]{{{\cal R}_{#1,\P}(#2)}}
\newcommand{\RPB}[1]{{{\cal R}_{#1,\P}^{*}}}
\newcommand{\RPxB}[2]{{{\cal R}_{#1,\P,#2}^*}}
\newcommand{\RDxB}[2]{{{\cal R}_{#1,D,#2}^*}}
\newcommand{\RD}[2]{{{\cal R}_{#1,\D}(#2)}}

\newcommand{\clippt} {{{}^{\textstyle \smallfrown}} \hspace*{-1.5ex} t\hspace*{0.5ex}}

\newcommand{\clippf} {{{}^{^{\textstyle \smallfrown}}} \hspace*{-2.1ex} f\hspace*{0.5ex}}
\newcommand{\clippfo} {{{}^{^{\textstyle \smallfrown}}} \hspace*{-2.1ex} f}
\newcommand{\clippgo} {{{}^{{\textstyle \smallfrown}}} \hspace*{-2ex} g}
\newcommand{\clippdot} {{{}^{{\textstyle \smallfrown}}} \,\hspace*{-2.2ex} \,\mycdot\,}

\newcommand{\hfc}[1]{{{h}_{{{{}^{^{\scriptstyle \smallfrown}}} \hspace*{-1.8ex} f_#1}}}}
\newcommand{\hfco}{{{h}_{{{{}^{^{\scriptstyle \smallfrown}}} \hspace*{-1.8ex} f}}}}

\newcommand{\fDc}{{\clippfo_{D}}}
\newcommand{\fO}{{f_0}}
\newcommand{\fOc}{{\clippfo_0}}
\newcommand{\fDl}{f_{\D,\lb}}
\newcommand{\gDl}{g_{\,\D,\lb}}
\newcommand{\gPl}{g_{\,\P,\lb}}
\newcommand{\fTc}{\clippfo_{\D,\lb}}

\newcommand{\fpb}{{f_{L,\P}^*}}

\newcommand{\quadiasn}[8]{
\setlength{\unitlength}{1ex}
\begin{picture}(60,20)
\put(20,15){\makebox(9,2)[rt]{$#1$}}
\put(45,15){\makebox(9,2)[lt]{$#2$}}
\put(20,2){\makebox(9,2)[rt]{$#3$}}
\put(45,2){\makebox(9,2)[lt]{$#4$}}
\put(30,16){\vector(1,0){14}}
\put(28,14){\vector(0,-1){9}}
\put(46,5){\vector(0,1){9}}
\put(30,3){\vector(1,0){14}}
\put(34,17){\makebox(6,2){$#5$}}
\put(21,8.5){\makebox(6,2)[rt]{$#6$}}
\put(47,8.5){\makebox(6,2)[lt]{$#7$}}
\put(34,0){\makebox(6,2){$#8$}}
\end{picture}}

\renewcommand{\theenumi}{\emph{\alph{enumi})}}

\endlocaldefs

\begin{document}

\begin{frontmatter}
\title{Self-regularized learning methods}
\runtitle{Self-regularized learning methods}

\begin{aug}
\author[A]{\fnms{Max}~\snm{Sch\"olpple}\ead[label=e1]{max.schoelpple@mathematik.uni-stuttgart.de}},
\author[B]{\fnms{Fanghui}~\snm{Liu}\ead[label=e2]{fanghui.liu @sjtu.edu.cn}}
\and
\author[A]{\fnms{Ingo}~\snm{Steinwart}\ead[label=e3]{ingo.steinwart@mathematik.uni-stuttgart.de}}
\address[A]{Institute for Stochastics and Applications, University of Stuttgart
\printead[presep={ ,\ }]{e1,e3}}

\address[B]{School of Mathematical Sciences, Institute of Natural Sciences and MOE-LSC, Shanghai Jiao Tong University \printead[presep={ ,\ }]{e2}}
\end{aug}

\begin{abstract}
We introduce a general framework for analyzing learning algorithms based on the notion of self-regularization, which captures implicit complexity control without requiring explicit regularization. 
This is motivated by previous observations that many algorithms, such as gradient-descent based learning, exhibit implicit regularization.
In a nutshell, for a self-regularized algorithm  the complexity of the predictor is inherently controlled by that of the simplest comparator achieving the same empirical risk.
We apply this framework to (approximate) regularized empirical risk minimization, (pre-conditioned) Banach space mirror descent, which covers classic Hilbert space gradient descent, and proximal point algorithms. 
Building on self-regularization, we provide a thorough statistical analysis of such algorithms including minmax-optimal rates, where it suffices to show that the algorithm is self-regularized -- all further requirements stem from the learning problem itself. We show that in many statistical learning setups this methodology improves existing results or weakens assumptions.
\end{abstract}

\end{frontmatter}

\section{Introduction}
Many well-understood learning methods such as support vector machines (SVMs)  \cite{cortes1995support} and the LASSO \cite{tibshirani1996regression} can be viewed
as  \emph{(regularized) empirical risk minimizers} \cite{Vapnik1998}, where the learning method at hand
constructs its decision functions $f_{D,\lb}$ for a given data set 
$D=((x_1,y_1),\dots, (x_n,y_n))\in (X\times Y)^n$ 
by solving an  optimization problem of the form 
\begin{align*}
f_{D,\lb} \in \argmin_{f\in F} \bigl( \lb \snorm f_F^p + \RD\loss f \bigr)\, , 
\end{align*}
where $\lb\geq 0$ is a hyperparameter, $p\geq 1$ is typically chosen by algorithmic considerations, $F$ is some Banach space consisting of   functions on $X$, 
 $L:X\times Y\times \R\to [0,\infty)$ is a (convex) loss, 
and $\RD\loss f$
denotes the empirical risk of the predictor $f:X\to \R$, that is 
\begin{align*}
\RD\loss f  := \frac 1n \sum_{i=1}^n L(x_i,y_i,f(x_i))\, .
\end{align*}
On the other hand, many learning algorithms used in practice do not explicitly solve such a regularized optimization problem, where the most prominent examples are (gradient) boosting and  neural networks. 
In particular, 
gradient descent, mirror descent, and related iterative methods \cite{ruder2016overview} generate a parameter path instead of performing explicit regularization, and  generalize nevertheless, even for over-parametrized models with more parameters than data points \cite{zhang2021understanding,hastie2019surprises,liang2020just}. 
The latter phenomenon is often intuitively  explained as \emph{implicit regularization}, or an \emph{implicit bias }towards simple predictors of these iterative methods \cite{neyshabur2017geometry,soudry2018implicit,chizat2020implicit,wu2025benefits}.

In this work, we extend this intuition by formally introducing and analyzing \emph{self-regularized} learning methods
 $f_{D,t}$, where $t\in T$ denotes hyperparameters such as the stopping time in gradient descent.
In a nutshell, our notion of self-regularization ensures that, 
along with sufficient richness of the possible predictors $\{f_{D,t}\mid t\in T\}$,
for all 
hyperparameters $t$ and all data sets $D$, the predictor $f_{D,t}$ satisfies  
\begin{align}
	\label{eq:simplified_selfreg}
	\norm{f_{D,t}}_F \le \cadd + \cmul \cdot \inf\bigl\{ \norm{f}_F  : f\in F \mbox{ with } \risk \loss D f = \risk \loss D {f_{D,t}}            \bigr\}~,
\end{align}
where  $\cmul\ge 1$ and $\cadd\ge 0$ are some constants.

In other words, the norm of the predictor $f_{D,t}$ is not arbitrarily   larger than that of any comparator $f$ achieving the same empirical risk in $F$, which is in the spirit of complexity control \cite{Bartlett98a}. For example, when $F$ is a Sobolev space, \eqref{eq:simplified_selfreg} implies that the predictors $f_{D,t}$ are similarly smooth as the ``Sobolev-smoothest'' function achieving the same risk. 
In this sense, \eqref{eq:simplified_selfreg} indeed ensures a form of regularization 
\emph{without }specifying how this regularization is obtained. Now the key insight is that, 
modulo a richness assumption on the predictors and some minor  technical details, see 
\cref{def:self-regularization} for the complete definition of self-regularization,   \eqref{eq:simplified_selfreg} is enough to analyze learning algorithms in a state-of-the-art fashion. 
In more detail, the results presented in this work are two-fold.

First, we show that 
the class of self-regularized learning algorithms covers many learning methods. 
A straightforward application is Banach space based approximate regularized empirical risk minimization in \cref{ex:rerm-is--self-reg}, where  \cref{eq:simplified_selfreg} is fulfilled for $\cmul=1$ and $\cadd$ depending on the approximation quality.
Furthermore, we show in \cref{thm:mirror-descent-selfreg} and \cref{ex:mirror-descent-in-lp} that a large class of Banach space mirror descent algorithms is self-regularized under mild assumptions, where $\cmul$ depends on the exact setting, and $\cadd=0$. 
These methods especially cover Hilbert space based gradient descent, also with possible pre-conditioning,
as well as dictionary methods with coefficients weighted in $\ell_p(\N)$ for all $1<p\le 2$.
Another example of self-regularization is the proximal point algorithm, see \cref{thm:ppa-self-reg}, where again $\cmul$ depends on the setting, and $\cadd =0$. 

Second, we establish cornerstones of a statistical analysis for 
self-regularized learning methods. Namely, we first establish
an oracle inequality for regularized learning methods under minimal
assumptions on $L$, $F$, and $\P$, see \cref{thm:simple-analysis} and the subsequent 
discussion on the assumptions. 
This first oracle inequality holds for basically all convex loss functions of practical 
interest. Moreover, it is possible to derive both consistency and learning rates 
from this oracle inequality. 
In a second step, \cref{thm:learning_rates_tempered_C_minimal_general_form} presents  another oracle inequality 
under more restricted assumptions including a so-called variance bound. 
These additional assumptions make it possible to derive learning rates that are 
known to be asymptotically min-max optimal in some cases.
Moreover, in many other cases, the resulting learning rates implied for the self-regularized algorithms we considered are either equal or even better 
than the state-of-the-art, which is discussed in detail following \cref{thm:learning_rates_tempered_C_minimal_general_form}, and summarized in \cref{table:comparison_of_rates}.
Notably, to our best knowledge, this is the first work to establish statistical learning rates for mirror descent in Banach spaces and for the proximal point algorithm under comparably general assumptions.

 The rest of this work is organized as follows: 
 Section \ref{sec:prelims} 
 collects all notation as well as some basic assumptions needed throughout this work. 
 In Section \ref{sec:self-reg} we then formally introduce self-regularized 
 learning algorithms and discuss this definition in some detail. In addition, 
 we provide important 
 examples of self-regularized learning algorithms.
\Cref{sec:statistical_analysis} contains our  statistical analysis of   self-regularized algorithms.
\Cref{sec:proofs} contains the proofs of central results, and the supplements contain further, mostly previously known, results concerning the methods used in the proof section.

\section{Preliminaries}\label{sec:prelims}

Given a measurable space $(\Om,\ca A)$, we write $\sLx 0{\Om}$ for the space of measurable functions $\Om\to \R$.
Analogously, $\sLx \infty{\Om}$ stands for the Banach space of all bounded, measurable functions $\Om\to \R$
 equipped with the supremum-norm $\inorm\cdot$.  Moreover, if $\nu$ is a measure on $(\Om,\ca A)$ and $q\in [1,\infty)$, then 
 $\sLx q \nu$ denotes the space of all $f\in \sLx 0 \Omega$ that are $q$-integrable with respect to $\nu$.
If $X\subset \Rd$ is bounded and measurable and $\nu$ is the uniform distribution on $X$, we sometimes use the notation 
$\sLx q X := \sLx q \nu$. In all cases, $\Lx q \nu$ and $\Lx q X$ denote the corresponding Banach spaces consisting of $\nu$-equivalence classes.

{\bf Loss function.} In the following, let $X$ be a set equipped with some $\sigma$-algebra $\ca A$ and $Y\subset \R$ be 
measurable.
Then a function $\loss:X\times Y\times \R\to [0,\infty)$ is called \emph{loss (function)}, if it is measurable. 
Moreover, a loss is said to be \emph{convex}, \emph{continuous}, or \emph{differentiable}, if for all 
$(x,y)\in X\times Y$ the function $t\mapsto \loss(x,y,t)$ has the corresponding property. In addition, following 
\cite[Definition 2.18]{StCh08} we say that $\loss$ is \emph{locally Lipschitz continuous}, if 
for all $a\geq 0$ there exists a $c_a\geq 0$ with 
\begin{align*}
\bigl|\loss(x,y,t)-\loss(x,y,t')   \bigr| \leq c_a |t-t'|\, 
\end{align*}
for all $(x,y)\in X\times Y$ and all $t,t'\in [-a,a]$. In this case $|\loss|_{a,1}$ denotes the smallest such constant $c_a$. 
We say that $\loss$ is $M$-smooth for $M>0$, 
if $\loss$  is differentiable and for all $s,t\in\R$ we have
\begin{align}\label{eq:loss-m-smooth}
  \abs{\loss'(x,y,s)- \loss'(x,y,t)}\le M \abs{s-t}~.
\end{align}
Finally, we often need the following \emph{growth condition} 
\begin{align}\label{eq:loss-growth-order-q}
  \loss (x,y,t) &\le B (\abs{t}^q + 1)
\end{align}
for some fixed constants $B, q\geq 1$  and all 
$(x,y)\in X\times Y$ and $t\in \R$. 
Most commonly used losses 
satisfy these assumptions, see 
the discussion following 
\cref{thm:simple-analysis}.

{\bf Risks.} Assume a loss $\loss:X\times Y\times \R\to [0,\infty)$ and a distribution $\P$ on 
$X\times Y$. 
We define the \emph{risk} of an $f\in \sLx 0X$ as 
\begin{align*}
\RP \loss f := \int_{X\times Y} \loss(x,y,f(x))\intd\P(x,y)\, ,
\end{align*}
with its infimum denoted as the \emph{Bayes risk} $\RPB \loss := \inf\{\RP \loss f: f\in \sLx 0 X\}$, and
$\fpb\in \sLx 0X$ as \emph{Bayes decision function} if $\RP\loss\fpb = \RPB\loss$.
In addition, for a non-empty $\ca F\subset \sLx 0 X$ we denote the best possible risk in $\ca F$ by 
\begin{align*}
\RPxB\loss{\ca F} := \inf\{\RP \loss f: f\in \ca F\}\, .
\end{align*}
As usual, the \emph{empirical risk} of an $f\in \sLx 0 X$ and a data set 
$D=((x_1,y_1),\dots,(x_n,y_n))\in (X\times Y)^n$ is the risk of the corresponding empirical measure, that is 
\begin{align*}
\RD \loss f := \frac 1 n \sum_{i=1}^n \loss(x_i,y_i,f(x_i))\, .
\end{align*}
\firstt{Finally, we sometimes need that the growth condition  \eqref{eq:loss-growth-order-q} is asymptotically exact in an average sense, that is, there is a constant $c\in (0,1]$
potentially depending on $\P$ with 
\begin{align}\label{eq:loss-exact-growth-order-q-average}
\RP Lf \geq c \snorm f_{\sLx q \PX}^q -c^{-1}\, , \myqquad f\in \sLx \infty X,
\end{align}
where $\PX$ denotes the marginal distribution of $\P$ on $X$. Note that \eqref{eq:loss-exact-growth-order-q-average} holds true, if, e.g.,
$L$ satisfies $\loss (x,y,t) \ge c\abs{t}^q - c^{-1}$ for all $(x,y)\in X\times Y$ and $t\in \R$. %
}

{\bf Function spaces.} Following usual conventions \cite{LiZhZh22a,BaDVRoVi23a} we say that a Banach space $F$ is 
a \emph{reproducing kernel Banach space (RKBS)} or \emph{Banach space of functions (BSF)} on $X$, if
$F$ consists of functions $f:X\to \R$ and for all $x\in X$ the
 point evaluation functional $\delta_x:F\to\R$ defined by 
 \begin{align*}
  \delta_x(f)\equalDef f(x)\, , \myqquad f\in F
 \end{align*}
 is continuous. Of course, if in this case $F$ is actually a Hilbert space, then $F$ is said to be
  \emph{reproducing kernel Hilbert space  (RKHS)}. 
  Note that if a sequence $(f_n)\subset F$ in an RKBS $F$ converges, then it also converges pointwise. Moreover,
  if we additionally have $F\subset \sLx \infty X$, then this convergence is even uniform, as the continuity of the 
  embedding map $F\hookrightarrow \sLx \infty X$ is automatically given by the closed graph theorem, see 
  \cite[Lemma 23]{ScSt25a}. 
  
{\bf Approximation.}  
Now we assume that the loss is continuous $\loss:X\times Y\times \R\to [0,\infty)$ and satisfies the growth condition \eqref{eq:loss-growth-order-q} for suitable constants $B, q\geq 1$.
  Moreover, let $\P$ be a distribution on $X\times Y$ and $F$ be an RKBS on $X$ with $F\subset \sLx q \PX$, where
  $\PX$ denotes the marginal distribution of $\P$ on $X$. Then we obviously have $\RP \loss f< \infty$ for all $f\in F$.
  Moreover, if $F\subset \sLx q \PX$ is even dense, then  \cite[Theorem 5.31]{StCh08} shows  
  \begin{align}\label{eq:risk-approx}
\RPxB L F  = \RPB L
\end{align}
since  \eqref{eq:loss-growth-order-q} guarantees that $\loss$ is a
so-called Nemitski loss of order $q$ in the sense of
\cite[Definition 2.16]{StCh08}. In other words, if we consider the entire space $F$ then no approximation error occurs. 
To describe the approximation error in a more quantitative form, we   consider 
the  
\emph{approximation  error function} $A_p:[0,\infty) \to [0,\infty)$  defined by 
\begin{align*}
A_p (\lb) := \inf_{f\in F} \bigl( \lb \norm f_F^p + \RP L f \bigr) - \RPxB LF\, ,\myqquad  \lb\geq 0, 
\end{align*}
where $p\in [1,\infty)$ is a  constant.
An almost literal repetition of    \cite[Lemma 5.15]{StCh08} 
shows that $A_p$ is an increasing, concave, and 
continuous function with $A_p(0) = 0$. 
Moreover, modulo a constant, the functions $A_p$ are comparable for all $p\in [1,\infty)$, see 
Lemma \ref{lem:reg-trafo} for details. Below, we will thus state our approximation error assumption for $p=2$ only, 
although for some learning algorithms other choices of $p$ may look more natural.

\begin{assumption}[\sc{Approximation error (AE)}]
  \label{ass:approx-error}
  \asslabel{approx-error}{AE}
The loss function $\loss:X\times Y\times \R \to [0,\infty)$ is continuous and satisfies \eqref{eq:loss-growth-order-q}.
Moreover, $\P$ is a distribution on $X\times Y$ and $F\subset \sLx \infty X$ is an RKBS that is dense in $\sLx q \PX$.
Finally, we have 
\begin{align}\label{eq:approx-error-behav}
A_2(\lb) \leq c_\b \lb^\b\, , \myqquad \lb \in [0,1]\,,
\end{align}
where $\b\in (0,1]$ and $c_\b>0$ are constants independent of $\lb$.
\end{assumption}
 The parameter $\beta$ measures how well the hypothesis space can approximate the optimal function. In short,  larger $\beta$ reflects better approximation.
  In particular, $\b=1$ is achieved 
  if there is an $f\in F$ with $\RP \loss f = \RPxB \loss F$, and in the RKHS case the latter condition is 
  also necessary for $\b=1$, see e.g.~\cite[Corollary 5.18]{StCh08}.
 Moreover, for the least squares loss, $\b\in (0,1)$, and dense $F\subset \sLx 2 \PX$, 
 there is a direct characterization of \eqref{eq:approx-error-behav} in terms of so-called interpolation spaces,
 see \cite{SmZh03a} for the RKHS case, which however directly generalizes to RKBSs.
Finally, for RKHSs and the least squares loss, 
\cite[Corollary 4.7]{StSc12a} in combination with the already mentioned result of \cite{SmZh03a}
establishes a simple connection to the so-called source condition
used in e.g.~ \cite{CaDe07a,BlM2017,LiCe2018,LiRuRoCe2018} and many other papers
on kernel-based least squares regression.
  
{\bf Complexity measures.}    
We also need to recall the definition of   entropy numbers, which will serve as a complexity measure.
To this end, 
we fix two  Banach spaces $E$ and $F$ and a bounded linear $A:E\to F$.
For $i\geq 1$, the $i$-th (dyadic)  \emph{entropy number} of $A$ is then defined by 
\begin{align*}
e_i(A) :=  \inf\biggl\{ \varepsilon > 0:  \exists y_1,\dots,y_{2^{i-1}} \in AB_E \mbox{ with } AB_E \subset \bigcup_{j=1}^{2^{i-1}} (y_j + \e B_F)  \biggr\} \, ,
\end{align*}
where $B_E$ and $B_F$ denote the closed unit balls of $E$ and $F$, and we use the convention
$\inf\emptyset := \infty$. Note that conceptually entropy numbers are the functional inverse of covering numbers,
see e.g.~\cite[Lemma 6.21 and Exercise 6.8]{StCh08} for the case of polynomial behavior. 

{\bf Mirror and gradient descent.}
To begin with, recall that a map  
\(
\Phi : F \to \mathbb{R}
\)
on a Banach space $F$ 
is Fréchet differentiable, if for all $x_0\in F$ there exists a unique $\nabla \Phi(x_0) \in F'$ 
fulfilling 
\begin{align*}
\lim_{x\to 0} \frac {\bigl|\Phi(x+x_0) - \Phi(x_0) - \nabla \Phi(x_0)(x)\bigr|}{\snorm x_F} = 0\, . 
\end{align*}
If $H$ is a Hilbert space we will view, as usual, $\nabla \Phi(x_0)$ as an element of $H$ via the Riesz-Fischer theorem.
Moreover, if $F$ is an RKBS on $X$ with $F\subset \sLx \infty X$, then for every differentiable loss $\loss$, the empirical risk $\calR _{\loss,D}:F\to [0,\infty)$ is Fréchet differentiable for any data set $D$, see 
e.g.~\cite[Lemma 2.21]{StCh08}.

\secondd{
To define mirror descent, cf.\ \cite[Sections~4.1--4.2]{bubeck2015convex}, we call a map $\psi: F\to [0,\infty)$ a \emph{mirror map}, if it is 
strictly convex, fulfills $\psi (0) = 0$, $\nabla\psi(0) = 0$ and it is Fréchet differentiable such that $\nabla\psi:F\to F'$ is a continuous bijection with continuous inverse.
Sufficient conditions for maps of the form $\psi(f)=\norm{Af}_F^2$ are discussed in \cref{lem:checking-mirror-map-requirements}.
The general form $\psi(f)=\norm{Af}^2_F$ allows for pre-conditioning cf.\ \cite[§9.4.1]{boyd2004convex}, used e.g.\ in kernel learning \cite{ma2017diving}.
In particular, 
if $H$ is a Hilbert space,  $\psi(x) := \snorm x_H^2/2$
defines a mirror map with $\nabla\psi = \id_H$, where again we identify $H'$ with $H$. 
Now, the \emph{mirror descent} algorithm for a Fréchet differentiable  $\calR:F\to \R$ and a sequence of step sizes $(\eta_k)_{k\in\N_0}\subset (0,\infty)$
is an iterative procedure that starts at $f_0 := 0$ and computes, for  $t>0$,  $k\in\N_0$, and $s\in (0,1]$ with $t=k+s$, its update by
        \begin{align}
            \label{eq:mirror-descent-dual-interpolation}
            \ffun{t} = \ffun{k+s} := (\nabla \psi)^{-1} \bigl(  \nabla \psi(\ffun {k}) - s\eta_k \nabla\calR(\ffun{k}) \bigr)~.
        \end{align}
Here we note that in the literature one often only considers $s=1$, in which case we obtain a sequence $(f_k)_{k\geq 0}$. 
As  we need   continuous maps $t\mapsto \ffun{t}$ for our later results, we also consider $s\in (0,1)$.
Note that the above mentioned mirror map
$\psi(f)\equalDef \norm f_H^2/2$ for Hilbert spaces $H$ directly leads to the well-known 
\emph{gradient descent} update rule 
        \begin{align}
            \label{eq:gradient-descent-interpolation}
            h_{k+s} &\equalDef h_k - s\eta_k \nabla \riskkk{h_k} \qquad k\in\N_0, s\in (0,1]\,  .
        \end{align}
Closely related is the \emph{proximal point algorithm} \cite{Ro76,Gu91}, which is  initialized at $h_0 := 0$ and iterated by solving the equations
\begin{align}
    \label{eq:ppa-step}
    h_{k+1} =  h_{k} - \eta_k \nabla \calR ({h_{k+1}})~,\myqquad k\geq 0
\end{align}
with respect to $h_{k+1}\in H$. To extend this to a continuous map $t\mapsto h_t$, we again use linear interpolation, namely
$ h_{k+s} = h_{k} + s(h_{k+1}-h_{k})$ for all $k\in \N_0$ and $s\in (0,1]$.

}

    \section{Self-Regularization: Definition and Examples}\label{sec:self-reg}

In this section we introduce the main concept of this paper, namely \emph{self-regularized learning  algorithms}. In addition, we present major examples of this class of algorithms.

Recall that
a \emph{learning method} assigns to every data set $D \in (X\times Y)^n$
a function $f_D:X\to \R$. Following, \cite[Definition 6.2]{StCh08}
we say that such a learning method is \emph{measurable}, if the map $(D,x) \mapsto f_D(x)$ 
is measurable for all $n\geq 1$. If $F$ is a separable BSF with $f_D\in F$ for all $D$, then this is equivalent to 
the Borel measurability of $D\mapsto f_D$ for all $n\geq 1$, see Lemma \ref{lem:basic-measurability}.
Moreover, for measurable learning methods,  maps such as $D\mapsto \RP L{f_D}$ are measurable, 
see e.g.~\cite[Lemma 6.3]{StCh08}.
Now, basically all learning methods of interest have some hyperparameters. 
We denote the set of such hyperparameters by $T$ and speak of a \emph{parametrized 
learning method} $(D,t)\mapsto f_{D,t}$ if for all $t\in T$ the maps $D\mapsto f_{D,t}$ give a  measurable
learning method.\footnote{To ensure the existence of certain  measurable learning methods  we always assume  
that $(X\times Y)^n$ is equipped with a $\sigma$-algebra containing
 the universal completion of  $(\ca A\otimes \frk B)^n$, where $\frk B$ is the Borel $\sigma$-algebra on $Y$.
 For example, this is the case if we consider the $\P^n$-completion  of 
 $(\ca A\otimes \frk B)^n$ for a distribution $\P$ on $\ca A\otimes \frk B$.  }
With these preparations we can now introduce self-regularized learning algorithms.

\begin{definition}[Self-regularized learning]\label{def:self-regularization}
Let $\loss:X\times Y\times \R \to [0,\infty)$  be a continuous loss,
 $F$ be a separable RKBS over $X$, and $T$ be a path-connected Polish space.
Then we say that a parametrized learning method $(D,t)\mapsto f_{D,t}$ is 
\emph{$(\const, \constt)$-self-regularized in $F$}, for constants $\const\geq 1$, $\constt\geq 0$, 
if for all data sets $D$  the following properties hold:
\begin{enumerate}
    \item For all $t\in T$ we have $f_{D,t}\in F$ and all maps $(D,t)\mapsto f_{D,t}$ are measurable.
    \item  The function
    $t\mapsto \RD\loss{f_{D,t}}$ is continuous.
    \item There exists a $t_0\in T$ with $\risk \loss D {f_{D,t_0}} \ge \risk \loss D 0$ and $\snorm {f_{D,t_0}}\leq \constt$.
    \item We have 
    \begin{align}\label{eq:self-reg:richness}
        \inf_{t\in T} \risk \loss D{f_{D,t} }= \RDxB \loss F \, .
    \end{align}
    \item\label{item:norm-bound} For all $t\in T$ and all $f\in F$ 
    with 
    $\risk\loss Df = \risk\loss D{f_{D,t}}$ 
    one has
    \begin{align} 
    \label{eq:norm_bound_selfreg}
    \norm{f_{D,t}}_F \le \const \cdot \norm f_F + \constt   ~.
    \end{align}
\end{enumerate}
\end{definition}

Before we consider examples of self-regularized learning methods, we briefly discuss the different 
assumptions in \cref{def:self-regularization}. Clearly, $f_{D,t}\in F$  is a necessary technical assumption as otherwise  neither \emph{c)} nor \emph{e)}
would make sense. Moreover, the measurability of $(D,t)\mapsto f_{D,t}$  is of technical nature. 
For parametrized 
learning methods and separable $F$ one can show that it is satisfied as soon as $t\mapsto f_{D,t}$ is continuous. Obviously this continuity also implies \emph{b)}
by the assumed continuity of $L$. 
Finally, our examples below only consider the Polish spaces $T=[0,\infty)$ and $T=(0,\infty]$.
\todo{can easly expand this by considering the step sizes as well. This looks like a trick, but makes some sense as they often in fact \emph{are} hyperparameters} 

Assumption \emph{c)} ensures that the parametrized learning method can produce a ``sufficiently na\"ive'' decision function $f_{D,t_0}$. 
Clearly, this assumption is satisfied if for all data sets $D$ we have a, potentially  data dependent, $t_0\in T$ with 
$f_{D,t_0} = 0$. 
In contrast, \emph{d)} guarantees that the  parametrized learning method can produce decision functions whose risks are arbitrarily close 
to the smallest possible empirical risk $ \RDxB \loss F$ in $F$.  In particular, if the parametrized learning method is able to produce 
\emph{interpolating} decision functions, that is, 
\begin{align*}
\RD\loss {f_{D,t^*_D}} = \RDB \loss
\end{align*}
for suitable, data-dependent hyperparameters $t^*_D\in T$, then  \emph{d)} is satisfied. 
Ignoring the case $\RD\loss 0 =  \RDxB \loss F $, 
the conditions \emph{b)} to \emph{d)} together with the path-connectedness of $T$ guarantee
that the parametrized learning method is able to construct 
decision functions $f_{D,t}$ with $\RD\loss {f_{D,t}} = \a$ for all $\RDxB \loss F < \a \leq \RD\loss 0$.
In this sense, the parametrization of self-regularized learning algorithms is ``flexible''.

The  name-giving aspect of self-regularized learning algorithms is condition \emph{e)}. 
To illustrate it, assume that 
there is a \emph{minimal norm predictor} $f_{D,t}^\dagger\in F$ for the risk $\risk\loss D{f_{D,t}}$, that is 
\begin{align*}
\RD\loss{f_{D,t}^\dagger} = \risk\loss D{f_{D,t}}
\end{align*}
and $\snorm{f_{D,t}^\dagger}_F \leq \snorm f$ for all $f\in F$ with $\RD\loss f = \risk\loss D{f_{D,t}}$. Then \emph{e)} implies 
\begin{align*}
\norm{f_{D,t}}_F \le \const \cdot \norm {f_{D,t}^\dagger}_F + \constt \, .
\end{align*}
In this sense, the \emph{self-regularization inequality} \eqref{eq:norm_bound_selfreg} ensures that none of the decision functions $f_{D,t}$ produced by 
the self-regularized learning algorithm behaves ``too wildly''. 
\secondd{Finally, our statistical analysis in Theorems \ref{thm:simple-analysis} and
\ref{thm:learning_rates_tempered_C_minimal_general_form} explicitly depends on the constants $\const$ and $\constt$.

The stated inequalities still hold if $\cmul$ and $\cadd$ depend, for example, on the sample size $n$ of $D$, cf.\ discussion after \eqref{eq:rho-in-text}
}

Next we present some examples of self-regularized learning algorithms. We begin 
by considering \emph{regularized empirical risk minimization (RERM)} in the following example.

\begin{example}\label{ex:rerm-is--self-reg}
 Let $\loss:X\times Y\times \R \to [0,\infty)$  be a convex and continuous loss, $p\in [1,\infty)$, $\epsilon >0$, and
 $F$ be a separable RKBS over $X$. We define $T:= (0,\infty]$. Moreover, for all datasets $D$ we set  $g_{D,\infty} := 0$ and pick, 
 for all $\lb>0$, a $g_{D,\lb}\in F$
 satisfying 
 \begin{align}\label{eq:rerm-opt-prob-eps} 
    \lb \snorm{g_{D,\lb}}_F^p + \RD\loss {g_{D,\lb}} \leq  \inf_{f\in F} \bigl(\lb \snorm f_F^p  + \RD\loss f   \bigr) + \lb \epsilon \, .
 \end{align}
 Then, if all maps $(D,\lb)\mapsto g_{D,\lb}$ are measurable and 
     all maps $\lb\mapsto \RD\loss{g_{D,\lb}}$  are continuous, 
 $(D,\lb)\mapsto g_{D,\lb}$ is self-regularized with $\const = 1$
 and $\constt = \epsilon^{1/p}$.
\end{example}

The measurability and continuity are usually not a problem. For example, 
if $F$ is a separable RKHS and $p=2$ then it is well-known that \eqref{eq:rerm-opt-prob-eps} with $\epsilon = 0$ has a unique   solution $g_{D,\lb}$. 
In addition, all maps $\lb\mapsto g_{D,\lb}$ are continuous, see e.g.~\cite[Corollary 5.19]{StCh08} and all maps $D\mapsto g_{D,\lb}$ are measurable, see 
e.g.~\cite[Lemma 6.23]{StCh08}. This gives the measurability of $(D,\lb)\mapsto g_{D,\lb}$, see e.g.~\cite[Lemma A.3.17]{StCh08}, 
and obviously also the continuity
 of $\lb\mapsto \RD\loss{g_{D,\lb}}$.
The same is true under mild technical assumptions for 
Lasso-type $\ell_1$-regularization with $p=1$ and the least squares loss, see e.g.~\cite[Lemma 2]{MaYu12a} for the continuity, while the measurability of 
$D\mapsto g_{D,\lb}$  can be seen by a 
straightforward adaptation
of \cite[Lemma 6.23]{StCh08}.

Closely related to exact RERM, that is \eqref{eq:rerm-opt-prob-eps}  with $\epsilon = 0$, is \emph{Ivanov  empirical risk minimization (IERM)}, which is defined by 
\begin{align*}
    f_{D,r} \in  \argmin_{f\in rB_F} \riskk{f}\, , \myqquad r\geq 0\, . 
\end{align*}
Here, it is assumed that such a minimizer exists and that in the case 
 of multiple minimizers a minimal-norm minimizer is chosen. Obviously, IERM controls the norm of its predictors even more explicitly than RERM does
 and the self-regularization inequality \eqref{eq:norm_bound_selfreg} holds with  $\const = 1$
 and $\constt = 0$. Moreover,   \emph{c)} and \emph{d)} of Definition \ref{def:self-regularization} are also obviously satisfied, so IERM
 is self-regularized as soon as the measurability \emph{a)} and the continuity \emph{b)} hold. In the Hilbert space case 
 the latter follow from  $f_{D,r} = g_{D,\lb}$ with $r = \snorm{g_{D,\lb}}_H$
 for all $\lb\geq 0$, see e.g.~\cite[Corollary 5.24]{StCh08}.
 Finally, similar remarks can be made for the so-called Morozov regularization, see e.g.~\cite{Oneto16} for details.

\secondd{Our next goal is to show that mirror descent, gradient descent, and the proximal point algorithm are all self-regularized. 
Before we present the corresponding results, we note that gradient descent in an RKHS $H$ with suitable step sizes $\eta_k$ satisfies, at discrete times $n>k\geq 0$, the 
following version of \emph{Fej\'er monotonicity}
\begin{align}
    \label{eq:fejer-monotonicity}
    \norm{\f {k+1} - f^\dagger } &\le \norm{\f{k} -f^\dagger }\qquad\text{ for all }f^\dagger \in H\text{ with }\riskk{f^\dagger } = \riskk{\f{n}}~.
\end{align}
Consequently, we have $ \snorm {f_{D,n} - f^\dagger} \leq \snorm {f_{D,0} - f^\dagger}$.
By the assumed $f_{D,0} = 0$, this shows
\begin{align*}
\snorm{f_{D,n}} \leq \snorm {f_{D,n} - f^\dagger} + \snorm{f^\dagger} \leq \snorm {f_{D,0} - f^\dagger} + \snorm{f^\dagger} = 2 \snorm{f^\dagger} \, . 
\end{align*}
In other words we have $\const = 2$ and $\constt = 0$ at discrete times $n\geq 1$. The proofs of the following results
pick up this idea and extend it to continuous times. 

}

\firstt{

}

\firstt{
The following theorem is the central result for self-regularization of mirror descent. A useful specialization is presented in the subsequent \cref{ex:mirror-descent-in-lp}.
\begin{theorem}[Mirror descent is self-regularized]\label{thm:mirror-descent-selfreg}
    Let $F\hookrightarrow \ell_\infty(X)$ be an RKBS, $\psi:F\to[0,\infty)$ be a mirror map and let the loss  $\loss$ 
    be convex and $\Mbeta$-smooth.
    Assume that there is a $d_\psi>0$ be such that
    \begin{align}
        \label{eq:constant-comparing-bmdiv-and-squarednorm}
        \BM \psi fg \ge \frac {d_\psi}2 \norm{f-g}^2_F
    \end{align}
    holds for all $f,g\in F$ and define 
    $
    \MMbeta \equalDef M{d_\psi}^{-1} \norm{F\hookrightarrow \ell_\infty(X)}^2~.
    $ Let $C_\psi$ be given by
    \begin{align}
        \label{eq:c-psi}
        C_\psi \equalDef &\ \sup_{f\in F\setminus\{0\}} \frac {\norm{\nabla\psi(f)}_{F'}\norm f_F}{\scal{\nabla \psi(f)}{f}_{F',F}-\psi(f)}
    \end{align}
    and assume $C_\psi<\infty$. Let $(\eta_k)$ be step sizes fulfilling $\eta_k\le 1/\MMbeta$  and $\sum_{k=0}^\infty \eta_k=\infty $.
    Then, the corresponding mirror descent \eqref{eq:mirror-descent-dual-interpolation}  with $\calR \equalDef \calR_{L,D}$ 
    is self-regularized with $\cmul= C_\psi$ and $\cadd =0$.
\end{theorem}
}

\firstt{The following example is of great generality. Especially, it covers RKHS gradient descent, possibly with pre-conditioning, as well as some dictionary methods.
\begin{example}\label{ex:mirror-descent-in-lp}
    Let the loss  $\loss$ be convex and $\Mbeta$-smooth. Let $F$ be an RKBS over $X$ with  $F\hookrightarrow \ell_\infty(X)$ that is isomorphic to $L_p(Z,\mu)$ for some measure space $(Z,\mu)$ with $1<p\le 2$.
    Let $A:F\to F$ be an isomorphism.
    Define the mirror map 
    \begin{align*}
        \psi(f) \equalDef \norm{Af}^2/2~,
    \end{align*}
    and let the step sizes $(\eta_k)$ fulfill $\eta_k\le (p-1)\amininv^{-2}\Mbeta^{-1}\norm{F\hookrightarrow \ell_\infty(X)}^{-2}
    $ 
    as well as $\sum_{k=0}^\infty \eta_k =\infty$. Then, the corresponding mirror descent is self-regularized with $\cmul=2\amininv \norm{A}$ and $\cadd =0$.
\end{example}
\begin{corollary}\label{thm:gradient_descent_in_RKHS}
    Let $H$ be an RKHS and let the loss function $\loss$ be convex and $M$-smooth.
    Let the step sizes $(\eta_k)_{k\in\N_0}$ fulfill $\eta_k \le (M \norm{H\hookrightarrow \calL_{\infty}(X)}^2)^{-1}$ for all $k\in\N_0$ and $\sum_{k=0}^{\infty} \eta_k = \infty$.
    Then, gradient descent given by \eqref{eq:gradient-descent-interpolation} is self-regularized with parameters $\const =2 ,\constt =0$. 
\end{corollary}

}

\firstt{
\begin{example}[Dictionary methods]\label{ex:dictionary-methods}
    Let $1<p\le 2$ and let $q=p/(p-1)$ be the conjugate exponent. Let $(h_i)_{i\in\N}:X\to \R$ be a family of functions satisfying $(h_i(x))_{i\in\N}\in \ell_q(\N)$ for all $x\in X$ 
    and let $T:\ell_p\to \R^X$,
    \begin{align*}
        T(a) \equalDef \sum_{i=1}^\infty a_i h_i
    \end{align*} 
    be injective. Define the RKBS $F\equalDef T(\ell_p)$ with the norm $\norm{T(a)} \equalDef \norm{a}_{\ell_p(\N)}$. 
    Then, $F$ is an RKBS which is isometrically isomorphic to $\ell_p(\bbN)$, cf. \cite[§3.2]{xu2019generalized}, and mirror descent \eqref{eq:mirror-descent-dual-interpolation} of suitable step sizes is self-regularized by \cref{ex:mirror-descent-in-lp}.
\end{example}
}
\firstt{
\begin{theorem}[Proximal points is self-regularized]\label{thm:ppa-self-reg}
    Let $H$ be an RKHS on $X$, let $L$ be a convex, $\Mbeta$-smooth loss and let $\eta_0\ge \eta_1\ge \dots $ be non-increasing step  sizes.
    Define $\MMbeta \equalDef \Mbeta\norm{H\hookrightarrow \ell_\infty(X)}^2$ and 
    assume $\sum_{k=0}^\infty \eta_k=\infty$ and $\eta_0 \MMbeta <1$.
    Then, the proximal point algorithm $h_{D,t}$ given by \eqref{eq:ppa-step} with $\calR \equalDef \calR_{\loss,D}$ is self-regularized with $\cmul=\tfrac{2(2-\eta_0\MMbeta)}{1-\eta_0\MMbeta}$ and $\cadd=0$.
\end{theorem}
}

\section{Statistical analysis: Oracle inequalities}\label{sec:statistical_analysis}

The goal of this section is to present two results on 
the statistical analysis of self-regularized learning algorithms: The first one holds under somewhat minimal assumptions 
  and can be easily extended to e.g.~non i.i.d.~observations, while the second one 
 presents a powerful oracle inequality that
in various cases leads to minmax optimal
rates, or at least to improvements of the current state of the art. 

Before we begin, we sketch the central techniques we will use:  First,  \cref{thm:self-reg-implies-crem} shows that  given a $\lb>0$ and a suitably chosen, data-dependent  hyperparameter $t_D$ 
the predictors $f_{D,t_D}$ of   self-regularized
learning algorithms can be viewed as approximate RERM solutions with the \emph{data-dependent} accuracy term 
\[
\epsilon_D=\lambda\bigl( (\const -1) \norm {g_{D,\lb}}_F +  (\const + 1)  \constt + \epsilon \bigr)\, , 
\]
where $g_{D,\lb}$ is a suitable approximate RERM solution of \eqref{eq:rerm-opt-prob-eps}  with $p=1$ and $\epsilon$. In the second step,
we then establish 
oracle inequalities for approximate RERM solutions with data-dependent accuracy terms that bound both, the excess risks \emph{and} the norms of these solutions,
see Theorems \ref{concentration-hoeffding-applicat-simple-thm} and \ref{concentration:bsf-svm-analysis-refined}.
In the third step we then apply these oracle inequalities twice, namely to first bound $\epsilon_D$ and then the excess risk of $f_{D,t_D}$. In fact, 
the main results  of this section, namely  \cref{thm:simple-analysis} and \cref{thm:learning_rates_tempered_C_minimal_general_form}, only differ by the employed oracle inequalities
in the third step. In particular, additional guarantees for self-regularized learning algorithms can be established by new oracle inequalities for RERMs in the sense of the second step. 

\begin{theorem}\label{thm:simple-analysis}
Let  
$L:X\times Y\times \R\to [0,\infty)$ be a convex, locally Lipschitz continuous loss satisfying 
the growth condition \eqref{eq:loss-growth-order-q}
with constants $B, q\geq 1$.
Moreover, let $F\subset \sLx \infty X$ be a separable RKBS satisfying both 
$\snorm{F\hookrightarrow \sL \infty X} \leq 1$
and 
\begin{align}\label{eq:sup-entropy-numbers}
e_{i} (\id:F\to \sL \infty X )
\leq a \, i^{-\frac 1{2\g}} \, , \myqquad i\geq 1,
\end{align}
for some constants $a\geq 1$ and $\g> 0$ independent of $i$. 
Then, for all 
 probability measures $\P$ on $X\times Y$, all self-regularized learning methods $(D,t) \to f_{D,t}$ in $F$
 and  all $n\geq 1$, $\lb\in (0,1]$, and $\t\geq 1$ with 
 $\t^{1+\g}\leq n \lb^{2q(1+\g)}$
 there exists a measurable map $D\mapsto t_D$ such that    we have 
\begin{align}\label{eq:oracle-inequality-simple}
 \RP L{f_{D,t_D}}-\RPxB LF 
\leq  A_{1}(\lb) + K_{\mathrm{alg}}  \cdot \lb^{-q}  \cdot  n^{-\frac {1}{2+2\g}}  \sqrt{\t}
\end{align}
with probability $\P^n$ not less than $1-2\eul^{-\t}$, where 
\begin{align*}
 K_{\mathrm{alg}} := \Bigl( 8\const\bigl( (2B)^{q }  (a^\g B + 2^{2+q}) + \constt\bigr) \Bigr)^{q  }\, .
\end{align*} 
\end{theorem}

Before we interpret \cref{thm:simple-analysis} we quickly discuss the assumptions: First, the local Lipschitz continuity is 
satisfied for essentially all margin-based losses for binary classification including the hinge loss,  the logistic loss for 
classification, and the (truncated) least squares loss, see e.g.~\cite[Chapter 2.3]{StCh08}. In addition,  the first two examples 
satisfy  \eqref{eq:loss-growth-order-q}
for $q=1$, 
while the (truncated) least squares loss satisfies  \eqref{eq:loss-growth-order-q}  
for $q=2$.
Moreover, if we consider distance-based losses for regression and assume that $Y$ is bounded, then all commonly used losses are also 
locally Lipschitz continuous and satisfy \eqref{eq:loss-growth-order-q}  and \eqref{eq:loss-exact-growth-order-q-average}  for either $q=1$ or $q=2$. For example, this is true for
the pinball losses used for quantile regression with $q=1$, the least squares loss and the asymmetric least squares losses used for expectile regression with $q=2$.
Moreover, it is true for Huber's loss, the logistic loss for regression, and the $\e$-insensitive loss used in SVMs with $q=1$. For details, we refer to 
\cite{FaSt17a,FaSt19a} for the asymmetric least squares loss and to 
\cite[Chapter 2.4]{StCh08} for the other losses mentioned.

The separability as well as the 
assumption  $F\subset \sLx \infty X$ are technical assumptions satisfied for essentially all RKBS of algorithmic interest. Moreover, note that 
$F\subset \sLx \infty X$ automatically implies the continuity of the resulting  embedding operator, see \cite[Lemma 23]{ScSt25a}. Consequently, 
the assumption $\snorm{F\hookrightarrow \sL \infty X} \leq 1$ can always be achieved by rescaling the norm of $F$. In fact, we only added this assumption 
to avoid additional constants in the oracle inequality. 
The entropy number assumption \eqref{eq:sup-entropy-numbers} is somewhat standard
in the statistical analysis of learning methods as  it can be reformulated 
as a polynomial upper bound on the log-covering numbers, see \eqref{eq:log-cover-bound} for a quick reference and 
\cite[Lemma 6.21 and Exercise 6.8]{StCh08} for more details.
Furthermore, \eqref{eq:sup-entropy-numbers} is satisfied with $\g:= \frac d {2s}$,
if there exists a BSF $E$ of ``$s$-times smooth'' functions with 
$F\hookrightarrow E\subset \sL \infty X$. Such spaces $E$ include
(fractional versions of) Sobolev spaces or spaces of $s$-times 
continuously differentiable functions. We refer to 
Lemma \ref{lem:entropy-for-besov-triebel} and the subsequent remarks for precise statements.

Let us now take a look at the consequences of \cref{thm:simple-analysis}. To this end, we assume that $F\subset \sLx q\PX$ is dense, 
so that we have $\RPxB L F  = \RPB L$, see
\eqref{eq:risk-approx}. If we then choose $\lb = \lb_n = n^{-\rho}$ with $0<\rho < (2q+2q\g)^{-1}$
the right-hand side of \eqref{eq:oracle-inequality-simple}  converges to $0$. Consequently, the learning method $D\mapsto f_{D,t_D}$ is 
consistent for $\P$, and it is even universally consistent, if $F\subset \sLx q\nu$ is dense for all probability measures $\nu$ on $X$.
The latter denseness assumption can be ensured, for example, if $X$ is a compact metric space and $F$ is universal in the sense of $F\subset C(X)$ being dense, where $C(X)$ denotes 
 the space of continuous functions $X\to \R$ equipped with the supremum norm.
 Indeed, this claim can be verified by 
  a literal repetition of 
 the proof of \cite[Corollary 5.29]{StCh08}.
 
 Moreover, it is also possible to derive learning rates under the Approximation Error \assref{approx-error}.
Indeed, if \eqref{eq:approx-error-behav} is satisfied, then \cref{lem:reg-trafo} shows $A_1(\lb) \leq c \lb^{\frac{2\b}{1+\b}}$ for some constant $c>0$ and all $\lb>0$.
If we  ``asymptotically optimize'' the right-hand side of \eqref{eq:oracle-inequality-simple}   by choosing   $\lb = \lb_n = n^{-\rho}$ with 
\begin{align}
\rho := \frac {1+\b} {2(1+\g)(q+q\b + 2\b)}\label{eq:rho-in-text}
\end{align}
we then obtain a learning rate of $n^{-2\b \rho/(1+\b)}$. 
Finally, an interesting observation is that the self-regularization constants $\const$ and $\constt$ may moderately grow with $n$ without losing some statistical control.
For simplicity we only  illustrate this in the case $\const = n^{\a}$ and  $\constt = 0$ with $0<\a < (2q+2q\g)^{-1}$. If we  then choose
$0<\rho<  (2q+2q\g)^{-1} - \a$ and $\lb_n:= n^{-\rho}$ we still get consistency. Moreover, in this regime we can also obtain rates, but not surprisingly, these are worse than the ones found above.

Our next goal is to present a refined analysis, which in the best cases
can produce learning rates up to $n^{-1}$. To this end, we need
to introduce a few more concepts. 

To begin with, we say that a loss $\loss:X\times Y\times \R\to [0,\infty)$ is 
\emph{clippable }at $M>0$  if it satisfies 
\begin{align}\label{eq:def-clippable}
\loss(x,y, \clippt) \leq \loss(x,y,t)\, , \myqquad x\in X, y\in Y, t\in \R,
\end{align}
where $\clippt$ denotes the value of $t$ clipped at $\pm M$, that is
\begin{align*}
\clippt := \begin{cases}
            t \, , & \mbox{ if } t\in[-M,M]\\
            M\, , & \mbox{ if } t>M\\
            -M \, , & \mbox{ if } t < -M \, .
        \end{cases}
\end{align*}
Informally speaking, if $\loss$ is clippable at $M$, then any prediction $t$ can be potentially improved by clipping it at $M$. This notion 
is somewhat folklore and 
probably goes back to 
\cite{Bartlett98a} in the context  of neural networks, see also \cite{BoEl02a,Chen2004Support,WuYiZh05a} in the context for 
RERM-like   schemes. 
Since then, it has been 
frequently used in the literature.
 Fortunately,  many 
loss functions are indeed clippable, see e.g.~\cite[Lemma 2.23]{StCh08} for a simple characterization of convex, clippable
losses. In particular, all losses mentioned  below \cref{thm:simple-analysis} are clippable with the exception
of the logistic loss of classification. 
In addition, two-sided versions of the latter loss are again clippable, see e.g.~\cite[Lemma 2.7]{Steinwart09a}, \firstt{and  
these versions also satisfy \eqref{eq:loss-exact-growth-order-q-average} for $q=1$. Finally, \eqref{eq:loss-exact-growth-order-q-average} 
holds for all two-sided versions of the margin-based losses mentioned below Theorem \ref{thm:simple-analysis} and the same is true
for the original margin-based losses, if the noise level  of $\P$  is bounded away from zero in the sense of the second condition in \eqref{eq:tsybakov-restricted}.
}

The following assumption collects the properties of losses we need for our refined analysis.
\begin{assumption}[\sc{Loss and distribution (LP)}]
    \label{ass:loss}
  \asslabel{loss}{LP}
    The loss $\loss:X\times Y\times \R \to [0,\infty)$ is locally Lipschitz continuous, convex,  and clippable at $M>0$. In addition, there are constants $B,c> 0$,  and $q\geq 1$ such that for all $(x,y)\in X\times Y$ we have both the growth condition \eqref{eq:loss-growth-order-q} and 
    \begin{align} \label{concentration:def-sup-bound-neu} 
    \loss (x,y,t) &\le B\, , &&   t\in [-M,M]\,. 
    \end{align}
   Finally, $\P$ is a distribution on $X\times Y$ for which 
   there exist a Bayes decision function $\fpb:X\to [-M,M]$ and constants
   $\vt\in [0,1]$ and $V\geq B^{2-\vt}$ such that   
\begin{align}\label{concentration:def-var-bound}
\E _{\P} \bigl( L\circ \clippf - L\circ \fpb    \bigr)^{2} & \leq  V \cdot  \bigl(\E _{\P} (L\circ \clippf - L\circ \fpb)    \bigr)^{\vt} 
\end{align}
for all $f\in \sLx 0X$,
where $L\circ g$ denotes the function $(x,y)\mapsto L(x,y,g(x))$.
\end{assumption}

Before we proceed, let us quickly note that 
the growth condition \eqref{eq:loss-growth-order-q} essentially  implies \eqref{concentration:def-sup-bound-neu}, as we have 
$ \loss (x,y,t) \leq B(1+|t|^q) \leq B+ BM^q$ by \eqref{eq:loss-growth-order-q}. 
Consequently, we can always ensure 
\eqref{concentration:def-sup-bound-neu} by suitably increasing the value $B$ in the 
growth condition \eqref{eq:loss-growth-order-q}. 

Moreover, the \emph{variance bound} \eqref{concentration:def-var-bound} is always true for $\vt = 0$. 
The more interesting case $\vt>0$, however,  is more
complex: For example, for the least squares loss and bounded $Y$, 
the variance bound \eqref{concentration:def-var-bound} holds for all $\P$ with the optimal value $\vt=1$, see e.g.~\cite[Example 7.3]{StCh08} and the same is true for its asymmetric versions used for expectile regression, see \cite[Lemma 4]{FaSt19a}.
In addition, under a polynomial behavior of the modulus of convexity of a general, (locally) Lipschitz continuous and  convex $L$, distribution independent variance bounds 
\eqref{concentration:def-var-bound} can be established, see \cite[Lemma 7]{BaJoMc06a}. In particular, this is true 
for e.g.~the two-sided versions of the logistic loss for classification with $\vt = 1$, see also \cite[Proposition 2.9]{Steinwart09a}.
For other losses, however, \eqref{concentration:def-var-bound} only holds for $\vt>0$ under additional assumptions on $\P$. Such cases include the hinge loss,
where Tsybakov's noise condition   ensures \eqref{concentration:def-var-bound}, see 
e.g.~\cite[Theorem 8.24]{StCh08}, and the pinball loss, where a certain 
concentration around the estimated quantile is needed, see \cite{StCh11a}.

With the help of these preparations we can now present our refined analysis for self-regularized 
learning algorithms.

\begin{theorem}\label{thm:learning_rates_tempered_C_minimal_general_form}
Let 
 $L:X\times Y\times \R\to [0,\infty)$ be a loss, $\P$ be a distribution on $X\times Y$, and 
 $F\subset \sLx \infty X$ be a separable RKBS such that 
 Assumptions 
 \assrefshort{loss} and  \assrefshort{approx-error}
are satisfied. In addition, assume that there are constants $\zeta\in (0,1]$ and $\cemb\geq 1$ such that 
\begin{align}\label{eq:embedding-by-norm}
\snorm f_{\Lx \infty \PX} \leq \cemb \snorm f_{\Lx q \PX}^{1-\zeta} \snorm f_F^\zeta\, , \myqquad f\in F,
\end{align}
and in the case $\zeta<1$ we further assume that the growth condition \eqref{eq:loss-growth-order-q} is asymptotically exact in the sense of \eqref{eq:loss-exact-growth-order-q-average}.
Finally, assume that 
 there exist constants $\g\in (0,1)$ and $a\geq 1$ such that 
\begin{equation}\label{concentration:sva-analysis-entropy}
\E_{D_X\sim\P_X^n} e_{i} \bigl(\id:F\to \Lx 2 {\D_X} \bigr)
\leq a \, i^{-\frac 1{2\g}}\, , \qquad\qquad i\geq 1.
\end{equation}
for all $n\geq 1$, where $\D_X$ denotes the empirical measure of the data set $D_X$.
Then there exists a constant $K\geq 1$ such that 
for all   self-regularized learning methods  $(D,t)\mapsto f_{D,t}$   in $F$ and all
    $n\ge 1$ there exists a measurable map $D\mapsto t_D$ such that for all 
     $\tau\ge 1$
    the inequality
    \begin{align}
        \label{eq:main_result_learning_rate}
        \risk \loss P {\clippfo_{D,t_D}} - \RPB\loss  < K \const   \tau \cdot n^{-\alpha} + 6 \const \constt \cdot n^{-\a \cdot \frac {1+\b}{2\b}}
    \end{align}
    holds with probability at least $1 - 6\eul^{-\tau}$ over the sampling of $D\sim \P^n$, where
    \begin{align}\label{eq:general-alpha-in-strong-rates}
        \alpha \equalDef \min \left\{ \frac{2\beta}{\beta(2-\zeta q)+ \zeta q }, \frac{\beta}{\beta(2-\gamma-\vartheta+\gamma\vartheta)+ \gamma} \right\}~.
    \end{align}
    As a consequence, in the cases 
    \begin{align}\label{eq:s-range-exact}
\zeta \leq \frac {2 \g +   2\b(1 - \g -\vt + \vt \g)}{(1-\b) q}\, , \myqquad \mbox{ or even } \myqquad \zeta \leq \frac {2\g }{q} 
\end{align}
    the inequality
    \eqref{eq:main_result_learning_rate} holds with probability at least $1 - 6\eul^{-\tau}$ over the sampling of $D\sim \P^n$ with 
        \begin{align}\label{eq:improved-alpha}
        \alpha \equalDef  \frac{\beta}{\beta(2-\gamma-\vartheta+\gamma\vartheta)+ \gamma}  \, .
    \end{align}
\end{theorem}

The average entropy number bound 
\eqref{concentration:sva-analysis-entropy}
is satisfied if the entropy number bound 
\eqref{eq:sup-entropy-numbers}  holds. 
In particular, all remarks made for 
\eqref{eq:sup-entropy-numbers} in the discussion following  \cref{thm:simple-analysis} also hold true for \eqref{concentration:sva-analysis-entropy}.
Moreover, in the RKHS case,  \eqref{concentration:sva-analysis-entropy} can be linked to the 
 eigenvalue behavior of the associated integral operator, see \cite[Corollary 7.31]{StCh08} in combination with
 \cite[Section 3, in particular (16) and (17)]{Steinwart17a}.
 Namely, we have \eqref{concentration:sva-analysis-entropy} if  these eigenvalues decay with a polynomial rate $i^{-1/\g}$, that is, twice as fast. 
 The latter is a common assumption\footnote{Sometimes this eigenvalue decay assumption is rephrased as an ``effective dimension'' assumption, see e.g.~\cite[Lemma 11]{FiSt20a} for the equivalence of both assumptions.} on RKHS-based learning methods for
 the least squares loss, we refer to \cite{CaDe07a}, where this assumption 
 was first used, and to \cite{LiCe2018,LiRuRoCe2018,BlMu20a,FiSt20a}
as well as the  papers cited therein.

In the case of $F$ being an RKHS and $L$ being the least squares loss, 
Assumption \eqref{eq:embedding-by-norm} was first used implicitly in \cite{MeNe10a}, and based on the latter paper,
explicitly in \cite{StHuSc09b}. Moreover,  a slightly stronger assumption has been employed in \cite{FiSt20a}, again for RKHSs and the least squares loss.
So far, however, it has not been used   for more general $F$ and $L$.
For $\zeta =1$, Assumption \eqref{eq:embedding-by-norm} 
is automatically guaranteed with $\cemb := \snorm{F\hookrightarrow \sLx \infty X}$
by our assumption $F\subset \sLx \infty X$, see \cite[Lemma 23]{ScSt25a}. Moreover, in the case $\zeta\in (0,1)$ it is satisfied, if 
we have an embedding 
\begin{align}\label{eq:interpol-embed}
[\Lx q \PX, F]_{\zeta,1}\hookrightarrow \Lx \infty \PX\, ,
\end{align}
where $[\Lx q \PX, F]_{\zeta,1}$ denotes the interpolation space between $\Lx q \mu$ and  $F$ of the so-called real interpolation method, 
see e.g.~\cite[Definition 1.7 on page 299]{BeSh88}. Indeed, the implication from \eqref{eq:interpol-embed} to 
\eqref{eq:embedding-by-norm} follows from basics of the real interpolation method, see  e.g.~\cite[Proposition 2.10 on page 316]{BeSh88} applied to the intermediate space $[\Lx q \PX, F]_{\zeta,1}$ 
of $\Lx q \PX$ and $F$. 

If there exists a BSF $E$ of ``$s$-times smooth'' functions with 
$F \subset E\subset \sL \infty X$, then we can typically find
a $\zeta \leq 2\g/q$ for which \eqref{eq:embedding-by-norm} holds, see
Lemma \ref{lem:embedding-for-besov-triebel} for details. In particular, 
such spaces $E$ include
(fractional versions of) Sobolev spaces $W^s_p(X)$ and $H^s_p(X)$ 
with $p\geq q$ 
or spaces of $s$-times 
continuously differentiable functions. In the RKHS case of these scales we have $p=2$, and hence all commonly used losses 
are included in this important case.

Finally, in the RKHS case and $q=2$, there is a relationship between the entropy number assumption \eqref{concentration:sva-analysis-entropy}
and \eqref{eq:embedding-by-norm}. Indeed, \cite[Theorem 5.3]{StSc12a} shows that \eqref{eq:embedding-by-norm}
implies an eigenvalue decay $i^{-1/\g}$ for the associated integral operator and all $\g>\zeta$. As discussed above, this eigenvalue decay implies 
\eqref{concentration:sva-analysis-entropy} for the same $\g$. Conversely, if the eigenfunctions of the associated integral operator
are uniformly bounded, then an eigenvalue decay $i^{-1/\g}$ implies \eqref{eq:embedding-by-norm} for $\zeta=\g$ and $q=2$, see \cite{MeNe10a}. 
Finally, \cite{BiStXXa} presents a more detailed analysis of 
 \eqref{eq:interpol-embed}, while 
Lemmas \ref{lem:entropy-for-besov-triebel} and \ref{lem:embedding-for-besov-triebel} show that for Sobolev-type RKHS we may choose $\zeta = \g$ as long as $q\leq 2$.

Let us now compare the rates in \eqref{eq:main_result_learning_rate} 
with known rates. To this end, we first treat the Hilbert space case. 
For the least squares loss, we then have $q=2$ and $\vt = 1$. Consequently,
the computation of $\a$ reduces to 
    \begin{align*}
        \alpha = \min \left\{ \b , \frac{\beta}{\beta +  \gamma} \right\}\, , \myqquad \text{respectively } \myqquad 
        \alpha =  \frac{\beta}{\beta +  \gamma}  \, , 
    \end{align*}
    where for the second case we additionally assume  $\zeta\leq \g/(1-\b)$ or  the even stronger 
   $\zeta \leq \g$.
For $\b+\g\geq 1$ or $\zeta\leq \g/(1-\b)$  we thus have $\a =  \frac{\beta}{\beta +  \gamma}$
and therefore the rates in  \eqref{eq:main_result_learning_rate} equal 
the asymptotically minmax optimal rates derived in \cite{StHuSc09b} for RERM, see also 
\cite{FiSt20a} for a more detailed 
discussion and references to other results.
Moreover, for the pinball loss we have $q=1$, and for $\zeta=1$
the rates \eqref{eq:main_result_learning_rate} 
equal those obtained in \cite{StCh11a} for RERM, where we note that  
$\vt$ is distribution dependent and the rates are again in some cases known to be 
asymptotically minmax optimal. On the other hand, for some values of
$\zeta$ and $\b$, Theorem \ref{thm:learning_rates_tempered_C_minimal_general_form}
improves the rates of \cite{StCh11a}. In particular, for the strongest
variance bounds, that is $\vt = 1$, this happens to be true for $\zeta \leq \g$
as one quickly verifies with the help of the first inequality in 
\eqref{eq:s-range-exact}.
Finally, \cite{PaGr19a} considers IERM in RKHSs with bounded kernels
and the least squares loss. In our language they 
further assume $\zeta=1$ and ``$\g=1$'' and obtain the exponent 
\begin{align}\label{eq:gruenewaelder}
\a = \frac \b {\b+1}\,. 
\end{align}
 Here, we wrote ``$\g=1$'' as $H\subset \sLx \infty X$ 
implies, up to a logarithmic factor $\ln n$,  an entropy number bound \eqref{concentration:sva-analysis-entropy} for $\g=1$, 
see e.g.~\cite[Lemma 7.21]{StCh08}, and adapting Theorem \ref{thm:learning_rates_tempered_C_minimal_general_form}
to this situation gives the rates for $\g=1$ up to a logarithmic factor, see 
\cite[Theorem 7.22]{StCh08} for the key arguments.
In this sense, \eqref{eq:gruenewaelder} is better by a logarithmic 
factor.
 
In view of \cref{thm:gradient_descent_in_RKHS}
it also makes sense to compare to known rates for gradient descent
in RKHSs. Here, the most general rates were obtained in \cite[Theorem 8]{LiRoZh16a}.
Namely, for convex and smooth  losses $L$ satisfying the growth condition \eqref{eq:loss-growth-order-q}
and the variance bound \eqref{concentration:def-var-bound} and RKHSs satisfying %
\assref{approx-error}
 and the entropy number bound \eqref{concentration:sva-analysis-entropy} the authors obtain the rate
 $n^{-\a}$, where
\begin{align*}
\a = \frac {2\b}{(2\b+1)(2-\vt + \g \vt) + (q-1)(1+\g) }
\end{align*}
for gradient descent with step sizes $\eta_k = \eta_0 k^{-\theta}$ with $\theta \in [0,1)$ and sufficiently small $\eta_0$
and \emph{distribution dependent} early stopping.
Note that our rates are \emph{always} better than those of \cite[Theorem 8]{LiRoZh16a}.
However, it seems fair to say that \cite[Theorem 8]{LiRoZh16a} does not involve the clipping operation.
More recently, 
\cite[Theorem 1]{HuLe22a} has substantially improved the findings of \cite{LiRoZh16a} for the step sizes mentioned above.
In a nutshell, they consider the same assumptions as our Theorem \ref{thm:learning_rates_tempered_C_minimal_general_form} but only with $\zeta=1$, and 
their rates equal ours in the case $\zeta=1$. Therefore, their rates are worse  for some combinations of $\b, q, \vt$, and $\g$
as soon as $\zeta<1$ is possible. 
Finally, for constant step sizes
\cite[Theorem 2]{StMuRo23a} also improved the rates of \cite[Theorem 8]{LiRoZh16a} to
\begin{align*}
\frac{\ln n}{\sqrt n}
\end{align*}
in the case of $\b=1$, $\vt= 0$, ``$\g=1$'', $\zeta=1$, and $L$ being (locally) Lipschitz continuous. Here ``$\g=1$'' has the same meaning as around \eqref{eq:gruenewaelder}.
As discussed after \eqref{eq:gruenewaelder} their rate therefore coincides 
with a rate we would obtain by a revised analysis.
On the other hand, our rates hold under significantly more general conditions. Finally,
\cite[Table 1]{LiRoZh16a} contains a comparison to additional results for related methods
in the case
$\b=1$, $\vt= 0$, and in the special case of $L$ being the least squares loss more results
such as \cite{YaRoCa07a,RaWaYu14a} are known. However, neither of these results obtain rates better than ours. 
In summary, it therefore seems fair to say that for RKHS-based gradient descent
our analysis  generalizes and improves existing results.

Let us finally compare to some existing results in the non-RKHS case.
Here we begin with ``$\ell_1$-regularization'', in which a countable and
uniformly bounded subset $\ca C\subset \sLx \infty X$ is fixed and 
a separable BSF is constructed by 
\begin{align*}
F  := \Bigl\{ \sum_{i=1}^\infty w_i h_i : (w_i)\in \ell_1, h_i \in \ca C  \Bigr\}\, , 
\end{align*}
where the norm of $F$ is 
the minimal $\ell_1$-norm of coefficients $(w_i)$,
see e.g.~\cite[Lemma 2.1]{Steinwart09a}. In particular, we have $F\subset \sLx \infty X$.
In such a setting, \cite{BlLuVa03a} derives rates $n^{-\a}$ with  
 \begin{align*}
\a = \frac{1}{1+\g} \cdot \frac{2\b}{\b(2-q) + q}
\end{align*}
for an \emph{unclipped}  form of ``structured risk minimization'' that is based upon IERM in $F$ and strictly convex, margin-based losses $L$
satisfying some additional technical assumptions ensuring  $\vt=1$   for all distributions $\P$, see 
  Supplement \ref{sec:comparison-blanchard} for details on this derivation.
  For  clippable\footnote{This assumption can be omitted, if a two-sided version of $L$ in the sense of \cite[Definition 2.6]{Steinwart09a} is used.} $L$,
  these rates are worse than those of 
   Theorem \ref{thm:learning_rates_tempered_C_minimal_general_form} whenever $0<\b<1$, while for $\b=1$ both 
  rates coincide, see again Supplement \ref{sec:comparison-blanchard}. Finally,  \cite{BlLuVa03a}
  provides an example for $\ca C$, in which these rates are minmax optimal for $\b=1$. 
In addition, in the case $\zeta=q=1$ and two-sided versions of the 
 margin-based losses considered in \cite{BlLuVa03a} the rates of Theorem \ref{thm:learning_rates_tempered_C_minimal_general_form} for RERM
 have already been found in \cite[Theorem 2.10]{Steinwart09a}. In fact, the latter result actually considers
 general, separable BSFs $F\subset \sLx \infty X$ 
 and RERM 
 with   $p=1$.
 
 Let us now consider RERM over  Besov spaces $F= \besov spq \Rd$, where $s>d/p$ is required to 
 ensure that $F$ is a BSF. Here we first recall that these spaces can be described by wavelets
and weighted sequence norms of the corresponding coefficients, see e.g.~\cite{Loustau09a}.
For $p=q$
this weighted sequence norm has a computationally attractive form if its $p$th power is taken  in \eqref{eq:rerm-opt-prob-eps}, see \cite[Sections 1.3 and 2.1]{Loustau09a}. In particular, for $p=q=1$, we   obtain a (weighted) $\ell_1$-regularization.
 Now, for $p\in [1,2]$,  the hinge loss,   and distributions $\P$ with $\vt = 1$ \firstt{and noise levels bounded away from zero,}  \cite{Loustau09a} derives rates $n^{-\a}$ 
 with exponent 
\begin{align*}
\a = \frac {2\b}{1+\b} \cdot \frac {2u_p}{2u_p+d} 
\end{align*}
with $u_p := s + d  (  1/2 - 1/p )$,
while Theorem \ref{thm:learning_rates_tempered_C_minimal_general_form}
gives the exponent $ \a = \frac{2 \beta s}{2\beta s + d} $, see Supplement \ref{sec:Loustau} for both derivations. In the case  $\b=1$ and $p=2$ 
both rates coincide, while in all other cases 
our rates are  better, see again  Supplement \ref{sec:Loustau}.
Moreover, if one considers $F= \besov spq \Rd$ with the least squares loss and $p\geq 2$, the rates 
of Theorem \ref{thm:learning_rates_tempered_C_minimal_general_form} coincide with known minmax optimal rates,
and the same is true in the case $p\in [1,2)$, if the target function $\fpb$ is not too irregular, see Supplement \ref{sec:besov-rates}
for precise statements. In addition, this supplement also derives rates for Besov spaces and Lipschitz continuous losses.

\is{I looked into the papers \emph{The Statistical Complexity of Early-Stopped Mirror
Descent}  and \emph{Mirror Descent on Reproducing Kernel Banach Spaces
} to get an idea on the SoTA of the statistical analysis of MD. I was not successful, but the following 3 papers could potentially be interesting for the intro:  \emph{arXiv:2003.07802}, \emph{Connecting Optimization and Regularization Paths}, \emph{ 	arXiv:1806.00952}. What we can say about 
\emph{The Statistical Complexity of Early-Stopped Mirror
Descent} is that they cite a result from their [30] that relates the  excess risk to some Rademacher complexity, but no rates are inferred, in particular the approximation error is not considered, see their Thm. 1. Also their Definition 2 is somewhat similar to our self-regularization inequality with $\const = 1$ and $\constt = \sqrt{\e/\g}$ but with the $F$-norm replaced by the empirical $\Lx 2 D$-norm, where $\e,\g$ are free parameters and their sets of dunctions $\ca F$ are scaled balls in $F$.}

\begin{table}[!htb]
\small\centering%
\setlength{\extrarowheight}{2pt}
\begin{tabular}{|*{5}{l|}}
\hline
Article & Algorithm & Loss & Add.~Assumptions &   Exponent $\a$ in rates $n^{-\a}$ \\
\hline
\cite[Ch.~7.4]{StCh08}  & $H$-RERM &LS, thus $\vt = 1$ &  $p=2$, $\zeta=1$ & $\min\{ \b, \frac{\b}{\b + \g}  \}$ \\
\hline
\cite{StHuSc09b} &  $H$-RERM & LS, thus $\vt = 1$  &   --- & $\min  \{  \frac{\beta}{\b + \zeta - \b \zeta  }, \frac{\beta}{\beta + \gamma} \}$  \\ \hline
\cite[Ch.~7.4]{StCh08} &  $H$-RERM & Lipschitz   &  $p=2$, $\zeta=1$ & $ \min\{ \frac{2\b }{\b + 1}, \frac{\b}{\b (2-\g-\vt+\g \vt) + \g}  \}$ \\ \hline
\cite[Thm.~7]{PaGr19a} &  $H$-IERM & LS, thus $\vt = 1$   &  $\g=1$, $\zeta=1$ & $  \frac{\b }{\b + 1}$ \\ \hline
\cite{BlLuVa03a} & $\ell_1$-IERM   &  \makecell{margin-based,\\ $C^2$+  s.t.~$\vt = 1$}  & $p=1$, $\zeta=1$  & \worse{$\frac{2\b}{\b(2-q) + q}\cdot  \frac{1}{1+\g} $}   \\ \hline 
\cite{Steinwart09a} & $F$-RERM   &  \makecell{margin-based,\\ $C^2$+  s.t.~$\vt = 1$}  & $p=1$, $\zeta=1$   & $ \min \left\{ \frac{2\beta}{\beta(2-q)+q }, \frac{\beta}{\beta + \gamma} \right\}$   \\ \hline 
\cite{Loustau09a} & \makecell{ $F$-RERM \\ $F$ Besov}   &  hinge  & \makecell{ $p=1$, $\zeta=1$, \\ $\P$ with $\vt=1$,  }  & \worse{   $ \leq \frac {2\b}{\b+1} \cdot \frac 1 {1+\g}$}   \\ \hline 
\cite[Thm.~8]{LiRoZh16a} & $H$-GD &  $M$-smooth & \makecell{$\zeta=1$, $\eta_k = \eta_0 k^{-\theta}$\\ with $\theta \in [0,1)$} & \worse{$\frac {2\b}{(2\b+1)(2-\vt + \g \vt) + (q-1)(1+\g) }$} \\ \hline 
\cite[Thm.~2]{StMuRo23a} & $H$-GD &  $M$-smooth & \makecell{ $\vt = 0$, $\g=1$, $\b = 1$, \\ $\zeta=1$, $\eta_k \equiv \eta$ } & $ \frac 12*$ \\ \hline 
\cite[Thm.~1]{HuLe22a}  & $H$-GD  & $M$-smooth &  \makecell{$\zeta=1$, $\eta_k = \eta_0 k^{-\theta}$ \\ with $\theta \in [0,1)$} &  $ \min  \{ \frac{2\beta}{\beta(2-q)+q }, \frac{\beta}{\beta(2-\gamma-\vartheta+\gamma\vartheta)+ \gamma}  \}$ \\ \hline 
\end{tabular}
\caption{$C^2$+ stands for twice continuously differentiable losses satisfying some additional properties. Moreover,  $ \frac 12*$ means that an additional logarithmic factor occurs.}
      \label{table:comparison_of_rates}
\end{table} 
 
\color{black}

\section{Proofs}\label{sec:proofs}

\subsection{Proofs for Section \ref{sec:self-reg} --  Self-Regularization}\label{sec:self-reg-proofs}

\firstt{
\paragraph*{Approximate empirical risk minimization}
}

\begin{proof}[Proof of \cref{ex:rerm-is--self-reg}]
Obviously, conditions \emph{a)} and \emph{b)} of the definition of self-regularized learning algorithms are 
satisfied by assumption. In addtion, $\emph{c)}$ is satisfied for $t_0:= \infty$. %

We first note that for $\lb\in (0,\infty)$ the map $D\mapsto g_{D,\lb}$ is measurable by a simple 
adaptation of \cite[Lemma 6.23]{StCh08} and the assumed separability of $F$. Moreover, for $\lb = \infty$,
the measurability is obvious, and we also have both 
$\RD\loss {g_{D,\infty}} = \RD\loss 0$ and  $\snorm{g_{D,\infty}} = 0 = \constt$.

To establish 
\eqref{eq:self-reg:richness} we  note that there is nothing to prove in the case $\RDxB \loss F = \RD\loss 0$. 
In the remaining case $\RDxB \loss F < \RD\loss 0$
we assume   that \eqref{eq:self-reg:richness} is false, that is,
 there is an $\varepsilon>0$ with 
\begin{align*}
\RD\loss {g_{D,\lb}} \geq \RDxB \loss F + 3\e 
\end{align*}
for all $\lb>0$. In addition, there exists a
 $g\in F$ with $g\neq 0$ and $\RD\loss g < \RDxB \loss F + \varepsilon$.  
We now pick a $\lb>0$ with both  $\lb <  \e  \snorm g_F^{-p}$ and $\lb < \e \epsilon^{-1}$. We then find a contradiction by 
\begin{align*}
\lb \snorm{g}_F^p + \RD\loss g +\lb \epsilon
< \RDxB \loss F  + 3\e 
\leq \RD\loss {g_{D,\lb}} 
\leq  \inf_{f\in F} \bigl(\lb \snorm f_F^p  + \RD\loss f   \bigr) + \lb \epsilon\, . 
\end{align*}

Let us finally verify   \eqref{eq:norm_bound_selfreg}. Clearly, for $\lb=\infty$ there is 
nothing to prove. For $\lb\in (0,\infty)$ we fix a $g\in F$ with $\RD\loss g = \RD\loss {g_{D,\lb}}$.
Now, if \eqref{eq:norm_bound_selfreg} was wrong for $\const = 1$
 and $\constt = \epsilon^{1/p}$ we would have  $\snorm{g_{D,\lb}}_F^p > (\snorm g_F + \epsilon^{1/p})^p \geq \snorm g_F^p + \epsilon$. This in turn would give 
 \begin{align*}
 \lb \snorm{g}_F^p + \RD\loss g +\lb \epsilon  
 <  \lb \snorm{g_{D,\lb}}_F^p + \RD\loss {g_{D,\lb}}
 \leq  \inf_{f\in F} \bigl(\lb \snorm f_F^p  + \RD\loss f   \bigr) + \lb \epsilon\, ,
 \end{align*}
 which again gives a contradiction.
\end{proof}

\firstt{
\paragraph*{Mirror descent}
\begin{lemma}\label{lem:mirror-map-cmul}
    Let $F$ be an RKBS and $\psi:F\to \R_{\ge 0}$ be a mirror map.
    Define $C_\psi$
    as in  \eqref{eq:c-psi}.
    If $f,g\in F$ fulfill $\BM{\psi} gf \le \BM \psi g 0 $, then
    \begin{align*}
        \norm{f}_F \le C_\psi \norm{g}_F.
    \end{align*}
\end{lemma}
\begin{proof}
    For $f=0$ there is nothing to prove. Let $f\ne 0$ and $g\in F$ such that  $\BM{\psi} gf \le \BM \psi g 0$, and note that $\BM \psi g0 =\psi (g)$ since $\psi$ is a mirror map and especially fulfills $\psi(0)=0 $ and $\nabla \psi(0)=0$.
    That is, we have 
    \[ 
    \psi(g) -\psi(f) -\scal{\nabla \psi(f)}{g-f}_{F',F} \le \psi(g)~.
    \] 
    Rearranging this, we obtain
    \begin{align*}
        \scal{\nabla \psi(f)}{f}_{F',F} - \psi(f) \le \scal{\nabla \psi(f)}{g}_{F',F}
    \end{align*}
    Note that by strict convexity we have $0< \BM\psi 0f =\scal{\nabla\psi(f)}{f}_{F',F} -\psi(f)$, since $\psi(0)=0$, which also yields $\norm{\nabla \psi(f)}_{F'}>0$.
    Hence we do not divide by zero in the following.
    By construction of $C_\psi$ we have 
    \begin{align*}
        \norm f_F \le \frac{C_\psi \big({\scal{\nabla\psi(f)}{f}}_{F',F}-\psi(f)\big)}{\norm{\nabla\psi(f)}_{F'}} \le C_\psi \frac{\scal{\nabla\psi(f)}{g}_{F',F}}{\norm{\nabla\psi(f)}_{F'}} \le C_\psi \norm{g}_F.
    \end{align*}
\end{proof}

\begin{lemma}\label{lem:mirror-descent-fejer}
    \ms{consider: should this move to the appendix?}
    Let $F$ be an RKBS, $\psi,\riskkk\cdot:F\to\R$ be convex
    maps such that $\riskkk\cdot$ is $\MMbeta$-smooth relative to $\psi$.
    Let $\ffun{t}$ be the mirror descent \eqref{eq:mirror-descent-dual-interpolation} for step sizes $(\eta_k)$ with $0<\eta_k\le 1/\MMbeta$ for all $k\in\N_0$. 
    \begin{enumerate}
        \item  If furthermore $\sum_{k=0}^\infty \eta_k =\infty$, then we have 
        \begin{align} \label{eq:mirror-descent-asymptotic-interpolation} 
            \lim_{k\to\infty} \riskkk{\ffun k}= \inf_{f\in F}\riskkk{f}~.
        \end{align}
        \item
        For all $t\ge 0$ and all $ g\in F$ with $ \riskkk{g}\le \riskkk{\ffun{t}}$ we have the Bregman-divergence style Fej\'er monotonicity 
        \begin{align}
            \label{eq:mirror-descent-bm-fejer}
            \BM \psi g {\ffun{t}} \le \BM \psi g{0} ~.
        \end{align}
    \end{enumerate}
\end{lemma}
\begin{proof}
    We begin with some preliminary calculations. To this end, let $j\in\N_0$.
    Since $\riskkk{\cdot}$ is $\MMbeta$-smooth relative to $\psi$  the definition of $D_{\riskkk\cdot}$ gives
    \begin{align}
        \label{eq:mirror-descent-descending-proof-beta-smoothness-application}
        \riskkk{\ffun{j+1}} =&\ \riskkk{\ffun{j}} +\scal{\nabla\riskkk {\ffun{j}} } {\ffun{j+1}-\ffun{j}}  + \BM{\riskkk{\cdot}}{\ffun{j+1}}{\ffun{j}}
        \\
        \nonumber 
        \le &\ \riskkk{\ffun{j}} + \scal{\nabla\riskkk{\ffun{j}}}{\ffun{j+1}-\ffun{j}}+\MMbeta \BM{\psi}{\ffun{j+1}}{\ffun{j}}~.
    \end{align}
    Combining the mirror descent equation \eqref{eq:mirror-descent-dual-interpolation} 
    $\nabla \psi(f_{j+s}) = \nabla\psi(f_j) -s \eta_j\nabla\calR(f_j)$
    with the Bregman symmetrization identity \eqref{eq:bregman-symmetrization}  yields
    \begin{align*}
        & \quad 
        \scal{\nabla\riskkk{\ffun{j}} }{\ffun{j+1}-\ffun{j}} 
        = 
        \frac{-1}{\eta_j} \big(\scal{\nabla\psi(\ffun{j})-\nabla\psi(\ffun{j+1})}{\ffun{j}-\ffun{j+1}}\big)
        \\ &= 
        \frac{-1}{\eta_j} \big(\BM\psi{\ffun{j}}{\ffun{j+1}}+\BM{\psi}{\ffun{j+1}}{\ffun{j}}\big)
        ~.
    \end{align*}
    With \eqref{eq:mirror-descent-descending-proof-beta-smoothness-application}
    we obtain 
    \begin{align}
        \label{eq:mirror-descent-proof-step}
        \eta_j \riskkk{\ffun{j}} - \eta_j\riskkk{\ffun{j+1}} \ge \BM{\psi}{\ffun{j}}{\ffun{j+1}}+\big(1-\eta_j\MMbeta\big) \BM{\psi}{\ffun{j+1}} {\ffun{j}}~.
    \end{align}
    Recall that the Bregman-divergence is non-negative since $\psi$ is convex.
    Since $\eta_j\le 1/\MMbeta$ we conclude 
    the key estimate 
    \begin{align}
        \label{eq:mirror-descent-key}
        0\le \BM\psi{\ffun{j}}{\ffun{j+1}} \le \eta_j (\riskkk{\ffun{j}} - \riskkk{\ffun{j+1}})~,
    \end{align}
    which especially implies that $\riskkk{\ffun{j}}$ decays monotonously in $j$.

    Now, let $g\in F$. 
    We use the general Bregman three-point identity \eqref{eq:bregman-three-point} 
    \begin{align*}
        \BM{\psi} a c = \BM \psi ab + \BM\psi bc + \scal{\nabla\psi(b) -\nabla \psi(c)}{a-b}
    \end{align*}
    with $a\equalDef g$, $b\equalDef \ffun{j}$ and $ c\equalDef \ffun{j+1}$ 
    as well as the mirror descent equation \eqref{eq:mirror-descent-dual-interpolation}
    to obtain
    \begin{align}
        \nonumber
        \BM \psi{g}{\ffun{j+1}} &= \BM\psi{g}{\ffun{j} } + \BM\psi{\ffun{j}}{\ffun{j+1}} +\eta_j \scal{\nabla\riskkk{\ffun{j}}}{g-\ffun{j}}
        \\
        &\le \BM\psi{g}{\ffun{j} } + \BM\psi{\ffun{j}}{\ffun{j+1}} +\eta_j(\riskkk{g}-\riskkk{\ffun {j}})~,
        \label{eq:MD-pre-central-estimate}
    \end{align}
    by simply applying convexity of $\riskkk{\cdot}$ in the last step. 

    \begin{enumerate}
        \item 
    Let $g\in F$.
    Rearranging \eqref{eq:MD-pre-central-estimate} we obtain for any $j\in \N_0$ the inequality
    \begin{align}
        \nonumber\eta_j(\riskkk{\ffun {j}}- \riskkk{g})
        &\le 
        \BM\psi{g}{\ffun{j} } + \BM\psi{\ffun{j}}{\ffun{j+1}} -\BM \psi{g}{\ffun{j+1}} 
        \\ 
        &\le 
        \label{eq:mirror-descent-step-comparison}\BM\psi{g}{\ffun{j} }  -\BM \psi{g}{\ffun{j+1}} ~.
    \end{align}
    Let $k\in\N$. The above estimate directly yields
    \begin{align*}
        \sum_{j=0}^{k-1} \eta_j (\riskkk{\ffun k} - \riskkk {g})
        &\le 
        \sum_{j=0}^{k-1} \eta_j(\riskkk{\ffun {j}}- \riskkk{g}) 
        \le 
        \BM\psi{g}{ 0 }  -\BM \psi{g}{\ffun{k}}
        \le \BM\psi{g}{0}
        ~,
    \end{align*}
    where the first step is simply by monotonic decay of $\riskkk{\ffun{j}}$, see \eqref{eq:mirror-descent-key}.
    We infer
    \begin{align*}
        \riskkk{\ffun k} - \riskkk {g} \le \frac{\BM\psi{g}{0}}{\sum_{j=0}^{k-1} \eta_j} \overset{k\to\infty}\longrightarrow 0~,
    \end{align*}
    and \Eqref{eq:mirror-descent-asymptotic-interpolation} follows since $g$ is an arbitrary function in $F$.
    \item
    Let $t\in \R_{\ge 0}$. We justify that we can assume $t\in\N_0$ without loss of generality. To this end, assume $t= k + s$ for $k\in \N_0$ and $s\in(0,1)$. Then, by \eqref{eq:mirror-descent-dual-interpolation} we have for $\tilde \eta_k \equalDef s\eta_k$ the identity 
    \begin{align*}
        \ffun{k+s} = (\nabla\psi)^{-1} \left(\nabla\psi\ffun{k} - s\eta_k\nabla\riskkk{\ffun {k}}\right)
        = (\nabla\psi)^{-1} \left(\nabla\psi\ffun{k} - \tilde \eta_k\nabla\riskkk{\ffun{k}}\right)~,
    \end{align*}
    that is, $\ffun{t}$ equals 
    $\tilde{f}_{k+1}$ for mirror descent 
    $(\tilde{f}_{t})_{t\ge 0}$ where we only replace the step size $\eta_k$ by $\tilde \eta_k$. Note that $\tilde \eta_k$ still fulfills $0<\tilde \eta_k \le 1/\MMbeta$.
    In order to check \Eqref{eq:mirror-descent-bm-fejer} it hence suffices to assume $t\in\N_0$.

    Now, let $k\ge j+1\ge 0$ be integers and let $g\in F$ with $\riskkk{g}\le \riskkk {\ffun {k}} \le \riskkk{\ffun{j+1}}$, where the second inequality is the monotonic decay of $\riskkk{\ffun{j}}$.
    Then, from \cref{eq:mirror-descent-step-comparison}
    we obtain 
    \begin{align*}
        \BM \psi{g}{\ffun{j+1}} 
        &\le 
        \BM\psi {g}{\ffun{j}}~.
    \end{align*}
    Iterating over $j=k,\dots,0$ we obtain $\BM{\psi}{g}{\ffun{k+1}}
    \le \dots \le \BM\psi{g}{0}$ as claimed.
\end{enumerate}
\end{proof}

\begin{proof}[Proof of \cref{thm:mirror-descent-selfreg}]
    The technical parts a), b) and c) of \cref{def:self-regularization} are trivially fulfilled. 
    By \cref{prop:m-smoothness-of-risks} the empirical risk $\riskk \cdot$ is $\MMbeta$-smooth relative to $\psi$ for all data sets.
    \Cref{lem:mirror-descent-fejer} a) gives the asymptotic interpolation in \cref{def:self-regularization}. 
    Finally, the self-regularization inequality \eqref{eq:norm_bound_selfreg}  follows from \cref{lem:mirror-descent-fejer} b) and \cref{lem:mirror-map-cmul}.
\end{proof}

The following lemma is used to establish \cref{ex:mirror-descent-in-lp} based on \cref{thm:mirror-descent-selfreg}.
\begin{lemma}\label{lem:mirror-descent-glorious-example-preparations}
    Let $F$
    be an RKBS over $X$.
    Let $A:F\to F$ be an isomorphism.
    Define $$\psi(f)\equalDef \norm{Af}^2/2$$
    and assume that $\psi$ is a mirror map. 
    \begin{enumerate}
        \item If $F$ fulfills the mid-point inequality 
        \begin{align}
            \label{eq:mid-point}
            \left\|\frac{f+g}{2}\right\|^2 \le \frac{\norm f^2 + \norm g^2}{2}-\frac{C_F}{4}\norm{f-g}^2~.
        \end{align}
        for some $C_F>0$ and all $f,g\in F$, then $d_\psi$ in \Eqref{eq:constant-comparing-bmdiv-and-squarednorm} is given by $d_\psi = C_F \aminsq$.
        \item The constant $C_\psi$ in \Eqref{eq:c-psi} fulfills $C_\psi \le 2\norm A \amininv$.
    \end{enumerate}
    
\end{lemma}

\begin{proof}
    \begin{enumerate}
    \item     
    Let $f,g\in F$ and denote $\delta\equalDef f-g$. Define for $\lambda\in \R$ the auxiliary functions 
    \begin{align*}
        h(\lambda) &\equalDef 2\psi(g+\lambda \delta) =\norm{A(g+\lambda\delta)}^2
        \\ 
        \tilde h(\lambda) &\equalDef h(\lambda) - C_F\lambda^2 2\psi(\delta) = \norm{A(g+\lambda\delta)}^2 -C_F\lambda^2\norm{A\delta}^2~.
    \end{align*}
    Let $\lambda_1,\lambda_2\in\R$.
    We apply \Eqref{eq:mid-point} to obtain
    \begin{align*}
        &h\left(\frac{\lambda_1+\lambda_2}{2}\right) = \norm{A \left(g +\frac{\lambda_1\delta+\lambda_2 \delta}{2}\right)}^2
        \\ \le\ & \frac{h(\lambda_1)+h(\lambda_2)}{2} -\frac{C_F}{4} \norm{A\left(\left(g+\lambda_1 \delta\right)- \left(g+\lambda_2 \delta\right)\right)}^2
        \\=\ & \frac{h(\lambda_1)+h(\lambda_2)}{2} -\frac{C_F}{4} (\lambda_1-\lambda_2)^2 \norm{A\delta}^2~.
    \end{align*}
    Hence, we have, using the identity $2 \lambda_1^2 + 2 \lambda_2^2 = (\lambda_1+\lambda_2)^2 + (\lambda_1-\lambda_2)^2$, 
    the inequality
    \begin{align*}
        & \tilde h\left(\frac{\lambda_1+\lambda_2}{2}\right) 
        = h\left(\frac{\lambda_1+\lambda_2}{2}\right) - \frac{C_F}{4} (\lambda_1+\lambda_2)^2 \norm{A\delta}^2
        \\ 
        \le \ & \frac{h(\lambda_1)+h(\lambda_2)}{2} - \frac{C_F}{2}(\lambda_1^2+\lambda_2^2)\norm{A\delta}^2
        \\
        =\ & \frac{\tilde h(\lambda_1) +\tilde h(\lambda_2)}{2}~,
    \end{align*}
    that is, $\tilde h$ is mid-point convex.
    Clearly, $\tilde h$ is furthermore continuous, and hence $\tilde h$ is convex by \cite[Lemma A.6.17]{StCh08}.
    Convexity yields
    \begin{align*}
        &\tilde h(1) \ge \tilde h(0) + \tilde h'(0) =h(0) +h'(0)= h(0) + 2\scal{\nabla \psi(g)}{\delta}
        \\ =\ & 2\psi(g) +2\scal{\nabla\psi(g)}{\delta}~,
    \end{align*}
    using the differentiability of $\psi$.
    Altogether, we obtain 
    \begin{align*}
        2\psi(f) - C_F \norm{A\delta}^2 =\tilde h(1) \ge 2\psi(g)+ \scal{\nabla 2\psi(g)}{\delta}~,
    \end{align*}
    which we rearrange to 
    \begin{align*}
        \BM\psi fg  \ge C_F\norm{A\delta}^2/2 \ge C_F \norm{A^{-1}}^{-2}\norm{f-g}^2/2~.
    \end{align*}
    \item 
    Since $\abs{(a+b)^2 -a^2} \le b^2 + 2\abs a \cdot \abs b$ holds for all $a,b\in \R$, we have for any $f,g\in F$ the estimate
    \begin{align*}
        \abs{\scal{\nabla\psi(f)}{g}}
        &=\lim_{\varepsilon\to 0}\frac{\abs{\norm{A(f+\varepsilon g)}^2 -\norm{Af}^2}}{2\abs{\varepsilon}}
        \\
        &=\lim_{\varepsilon\to 0} \frac{\abs{\big(\norm{Af} + \big(\norm{A(f+\varepsilon g)}-\norm{Af}\big)\big)^2 -\norm{Af}^2}}{2\abs{\varepsilon}}
        \\
        &=\lim_{\varepsilon\to 0} \frac{\abs{ 2 \norm{Af} \left(\norm{A(f+\varepsilon g)} -\norm{Af} \right)+ \left(\norm{A(f+\varepsilon g)}-\norm{Af}\right)^2 }}{2\abs{\varepsilon}}
        \\
        &\le \lim_{\varepsilon\to 0} \frac{\abs{2\varepsilon \norm{Af}\norm{Ag}+\varepsilon^2 \norm{Ag}^2}}{2\abs{\varepsilon}}
        \\
        &= \norm{Af}\norm{Ag}~,
    \end{align*}
    hence $\norm{\nabla\psi(f)}_{F'}\le \norm A \norm{Af}$.
    Furthermore clearly $\scal{\nabla\psi(f)}{f}= \norm{Af}^2$ holds.
    Altogether, we conclude
    \begin{align*}
         \frac {\norm{\nabla\psi(f)}_{F'}\norm f_F}{\scal{\nabla \psi(f)}{f}_{F',F}-\psi(f)}
         &\le \frac{\norm A \cdot \norm f\cdot  \norm{Af}}{\norm{Af}^2/2}
         = 2 \frac{\norm A \cdot \norm f}{\norm {Af}} \le 2 \norm A \amininv~.
    \end{align*}
\end{enumerate}
\end{proof}

\begin{proof}[Proof of \cref{ex:mirror-descent-in-lp}]
    {\cref{lem:checking-mirror-map-requirements}} shows that $\psi$ is a mirror map.
    Since $p\le 2$, \cite[Proposition 3]{BallCarlenLieb1994} yields, after the simple substitution, 
    $x=(f+g)/2, y=(f-g)/2$, 
    the mid-point inequality \eqref{eq:mid-point} with $C_F =p-1$,
    we can apply \cref{lem:mirror-descent-glorious-example-preparations} a) and obtain \eqref{eq:constant-comparing-bmdiv-and-squarednorm} with $d_\psi = (p-1)\aminsq$.
    The step sizes fulfill the requirements of \cref{thm:mirror-descent-selfreg}.
    
    Applying \cref{lem:mirror-descent-glorious-example-preparations} b) gives $C_\psi \le 2 \norm A\amininv$ in \eqref{eq:c-psi}, and \cref{thm:mirror-descent-selfreg} yields self-regularization with $\cadd=0$ and $\cmul = 2\norm A\amininv$.
\end{proof}

}

\firstt{

\paragraph*{Proximal point algorithm}

\begin{proof}[Proof of \cref{thm:ppa-self-reg}]
    By \cref{prop:m-smoothness-of-risks}
    the technical properties a), b) and c) of \cref{def:self-regularization} are trivial, and d) is shown to be equivalent to $\sum_{k=0}^\infty \eta_k=\infty$ in \cite{Gu91}. 

    It remains to show the self-regularization inequality \eqref{eq:norm_bound_selfreg}, i.e.\ for any $t\in\R_{\ge 0}$ and any $f^\dagger\in H$ with  $\riskk {f^\dagger} = \riskk{\fdfun{t}}$ 
    we have
    \begin{align}
        \label{eq:ppa-self-reg}
        \norm{\fdfun{t}}\le \frac{2(2-\eta_0\MMbeta)}{1-\eta_0\MMbeta}\norm{f^\dagger}~.
    \end{align}
    \Eqref{eq:ppa-step} is equivalent to 
    \[ \fdfun{k+1}= \argmin_{f\in H} \Big(\risk \loss D f + \frac 1{2\eta_k} \norm{f-\fdfun{k}}^2_H\Big)~. \]
    This implies the risk-decay $\risk\loss D{\fdfun{k+1}}\le \risk\loss D{\fdfun{k}}$ for all $k\in\N_0$.

    Firstly, we show that for any $k\in\N_0$ and any $f^\dagger \in H$ with $ \riskk {f^\dagger} = \riskk {\fdfun{k}}$ we have     
    \[\norm{\fdfun{k}-f^\dagger}\le \norm {f^\dagger}~.\]
    As this is trivially satisfied for $k=0$, let $k\ge 1$. Consider $j\in\{0,\dots,k-1\}$. 
    By the decay of the risk we have $\riskk {f^\dagger} \le \riskk{\fdfun{k}}\le \riskk{\fdfun{j+1}}$. 
    By convexity and \eqref{eq:ppa-step} we have
    \begin{align*}
        &\scal{\fdfun{j}-\fdfun{j+1}}{ \fdfun{j+1} -f^\dagger } = -  \scal{\eta_j \nabla \risk\loss D{\fdfun{j+1}}}{f^\dagger- \fdfun{j+1}}
        \\ \ge\ & \eta_j(\risk\loss D{\fdfun{j+1}}- \risk\loss D{f^\dagger} )
        \ge 0~.
    \end{align*}
    This in turn yields
    \begin{align*}
        & \quad \norm{\fdfun{j+1}- f^\dagger }^2 + \norm{\fdfun{j} -\fdfun{j+1}}^2 -\norm{\fdfun{j}-f^\dagger}^2
        \\&= 2\norm{\fdfun{j+1}}^2+2\scal{\fdfun{j}-\fdfun{j+1}}{f^\dagger} -2\scal{\fdfun{j}}{\fdfun{j+1}}
        \\ &= 2\scal{\fdfun{j}-\fdfun{j+1}}{f^\dagger-\fdfun{j+1}}
        \\ &\le 0~,
    \end{align*}
    and we infer
    \begin{align*}\norm{\fdfun{j+1}-f^\dagger}^2\le \norm{\fdfun{j}-f^\dagger}^2-\norm{\fdfun{j}-\fdfun{j+1}}^2\le \norm{\fdfun{j} -f^\dagger}^2~,\end{align*}
    that is, Fej\'er monotonicity \eqref{eq:fejer-monotonicity} on $\N_0$.
    Since this estimate  holds for all 
    $j=k-1,\dots,0$ 
    we conclude 
    \begin{align}
        \label{eq:ppa-fejer}
        \norm{\fdfun{k}-f^\dagger} \le \norm{\fdfun{0}-f^\dagger}=\norm{f^\dagger}~.
    \end{align}

    It remains to show that linear interpolation between the integer times preserves \eqref{eq:ppa-self-reg}.
    To this end, let $t\in \R_{>0}\setminus\N$ and let $k\in\N_0$ be such that $t = (1-s)k+ s(k+1)$ for a unique $s\in(0,1)$, that is, 
    $$\fdfun{t} =(1-s)\fdfun{k} + s \fdfun{k+1}~.$$
    By convexity of the risk and since $ \riskk{\fdfun{k+1}}\le \riskk{\fdfun{k}}$ we have $ \risk\loss D{\fdfun{t}} \le \risk\loss D{\fdfun{k}}$.

    \textbf{Case $k=0$.}

    Since $\fdfun{0}=0$ we have $\fdfun{t}=s\fdfun{1}$ and $\nabla\risk\loss D{\fdfun{1}} = -\eta_0^{-1}\fdfun{1}$, and by \cref{prop:m-smoothness-of-risks}  $\risk\loss D{\cdot}$ is $\MMbeta$-smooth, which  yields
    \begin{align}
        \label{eq:risk-growth-first-step}
        \norm{\nabla\risk\loss D{\fdfun{t}} - \nabla \risk\loss D{\fdfun{1}}} 
        \le \MMbeta \norm{\fdfun{t} -\fdfun1}
        = \MMbeta(1-s) \norm{\fdfun{1}}~.
    \end{align}
    Consequently, we can estimate 
    \begin{align}
         \label{eq:ppa-proof-intermediate-gradient}       
         &\scal{-\nabla\risk\loss D{\fdfun{t}}}{\fdfun{t}} \\
        \nonumber =\ & \scal{-\nabla\risk\loss D{\fdfun{1}}+(\nabla\risk\loss D{\fdfun{1}}-\nabla\risk\loss D{\fdfun{t}})}{\fdfun{t}}
        \\ 
        \nonumber =\ &  s\eta_0^{-1} \norm{\fdfun{1}}^2 + \scal{\nabla\risk\loss D{\fdfun{1}}-\nabla\risk\loss D{\fdfun{t}}}{\fdfun{t}}
        \\ 
        \nonumber \ge \ & s\eta_0^{-1} \norm{\fdfun{1}}^2 - \norm{\nabla\risk\loss D{\fdfun{1}}-\nabla\risk\loss D{\fdfun{t}}} \cdot \norm{\fdfun{t}}
        \\
        \nonumber \ge \ &  s\eta_0^{-1}\norm{\fdfun{1}}^2 (1-(1-s)\eta_0\MMbeta)~.
    \end{align}
    Furthermore, \eqref{eq:risk-growth-first-step} and $\nabla\risk\loss D{\fdfun{1}} = -\eta_0^{-1}\fdfun{1}$
    yield
    \begin{align}
        \label{eq:ppa-proof-gradientupperbound}
        & \norm{\nabla\risk\loss D{\fdfun{t}}} \le
        \norm{\nabla\riskk {\fdfun{1}}} + \norm{\nabla\riskk {\fdfun{t}}-\nabla\riskk {\fdfun{1}}}
        \\ \nonumber \le& \,
        \eta_{0}^{-1}\norm{\fdfun{1}}(1+\eta_0(1-s)\MMbeta)~.
    \end{align}
    Let $f^\dagger\in H, \riskk{f^\dagger}\le \riskk{\fdfun{t}}$, which by convexity implies
    \begin{align*}
        0\ge \risk\loss D{f^\dagger} -\risk\loss D{\fdfun{t}} \ge \scal{\nabla\risk\loss D{\fdfun{t}}}{f^\dagger -\fdfun{t}}~.
    \end{align*}
    In other words, we have
    \begin{align}
        \label{eq:another-ppa-argument}
        0\le \scal{- \nabla \risk\loss D{\fdfun{t}}}{\fdfun{t}} \le \scal{- \nabla\risk\loss D{\fdfun{t}}}{f^\dagger} \le  \norm{\nabla \riskk {\fdfun{t}}} \norm{f^\dagger}
    ~.
    \end{align}
    If $\fdfun 1 =0$, then $\fdfun t = t \fdfun 1 =0$, so the norm bound holds. Hence, assume $\fdfun 1 \ne 0$. Since $s>0$ and $\eta_0 \MMbeta <1$, we obtain $\nabla \risk\loss D {\fdfun{t}} \ne 0$ from \eqref{eq:ppa-proof-intermediate-gradient}.
    Combining \eqref{eq:another-ppa-argument} with \eqref{eq:ppa-proof-intermediate-gradient} and \eqref{eq:ppa-proof-gradientupperbound}, we obtain 
    \begin{align*}
        &\quad \norm{f^\dagger} \ge \frac{\scal{-\nabla\risk\loss D{\fdfun{t}}}{\fdfun{t}}}{\norm{\nabla\risk\loss D{\fdfun{t}}}}
        \ge s\norm{\fdfun{1}} \frac{1-(1-s)\eta_0\MMbeta }{1+(1-s)\eta_0\MMbeta}\\
        &= \norm{\fdfun{t}} \frac{1-(1-s)\eta_0\MMbeta }{1+(1-s)\eta_0\MMbeta}~.
    \end{align*}
    By the assumption $\eta_0\MMbeta<1$  we obtain  as desired
    \begin{align*}
        &\norm{\fdfun{t}} \le \frac{1+(1-s)\eta_0 \MMbeta}{1-(1-s)\eta_0\MMbeta}\norm{f^\dagger } \le \frac {1+\eta_0\MMbeta}{1-\eta_0\MMbeta} \norm{f^\dagger} < \frac {2+2(1-\eta_0\MMbeta)}{1-\eta_0\MMbeta} \norm{f^\dagger} = \cmul \norm{f^\dagger}~.
    \end{align*}

    \textbf{Case $k\ge 1$.}
    Define for all $j=0,\dots, k-1$ the function $h_{j}\equalDef \fdfun{j+1}-\fdfun{j}$ and note that by \eqref{eq:ppa-fejer}, with $f^\dagger \equalDef \fdfun{j+1}$, we have $\norm{h_j}\le \norm{\fdfun{j+1}}$. 
    By \eqref{eq:ppa-step} the $\MMbeta$-smoothness of $\risk\loss D{\cdot}$ we have
\begin{align*}
    &\quad \norm{h_k}=\eta_k \norm{\nabla\risk\loss D{\fdfun{k+1}}}\le \eta_k(\norm{\nabla\risk\loss D{\fdfun{k}}}+\MMbeta \norm{h_k})
    =  \eta_k\left(\eta_{k-1}^{-1}\norm{h_{k-1}} + \MMbeta\norm{h_k} \right)
    \\ &\le \eta_k\left(\eta_{k-1}^{-1} \norm{\fdfun{k}} + \MMbeta\norm{h_k}\right)
    \\ &\le \norm{\fdfun{k}} +\eta_k\MMbeta \norm{h_k}~,
\end{align*}
using $\eta_k\le \eta_{k-1}$. 
This  implies 
\begin{align*}
    \norm{h_k}\le \frac{1}{1-\eta_0\MMbeta }\norm{\fdfun{k}}~.
\end{align*}
    Now, let $f^\dagger \in H$ with $ \riskk{f^\dagger}\le \riskk{\fdfun{k}}$, which implies $\norm{\fdfun{k}}\le 2\norm {f^\dagger} $ by Fej\'er monotonicity \eqref{eq:ppa-fejer} on $\N_0$.
    We straightforwardly conclude
    \begin{align*}
        \norm{\fdfun{t}} =\norm{\fdfun{t}+ s h_k} \le \norm{\fdfun{k}}+s\norm{h_k}\le \frac{2-\eta_0\MMbeta}{1-\eta_0\MMbeta} \norm{\fdfun{k}} \le \frac{2(2-\eta_0\MMbeta)}{1-\eta_0\MMbeta} \norm{f^\dagger}~,
    \end{align*}
    as desired.         
\end{proof}

}

\subsection{Proofs for Section \ref{sec:statistical_analysis} -- Preparations}\label{sec:stat-ana-preps}

The   goal of this subsection is to show that for suitable parameters, the 
output of self-regularized learning methods are approximate solutions of regularized
empirical risk functionals. To this end, our first lemma establishes 
the existence of regularized empirical risk minimizers satisfying some additional properties.

\begin{lemma}\label{lem:nice-cr-erm}
Let $\loss:X\times Y\times \R \to [0,\infty)$  be a continuous and  convex loss,
 $F$ be a separable RKBS over $X$  and $\lb>0$. Then for all $\epsilon>0$
 there exists a measurable learning method $D\mapsto g_{D,\lb}$ satisfying the approximate RERM condition \eqref{eq:rerm-opt-prob-eps}  with $p=1$
as well as 
\begin{align*}
\RDxB\loss F< \RD\loss{g_{D,\lb}} \leq \RD \loss 0
\end{align*}
if $\RDxB\loss F < \RD \loss 0$
and $g_{D,\lb} = 0$ else. %
\end{lemma}

\begin{proof}
We first note that there exists a measurable learning method $D\mapsto g_{D,\lb}$  satisfying 
\begin{align}\label{lem:nice-cr-erm-h3}
\lb \snorm {g_{D,\lb}}_F + \RD\loss {g_{D,\lb}}\leq \inf_{f\in F} \bigl( \lb \snorm f_F + \RD\loss f \bigr) + \lb \epsilon / 2 \, . 
\end{align}
This can be seen by repeating the proof of 
\cite[Lemma 7.19]{StCh08} without the clipping operation. In the following, we will modify this learning method in three steps so that the additional requirements are satisfied.

To begin with, we note that $D\mapsto \RDxB\loss F$ is measurable by the continuity of $\loss$ and the separability of $F$. Consequently, the event $Z_0 := \{D:  \RDxB\loss F = \RD \loss 0\}$ is measurable. Moreover, for $D\in Z_0$ we have 
\begin{align*}
\lb \snorm 0_F + \RD\loss 0 = \RDxB \loss F \leq \inf_{f\in F} \bigl( \lb \snorm f_F + \RD\loss f \bigr)
\end{align*}
and therefore redefining $g_{D,\lb}:= 0$ for $D\in Z_0$ gives a measurable learning algorithm 
satisfying \eqref{eq:rerm-opt-prob-eps} and the additional requirement on $Z_0$.  

Similarly, the set $Z_1:= \{D:  \RDxB \loss F <  \RD \loss 0 <\RD\loss{g_{D,\lb}}\}$ is measurable and for $D\in Z_1$ we have 
\begin{align*}
\lb \snorm 0_F + \RD\loss 0 < \lb \snorm {g_{D,\lb}}_F + \RD\loss {g_{D,\lb}} \, .
\end{align*}
If we redefine $g_{D,\lb}:= 0$ for $D\in Z_1$ we thus obtain a measurable learning algorithm 
satisfying \eqref{eq:rerm-opt-prob-eps} and $\RDxB\loss F < \RD \loss {g_{D,\lb}} \leq \RD \loss 0$ for all $D$ with 
$\RDxB\loss F < \RD \loss 0$.

Finally, the set $Z_2:= \{D:  \RDxB \loss F = \RD\loss{g_{D,\lb}}  <   \RD \loss 0 \}$ is also  measurable.
Let us now consider the map $r:[0,1]\to [0,\infty)$ defined by 
\begin{align*}
r(\a) := \RD\loss{\a g_{D,\lb}}\,, \myqquad \a\in [0,1].
\end{align*}
Clearly, $r$ is both continuous and convex. Moreover, we have 
$r(0) = \RD\loss 0$ and $r(1) = \RD\loss{g_{D,\lb}}$, and $r$ has a global minimum at $\a=1$.
Consequently, the set $\{\a: r(\a) = r(1)\}$ is both closed and convex, that is, a closed interval.
We can thus consider
\begin{align*}
\a_D^* := \min\{ \a\in [0,1]: r(\a) = r(1)\} \, .
\end{align*}
In the following, we use the shorthand $g_{D,\lb}^* := \a_D^* g_{D,\lb}$. Now recall that 
for $D\in Z_2$ we  have $\RD \loss 0 >  \RDxB \loss F \geq 0$, and hence 
\begin{align*}
\a:= \a_D := \max\biggl\{\frac 12 ,  1 - \frac {\lb \e}{2 \RD \loss 0}   \biggr\}
\end{align*}
satisfies both $\a\in [1/2,1)$ and $(1-\a) \RD\loss 0 \leq \lb\epsilon /2$.
In particular, we have $\a \a_D^* < \a_D^*$ and therefore also 
\begin{align*}
\RD \loss {\a g^*_{D,\lb}} = \RD \loss {\a \a_D^* g_{D,\lb}} = r(\a\a_D^*) > r(1) = \RD \loss {g_{D,\lb}} = \RDxB\loss F\, .
\end{align*}
Using the convexity of $r$ and $r(1) < r(0)$ we similarly obtain
\begin{align*}
\RD \loss {\a g^*_{D,\lb}} =  r(\a\a_D^*)  \leq  \a\a_D^*r(1) + (1-\a\a_D^*)r(0) \leq r(0)\, .
\end{align*}
By the same convexity argument and the definition of $\a_D^*$ we finally find
\begin{align*}
\lb \snorm{\a g^*_{D,\lb}}_F + \RD\loss {{\a g^*_{D,\lb}}} 
&\leq 
\lb \snorm{\a g^*_{D,\lb}}_F + \a \RD\loss {g^*_{D,\lb}} + (1-\a) \RD\loss 0 \\
&=
\lb \a \a_D^* \snorm{ g_{D,\lb}}_F + \a \RD\loss {g_{D,\lb}} + (1-\a) \RD\loss 0 \\
&\leq 
\lb \snorm {g_{D,\lb}}_F + \RD\loss {g_{D,\lb}} +  (1-\a) \RD\loss 0 \\
&\leq 
\lb \snorm {g_{D,\lb}}_F + \RD\loss {g_{D,\lb}} +  \lb \epsilon /2\, .
\end{align*}
Using \eqref{lem:nice-cr-erm-h3} we conclude that $\a_D \a_D^* g_{D,\lb} = \a g^*_{D,\lb}$ also satisfies 
\eqref{eq:rerm-opt-prob-eps}. Replacing   $ g_{D,\lb}$  by $\a_D \a_D^* g_{D,\lb}$ on $Z_2$ thus gives a  learning algorithm 
satisfying   \eqref{eq:rerm-opt-prob-eps} as well as $\RDxB\loss F< \RD\loss{g_{D,\lb}} \leq \RD \loss 0$ 
for all $D$ with  $\RDxB\loss F< \RD \loss 0$. 

To verify the measurability claims, %
 it suffices to show that $D\mapsto \alpha_D^*$ is   measurable.
To verify this, we consider the measurable map  $h:(X\times Y)^n\times [0,1]\to \R$ given by 
\begin{align*}
h(D,\alpha) := \RD\loss{\a g_{D,\lb}}  -  \RD\loss{g_{D,\lb}}   \, .
\end{align*}
Moreover, for $A:= \{ 0\}$ we define 
 the set-valued map $G:(X\times Y)^n\to 2^{[0,1]}$ given by 
 \begin{align*}
 G(D) := \bigl\{  \a\in [0,1]: h(D,t)\in A  \bigr\}\, .
 \end{align*}
 Clearly, we have $1\in G(D)$ for all $D$.
  Aumann’s Measurable Selection Principle, see 
 e.g.~\cite[Lemma A.3.18]{StCh08}, thus  gives a sequence $(h_n)$ of 
  measurable maps $h_n:(X\times Y)^n \to [0,1]$ such that $\{h_n(D): n\geq 1\} \subset G(D)$ is dense for all $D$.
  By our construction we then find 
  \begin{align*}
  \a_D^* = \inf G(D) = \inf \{h_n(D): n\geq 1\}
  \end{align*}
for all $D$. This shows the measurability of $D\mapsto  \a_D^* $.
\end{proof}

With the help of Lemma \ref{lem:nice-cr-erm} we can now formulate and prove the announced relationship between
self-regularized learning methods and regularized empirical risk minimizers.

\begin{theorem}\label{thm:self-reg-implies-crem}
Let $\loss:X\times Y\times \R \to [0,\infty)$  be a continuous and  convex  loss,
 $F$ be a separable RKBS over $X$ and 
 $T$ be a path-connected Polish space, and 
$(D,t)\mapsto f_{D,t}$ be a self-regularized learning method in $F$
with constants $\const\geq 1$ and $\constt\geq 0$. Moreover, let 
$\lb >0$ and $n\geq 1$.
Then for all $\epsilon>0$ and  $D\in (X\times Y)^n$
there exists a $t_D\in T$ with 
\begin{align}\label{thm:self-reg-implies-crem-eq}
\lb \snorm{f_{D,t_D}}_F + \RD\loss {f_{D,t_D}}
\leq 
\inf_{f\in F} \bigl( \lb \snorm f_F + \RD\loss f \bigr) + \epsilon_D\, , 
\end{align}
where 
\begin{align*}
\epsilon_D \equalDef \lambda\bigl( (\const -1) \norm {g_{D,\lb}}_F +  (\const + 1)  \constt + \epsilon \bigr)  
\end{align*}
and the learning method $D\mapsto {g_{D,\lb}}$ is taken from Lemma \ref{lem:nice-cr-erm}.
Finally, $t_D$ can be chosen such that the resulting maps $D\mapsto f_{D,t_D}$ are measurable.
\end{theorem}

\begin{proof}
Let us fix a data set $D\in (X\times Y)^n$ and an $\e>0$. 
In the case $\RDxB\loss F < \RD \loss 0$,  Lemma \ref{lem:nice-cr-erm} guarantees 
\begin{align*}
\inf_{t\in T} \RD\loss f = \RDxB \loss F <\RD \loss {g_{D,\lb}} \leq  \RD \loss 0 \leq \RD \loss {f_{D,t_0}}  \, . 
\end{align*}
Consequently, 
there exists a $t_1\in T$ with $\RD \loss {f_{D,t_1}} < \RD \loss {g_{D,\lb}} \leq  \RD \loss {f_{D,t_0}}$. 
Since $T$ is path-connected there exists a continuous $\g:[0,1]\to T$
with $\g(0) = t_0$ and $\g(1) = t_1$. We now define the map $r:[0,1]\to [0,\infty)$ by 
\begin{align*}
r(s) := \RD\loss{f_{D,\g(s))}}\, , \myqquad s\in [0,1].
\end{align*}
Since by assumption $t\mapsto \RD\loss{f_{D,t}}$ is continuous, we conclude that $r$ is continuous. 
Moreover, our construction ensures $r(1) < \RD\loss {g_{D,\lb}} \leq r(0)$.
The intermediate value theorem then gives an $s^*\in [0,1)$ with 
\begin{align*}
\RD\loss{f_{D,\g(s^*))}} = \RD \loss {g_{D,\lb}}\, .
\end{align*}
Let us write $t_D := \g(s^*)$. Then the self-regularization inequality 
\eqref{eq:norm_bound_selfreg} ensures 
\begin{align*}
\norm{f_{D,t_D}}_F \le \const \cdot \norm {g_{D,\lb}}_F + \constt \, , 
\end{align*}
and consequently we find 
\begin{align*}
    \lb \norm{f_{D,t_D}}_F + \risk \loss D {f_{D,t_D}}
    &\le \lb\bigl( \const \cdot \norm {g_{D,\lb}}_F + \constt\bigr)  +  \risk \loss D {g_{D,\lb}}   \\
    &\leq \lb \snorm {g_{D,\lb}}_F + \RD\loss {g_{D,\lb}}  + \lb(\const -1) \snorm {g_{D,\lb}}_F + \constt
    \\
    &\leq \inf_{f\in F} \bigl( \lb \snorm f_F + \RD\loss f \bigr) + \epsilon_D \, .
\end{align*}

Let us now consider the remaining case $\RDxB\loss F = \RD \loss 0$. Here, we know by  Lemma \ref{lem:nice-cr-erm} that $g_{D,\lb} = 0$. 
Moreover, there exists a sequence $(t_n)_{n\geq 1}\subset T$ with 
\begin{align}\label{thm:self-reg-implies-crem-h1x}
\RD\loss{f_{D,t_n}} \to \RDxB\loss F = \RD \loss 0\, .
\end{align}
Furthermore, by the definition of self-regularized learning methods there exists 
a $t_0\in T$ with $\RD\loss{f_{D,t_0}} \geq \RD\loss 0$ and $\snorm{f_{D,t_n}}_F\leq \constt$.
Subsequently, we consider two subcases.

In the first subcase, there is an 
 $n\geq 0$ with $\RD\loss{f_{D,t_n}}   = \RD \loss 0$.
 Then the self-regularization inequality ensures 
 \begin{align*}
 \snorm{f_{D,t_n}}_F \leq \const \snorm 0_F + \constt = \constt\, ,
 \end{align*}
 and therefore,  $\RDxB\loss F = \RD \loss 0$ and the definition of $\epsilon_D$ imply
 \begin{align*}
 \lb \snorm{f_{D,t_n}}_F + \RD\loss {f_{D,t_n}}
\leq 
 \lb \constt + \RDxB\loss F
 \leq 
\inf_{f\in F} \bigl( \lb \snorm f_F + \RD\loss f \bigr) + \epsilon_D\, .
 \end{align*}
 Consequently, we choose $t_D := t_n$.

Let us finally consider the remaining subcase $\RD\loss{f_{D,t_n}}   \neq  \RD \loss 0$ for all $n\geq 0$. 
Here, we write $f:= f_{D,t_0}$. By \eqref{thm:self-reg-implies-crem-h1x} 
we may then assume without loss of generality that 
\begin{align}\label{thm:self-reg-implies-crem-h1}
\RD\loss 0 < \RD\loss{f_{D,t_n}} < \RD\loss f
\end{align}
 for all $n\geq 1$. 
For $\a\in [0,1]$ we now consider the function $f_\a:= \a f\in F$. Obviously, we have
$\snorm{f_\a}_F = \a \snorm{f_{D,t_0}}_F \leq \a \constt$
and the convexity of $L$ ensures 
\begin{align*}
\RD\loss {f_\a} \leq \a \RD \loss f + (1-\a) \RD \loss 0 \leq \RD \loss 0 + \a \RD \loss f\, .
\end{align*}
In addition, the map $r:[0,1]\to [0,\infty)$ given by $r(\a) := \RD\loss{f_\a}$ is obviously continuous 
and the first inequality in the above estimate shows that its range equals $[\RD \loss 0, \RD\loss f]$.
By the intermediate value theorem and \eqref{thm:self-reg-implies-crem-h1} 
there thus exist $\a_n\in (0,1)$ with 
\begin{align*}%
\RD\loss{f_{\a_n}} = \RD\loss {f_{D,t_n}}\, , \myqquad n\geq 1.
\end{align*}
By the self-regularization inequality \eqref{eq:norm_bound_selfreg} and the definition of $f$ we then find 
\begin{align*}
\snorm{f_{D,t_n}}
\leq \const \snorm{f_{\a_n}}_F + \constt 
\leq \a_n  \const  \constt + \constt
\leq (\const + 1)  \constt 
\end{align*}
for all $n\geq 1$. Moreover, \eqref{thm:self-reg-implies-crem-h1x} gives an $n\geq 1$ with 
$\RD\loss {f_{D,t_n}} \leq \RD\loss 0 + \lb\epsilon$.
Combining these considerations, we find 
\begin{align*}
\lb \snorm{f_{D,t_n}}_F + \RD\loss {f_{D,t_n}} 
&\leq \lb(\const + 1)  \constt    +  \RD\loss 0 + \lb\epsilon\\
&= \lb  (\const + 1)  \constt   + \RDxB \loss F + \lb\epsilon \\
&\leq \inf_{f\in F} \bigl( \lb \snorm f_F + \RD\loss f \bigr) +  \lb \bigl( (\const + 1)  \constt  + \epsilon \bigr)\, .
\end{align*}
Using the definition of $\epsilon_D$ we then find the assertion for $t_D := t_n$.

To verify the measurability claim, we consider the map 
 $h:(X\times Y)^n \times T\to \R$ defined by 
\begin{align*}
h(D,t) := \lb \snorm{f_{D,t}}_F + \RD\loss {f_{D,t}}  -  \inf_{f\in F} \lb \snorm f_F + \RD\loss f - \epsilon_D\, .
\end{align*}
We first show that $h$ is measurable. Here, the assumed measurability of $(D,t)\mapsto f_{D,t}$ ensures that $(D,t)\mapsto \lb \snorm{f_{D,t}}_F + \RD\loss {f_{D,t}}$
is measurable. In addition, the measurability of $D\mapsto \epsilon_D$ is shown in Lemma \ref{lem:nice-cr-erm} and Lemma \ref{lem:basic-measurability}. To verify the measurability 
of the infimum, we note that for each  $D$, the map $f\mapsto \lb \snorm f_F + \RD\loss f$  is continuous. Consequently, we find 
\begin{align}\label{thm:self-reg-implies-crem-m1}
\inf_{f\in F} \lb \snorm f_F + \RD\loss f = \inf_{f\in F_{\mathrm{cd}}} \lb \snorm f_F + \RD\loss f \, , 
\end{align}
where $F_{\mathrm{cd}}\subset F$ is a countable and dense subset.
In addition, for each $f\in F$ the map $D\mapsto \lb \snorm f_F + \RD\loss f$ is measurable, and therefore the infimum on the right-hand side of 
\eqref{thm:self-reg-implies-crem-m1} is measurable. In summary, we conclude that $h$ is measurable.

 For $A:= (-\infty, 0]$ we now consider the set-valued map $G:(X\times Y)^n\to 2^T$ given by 
 \begin{align*}
 G(D) := \bigl\{  t\in T: h(D,t)\in A  \bigr\}\, .
 \end{align*}
 In the first part of the proof we have already seen that $G(D) \neq \emptyset$ for all $D\in (X\times Y)^n$. Now, Aumann’s Measurable Selection Principle, see 
 e.g.~\cite[Lemma A.3.18]{StCh08} gives a measurable map $D\mapsto t_D$ with $t_D\in G(D)$ for all $D\in (X\times Y)^n$. By construction, we then see that 
 $f_{D,t_D}$ satisfies \eqref{thm:self-reg-implies-crem-eq}. In addition,  combining the measurability of 
 $D\mapsto t_D$  with
 the assumed measurability of $(D,t)\mapsto f_{D,t}$ we see that $D\mapsto f_{D,t_D}$ is measurable.
\end{proof}

\subsection{Proofs for Section \ref{sec:statistical_analysis} -- Theorem \ref{thm:simple-analysis}}\label{sec:stat-ana-rerm}

In this subsection we present the proof of \cref{thm:simple-analysis}, which relies 
on a combination of \cref{thm:self-reg-implies-crem} presented in the previous subsection and 
\cref{concentration-hoeffding-applicat-simple-thm} presented in Supplement \ref{app:A}.

\begin{proof}[Proof of Theorem \ref{thm:simple-analysis}]
For $\epsilon := 1/n$ we consider the measurable learning method $D\mapsto \gDl$
obtained by \cref{lem:nice-cr-erm}. Then this learning method is a PAC-RERM
in the sense of \eqref{eq:pac-rerm} for $\varrho = 0$ and $\epsilon$. Setting $M_{\lb,\epsilon}:=(B+\epsilon)  \lb^{-1}$, \cref{concentration-hoeffding-applicat-simple-thm} thus ensures 
\begin{align*}
\lb \snorm\gDl_{F} + \RP L\gDl-\RPxB LF 
\leq  A_{1}(\lb) + K M_{\lb,\epsilon}^{q  } \cdot  n^{-\frac {1}{2+2\g}}  \sqrt{\t}
\end{align*}
with probability $\P^n$ not less than $1-\eul^{-\t}$, where $K:= 4(a^\g B + 2^{2+q})$.
This in turn implies 
\begin{align*}
\epsilon_D 
&\equalDef \lambda\bigl( (\const -1) \norm {g_{D,\lb}}_F +  (\const + 1)  \constt + \epsilon \bigr)  \\
&\leq 
 (\const -1) A_1(\lb) + (\const -1)   K M_{\lb,\epsilon}^{q  } \cdot   n^{-\frac {1}{2+2\g}}  \sqrt{\t}  +    \lb (\const + 1)  \constt + \lb/n  \\
 &\leq 
 \const A_1(\lb) + \const K (2B \lb^{-1})^{q  }   n^{-\frac {1}{2+2\g}}  \sqrt{\t}  +      (\const + 1)  \constt  + \lb/n  \\
  &\leq 
 \const A_1(\lb) + (2B)^{q } \const   K \lb^{-q  }   n^{-\frac {1}{2+2\g}}  \sqrt{\t}  +     2 \const   \constt    \\
 &\leq  \const B + (2B)^{q } \const   K   +   2 \const   \constt  \\
 &=: \hat \epsilon
\end{align*}
with probability $\P^n$ not less than $1-\eul^{-\t}$.
If we now choose the measurable map $D\mapsto t_D$ according to \cref{thm:self-reg-implies-crem} we then see that 
$D\mapsto f_{D,t_D}$ is an 
PAC-RERM
in the sense of \eqref{eq:pac-rerm} for $\varrho = \eul^{-\t}$ and $\hat \epsilon$. Consequently, 
 \cref{concentration-hoeffding-applicat-simple-thm} ensures 
\begin{align*}
\lb \snorm\gDl_{F} + \RP L\gDl-\RPxB LF 
\leq  A_{1}(\lb) + K M_{\lb,\hat \epsilon}^{q  } \cdot  n^{-\frac {1}{2+2\g}}  \sqrt{\t}
\end{align*}
with probability $\P^n$ not less than $1-2\eul^{-\t}$. It therefore remains 
to bound parts of the second term. Here, we note that $q\geq 1$ implies 
\begin{align*}
 M_{\lb,\hat \epsilon}^{q  }  
= (B+\hat \epsilon)^{q  }    \lb^{-q}
&\leq 
\Bigl( \const\bigl(2B + (2B)^{q }    K  + 2\constt\bigr) \Bigr)^{q  }   \cdot   \lb^{-q} \\
&\leq 
\Bigl( 2\const\bigl( (2B)^{q }    K  + \constt\bigr) \Bigr)^{q  }   \cdot   \lb^{-q} 
\end{align*}
Combining both estimates then yields the assertion.
\end{proof}

\subsection{Proofs for Section \ref{sec:statistical_analysis} -- Theorem \ref{thm:learning_rates_tempered_C_minimal_general_form}}\label{sec:stat-ana-cr-erm}

In this subsection we present the proof of \cref{thm:learning_rates_tempered_C_minimal_general_form}, which relies 
on a combination of \cref{thm:self-reg-implies-crem} presented in   Subsection
\ref{sec:stat-ana-preps} and a simple generalization of \cite[Theorem 7.20]{StCh08}, see
\cref{concentration:rerm2}   in Supplement \ref{app:B}.

We begin with a technical lemma, which will bound the 
constant $B_0$ in \cref{concentration:rerm2}
when proving \cref{concentration:bsf-svm-analysis-refined} below.

\begin{lemma}\label{lem:embedding-improves-norm}
Let $\loss:X\times Y\times \R\to [0,\infty)$ be a loss satisfying the  growth condition \eqref{eq:loss-growth-order-q}
for some  $q\geq 1$  and $\P$ be a distribution on $X\times Y$.
Moreover, let $F$ be an RKBS on $X$ with $F\subset \sLx \infty X$ and $\zeta\in (0,1]$, $\cemb\geq 1$ be constants with   
\begin{align}\label{lem:embedding-improves-norm-h1}
\snorm f_{\Lx \infty \PX} \leq \cemb \snorm f_{\Lx q \PX}^{1-\zeta} \snorm f_F^\zeta\, , \myqquad f\in F.  
\end{align}
In the case $\zeta<1$ we further assume that \firstt{\eqref{eq:loss-exact-growth-order-q-average} holds true for $q$, $\P$, and some $c\in (0,1]$.}
Then, there exists a constant $K>0$ such that for all 
$\lb >0$, $p\geq 1$, $\e\in (0,B]$, and $f_0\in F$ with 
\begin{align}\label{eq:approx-min-f0}
\lb \snorm{f_0}_F^p + \RP \loss {f_0} - \RPxB  \loss F \leq A_p(\lb) + \e
\end{align}
we have  %
\begin{align*}
\snorm { f_0}_{\Lx \infty \PX}\leq K\cdot  \Bigl( \frac{A_p(\lb) + \e}{\lb}  \Bigr)^{\zeta/p} \, .
\end{align*}
\end{lemma}

\begin{proof}
As a preparation for both cases, we note that Assumption \eqref{eq:approx-min-f0}  gives 
\begin{align}\label{lem:embedding-improves-norm-h11}
\lb   \snorm{f_0}_F^p \leq \lb  \snorm{f_0}_F^p + \RP \loss {f_0} -\RPxB  \loss F \leq   A_p(\lb) + \e\, .
\end{align}
Now, if $\zeta=1$  we 
know $K:= \snorm{F\hookrightarrow \sLx \infty X} <\infty$  by \cite[Lemma 23]{ScSt25a}, and consequently
we find $\lb \inorm{f_0}^p \leq  \lb K^p  \snorm{f_0}_F^p $. Combining this with 
\eqref{lem:embedding-improves-norm-h11} gives the assertion.

Let us now consider the case $\zeta<1$. %
For $f_0\in F$ satisfying \eqref{eq:approx-min-f0} we have 
\begin{align*}
 \RP \loss {f_0}  \leq \lb \snorm{f_0}_F^p + \RP \loss {f_0} \leq \inf_{f\in F} \bigl( \lb \snorm{f}_F^p + \RP \loss {f}  \bigr)+ \e \leq \RP L 0 + \e \leq 2B\, ,
\end{align*}
where in the last step we used $\RP L 0 \leq B$ by \eqref{eq:loss-growth-order-q}.
By \eqref{eq:loss-exact-growth-order-q-average} and $c\leq 1\leq B$ we thus find
\begin{align*}
\snorm {f_0}_{\sLx q \PX}^q \leq \firstt{ \frac { \RP \loss {f_0} + c^{-1}}{c} \leq \frac{2B+ c^{-1}}{c} \leq \frac{3B}{c^2} }\, . 
\end{align*}
Combining this with \eqref{lem:embedding-improves-norm-h1} and \eqref{lem:embedding-improves-norm-h11} 
we then obtain
\begin{align*}
\snorm {f_0}_{\Lx \infty \PX} \leq \cemb \snorm {f_0}_{\Lx q \PX}^{1-\zeta} \snorm  {f_0}_F^\zeta \leq \cemb \cdot  \Bigl(  \frac{3B}{\firstt{c^2}} \Bigr)^{(1-\zeta)/q} \cdot \Bigl( \frac{A_p(\lb) + \e}{\lb}  \Bigr)^{\zeta/p} \, . 
\end{align*}
In other words, the assertion holds with $K:= \cemb \cdot  \bigl(  \frac{3B}{\firstt{c^2}} \bigr)^{(1-\zeta)/q} $.
\end{proof}

\begin{theorem}\label{concentration:bsf-svm-analysis-refined} 
Let 
 $L:X\times Y\times \R\to [0,\infty)$ be a loss, $\P$ be a distribution on $X\times Y$, and 
 $F\subset \sLx \infty X$ be a separable RKBS such that
 \assref{loss}
 is satisfied. In addition, assume that 
 $F\subset \sLx q \PX$ is dense
 and that the embedding assumption \eqref{eq:embedding-by-norm}
 is satisfied for some $\zeta\in (0,1]$ and $\cemb\geq 1$, where 
 in the case $\zeta<1$ we further assume that the growth condition \eqref{eq:loss-growth-order-q} is asymptotically exact in the sense of \eqref{eq:loss-exact-growth-order-q-average}. Finally, assume that 
 the average entropy number bound \eqref{concentration:sva-analysis-entropy} 
holds with $a\geq B$. 
 \firstt{Then, for all $p\geq 1$ there exists a constant $K\geq 1$ such that 
 for all fixed $\lb>0$, $n\geq 1$, and
$\t>0$,  
any learning method guaranteeing 
\begin{align}\label{eq:cr-crm-4-bsf}
 \lb \norm \fDl_F^p + \RD L \fDl
 \leq 
  \inf_{f\in F} \left( \lb \norm f_F^p + \RD L f \right) + \epsilon_D
\end{align}
for all $D\in (X\times Y)^n$ and some measurable map 
 $D\mapsto \epsilon_D$}, %
satisfies 
\begin{align*}
\lb\norm\fDl^p_F \!+\! \RP L \fTc\! -\! \RPB L  & \leq
9 A_p(\lb) 
+ 
K  \biggl( \frac{a^{2\g p}}{\lb^{2\g}n^p}  \biggr)^{\frac{1}{2p-2\g-\vt p(1-\g)}} \\
&\quad
+ K \cdot   \Bigl( \frac {A_p(\lb)}\lb  \Bigr)^{\zeta q/p} \cdot 
\frac {\t }{n} \\
&\quad+
\firstt{4} \lb |L|_{M,1}^{-p}
 + 3 \Bigl(\frac{72 V\t}n\Bigr)^{\!\frac 1{2-\vt}}  + \frac {15B\t}{n}
+ 3\epsilon_D
\end{align*}
with probability $\P^n$ not less than $1-3\eul^{-\t}$.
\end{theorem} 

Before we prove this theorem, we note that we always have $2p-2\g-\vt p(1-\g)>0$. Indeed, 
 $\vt\in [0,1]$ yields  $2p-2\g-\vt p(1-\g) \geq p - 2\g + \g p = p + \g (p-2)$.
In the case $p> 2$ we thus have $2p-2\g-\vt p(1-\g)> 0$, and in the case $p\in [1,2)$ the same is true 
as $\g< 1$ implies $p + \g (p-2) > p + (p-2) \geq 0$. Finally, in the case $p=2$ we have 
$p + \g (p-2) = p>0$, which again yields the assertion by our initial estimate.

\begin{proof}
 We first note that in the case 
 $a^{2\g p} > \lb^{2\g} n^p$ we have  
\begin{align*}
\lb \norm\fDl_F^p + \RP L \fTc - \RPB L 
& \leq
 \lb \norm\fDl_F^p + \RD L \fDl + B \\
&\leq  \RD L 0 + \epsilon_D + B \\
&\leq  2B  \biggl( \frac{a^{2\g p}}{\lb^{2\g}n^p}  \biggr)^{\frac{1}{2p-2\g-\vt p(1-\g)}}  + \epsilon_D  \, .
\end{align*}
In other words, 
the assertion is trivially satisfied for $a^{2\g p} > \lb^{2\g} n^p$  whenever we choose $K\geq 2B$. At the end of the proof we will thus 
make sure that $K$ is chosen in this way.%

For the remaining case $a^{2\g p} \leq \lb^{2\g} n^p$, we will use 
 Theorem \ref{concentration:rerm2}  applied to $\ca F := F$ and
$\Upsilon:F\to [0,\infty)$ defined by 
$\Upsilon(f):= \lb \norm f_F^p$. To this end, we define
$r^*$ by  (\ref{concentration:def-rstar}), that is 
\begin{align}\label{eq:def-r-star-new}
r^{*}:= \inf_{f\in F}\bigl(  \lb \norm f_F^p + \RP L {\clippf}\bigr) - \RPB L\, , 
\end{align}
and for
$r>r^*$ we write 
\begin{align}\label{concentration:svm-analysis-basic-def:Fr-new}
\ca F_{r} & :=  \{ f\in F: \lb \norm f_F^p + \RP L{\clippf}-\RPB L\leq r\}\, , \\ \label{concentration:svm-analysis-basic-def:Hr-new}
\ca H_{r} & :=  \{ L\circ \clippf - L\circ \fpb: f\in \ca F_{r}\}\, .
\end{align}

Since  $\ca F_r\subset ( r/ \lb)^{1/p}B_F$, \cite[Lemma 7.17]{StCh08}
together with  (\ref{concentration:sva-analysis-entropy}) and
the injectivity of the entropy numbers
then yields
\begin{align*}
\E_{D\sim\P^n} e_{i}(\ca H_r,\Lx 2 \D ) 
 &\leq
|L|_{M,1} \E_{D_X\sim\P_X^n} e_{i}(\ca F_r,\Lx 2 {\D_X} ) \\ 
&\leq 
2 |L|_{M,1}  \Bigl(\frac r \lb\Bigr)^{1/p} a \, i^{-\frac 1{2\g}}\, .
\end{align*}
Our next goal is to combine this estimate with 
\cite[Theorem 7.16]{StCh08} to find a suitable bound \eqref{concentration:rerm2-rad-est} on 
the Rademacher averages. To this end, we first note that \cite[Theorem 7.16]{StCh08}
requires the constant in the considered entropy number bound to be at not smaller than $B$, that is 
\begin{align}\label{eq:a-vs-B-new}
2 |L|_{M,1}  \Bigl(\frac r \lb\Bigr)^{1/p} a \geq B\, .
\end{align}
However, a simple calculation using $a\geq B$ shows that 
this inequality is satisfied
for all \firstt{$r\geq |L|_{M,1}^{-p}\,  \lb /2$}.
In the following we thus assume that $r$ satisfies this lower bound.

Now, 
for $f\in \ca F_r$, we have $\E_\P (L\circ \clippf - L\circ \fpb)^2\leq V r^\vt$, and consequently 
\cite[Theorem 7.16]{StCh08} applied to $\ca H:= \ca H_r$ shows 
that (\ref{concentration:rerm2-rad-est}) is satisfied for  
 \begin{align*}
\p_n(r) &:= \max\biggl\{ C_1(\g)2^\g|L|_{M,1}^\g a^\g   \Bigl(\frac r \lb\Bigr)^{\frac \g p}(V r^\vt)^{\frac {1-\g}2}n^{-\frac 1 2},  \\
& \qquad\qquad 
C_2(\g)\bigl( 2^\g|L|_{M,1}^\g a^\g \bigr)^{\frac 2{1+\g}}\Bigl(\frac r \lb\Bigr)^{\frac {2\g}{(1+\g)p}}  B^{\frac {1-\g}{1+\g} }n^{ -\frac {1} {1+\g}} \biggr\} \, ,
\end{align*}
where  $C_1(\g)$ and $C_2(\g)$ are the constants appearing in 
\cite[Theorem 7.16]{StCh08}.

Before we can apply Theorem \ref{concentration:rerm2}, we need to verify that $\p_n$ satisfies the doubling
or sub-root condition $\p_n(2r)\leq 2^\a \p_n(r)$ for some $\a\in (0,1)$. To this end, we define 
\begin{align}\label{eq:doubling-alpha-new}
 \a := \frac \g p + \frac {1-\g}2\, .
\end{align}
Note that $\g\in (0,1)$ ensures $\a>0$ and $0<\g<1\leq p$ implies 
\begin{align*}
\a = \frac \g p + \frac {1-\g}2 \leq \g +  \frac {1-\g}2 = \frac{1+\g}{2} < 1 \,.
\end{align*}
Now, $\vt \in [0,1]$ implies $\frac \g p + \frac {\vt(1-\g)}{2} \leq \a$, and hence we find
\begin{align*}
&
 C_1(\g)2^\g|L|_{M,1}^\g a^\g   \Bigl(\frac {2r} \lb\Bigr)^{\frac \g p}(V (2r)^\vt)^{\frac {1-\g}2}n^{-\frac 1 2} \\
 &= 2^{\frac{\g}{p}+ \frac {\vt(1-\g)}{2}}
 C_1(\g)2^\g|L|_{M,1}^\g a^\g   \Bigl(\frac r \lb\Bigr)^{\frac \g p}(V r^\vt)^{\frac {1-\g}2}n^{-\frac 1 2} \\
 &\leq 2^\a \p_n(r)
\end{align*}
for all $r>0$. Moreover, $p\geq 1\geq \frac {2\g}{1+\g}$ implies 
$\frac {2\g}{(1+\g)p} - \frac \g p = \frac{\g(1-\g)}{(1+\g)p}\leq \frac{\g(1-\g)}{1+\g} \cdot \frac {1+\g}{2\g} = \frac {1-\g}2$
and thus also $\a \geq \frac {2\g}{(1+\g)p} $. The latter, in turn, gives
\begin{align*}
&
 C_2(\g)\bigl( 2^\g|L|_{M,1}^\g a^\g \bigr)^{\frac 2{1+\g}}\Bigl(\frac {2r} \lb\Bigr)^{\frac {2\g}{(1+\g)p}}  B^{\frac {1-\g}{1+\g} }n^{ -\frac {1} {1+\g}} \\
&=
2^{\frac {2\g}{(1+\g)p} }
 C_2(\g)\bigl( 2^\g|L|_{M,1}^\g a^\g \bigr)^{\frac 2{1+\g}}\Bigl(\frac r \lb\Bigr)^{\frac {2\g}{(1+\g)p}}  B^{\frac {1-\g}{1+\g} }n^{ -\frac {1} {1+\g}}\\
 &\leq 2^\a \p_n(r)
\end{align*}
for all $r>0$. Together, these estimates show the required $\p_n(2r)\leq 2^\a \p_n(r)$ for all $r>0$.

The next step for the application of Theorem \ref{concentration:rerm2} is to solve the implicit inequality
\begin{align}\label{concentration:svm-analysis-hxxx-new} 
 r > 8 \cdot \frac{2+2^\a}{2-2^\a} \p_n(r)
\end{align}
given in \eqref{conc-adv:crerm-r-cond}. 
\secondd{\ms{so much headache! delayed this for now. Looks good however. }
To this end, we write 
\begin{align*}
 c_{\g,p,M,V,1} &:=  \frac{64}{2-2^\a} \cdot C_1(\g)|L|_{M,1}^\g   V^{\frac {1-\g}2}\, ,\\
 c_{\g,p,M,B,2} &:= \frac{64}{2-2^\a} \cdot C_2(\g)  |L|_{M,1}^{\frac {2\g}{1+\g}}  B^{\frac {1-\g}{1+\g} } \, .
\end{align*}
Now, $1-\frac \g p - \vt \frac{1-\g}2 = \frac{2p-2\g-\vt p (1-\g)}{2p}$ shows that for
\begin{align}\label{concentration:svm-analysis-hxy-new} 
r \geq c_{\g,p,M,V,1}^{\frac{2p}{2p-2\g-\vt p (1-\g)}} \biggl( \frac{a^{2\g p}}{\lb^{2\g}n^p}  \biggr)^{\frac{1}{2p-2\g-\vt p(1-\g)}}
\end{align}
we have $r > 8 \cdot \frac{2+2^\a}{2-2^\a} \cdot  C_1(\g)2^\g|L|_{M,1}^\g a^\g   \bigl(\frac r \lb\bigr)^{\frac \g p}(V r^\vt)^{\frac {1-\g}2}n^{-\frac 1 2}$. Similarly, $1-\frac{2\g}{(1+\g)p} = \frac{p+\g p-2\g}{(1+\g)p}$
shows that for
\begin{align}\label{concentration:svm-analysis-hxz-new} 
 r \geq c_{\g,p,M,B,2}^{\frac {(1+\g)p}{p+\g p-2\g}} \biggl( \frac{a^{2\g p}}{\lb^{2\g}n^p}  \biggr)^{\frac {1}{p+\g p-2\g}}
\end{align}
we have $r > 8 \cdot \frac{2+2^\a}{2-2^\a} \cdot C_2(\g)\bigl( 2^\g|L|_{M,1}^\g a^\g \bigr)^{\frac 2{1+\g}}\bigl(\frac r \lb\bigr)^{\frac {2\g}{(1+\g)p}}  B^{\frac {1-\g}{1+\g} }n^{ -\frac {1} {1+\g}}$.
}
To unify Conditions \eqref{concentration:svm-analysis-hxy-new}  and 
\eqref{concentration:svm-analysis-hxz-new} we note that 
we have $p + \g p - 2\g = 2p -2\g -p(1-\g) \leq 2p-2\g -\vt p (1-\g)$ and since we are in the case 
$a^{2\g p} \leq \lb^{2\g} n^p$ we conclude that 
\begin{align*}
 \biggl( \frac{a^{2\g p}}{\lb^{2\g}n^p}  \biggr)^{\frac {1}{p+\g p-2\g}}
\leq
\biggl( \frac{a^{2\g p}}{\lb^{2\g}n^p}  \biggr)^{\frac{1}{2p-2\g-\vt p(1-\g)}} \, .
\end{align*}
Consequently, if we set $K_{\g,p,\vt,M,V,B}:= \max \bigl\{  c_{\g,p,M,V,1}^{\frac{2p}{2p-2\g-\vt p (1-\g)}}, c_{\g,p,M,B,2}^{\frac {(1+\g)p}{p+\g p-2\g}}  \bigr \}$, then for all 
\begin{align}\label{eq:extra-assump-on-r-new}
 r \geq K_{\g,p,\vt,M,V,B}
 \cdot \biggl( \frac{a^{2\g p}}{\lb^{2\g}n^p}  \biggr)^{\frac{1}{2p-2\g-\vt p(1-\g)}}
\end{align}
we have \eqref{concentration:svm-analysis-hxxx-new}. 
Let us now fix an 
arbitrary $\e>0$ and an $\fO\in F$ with 
\begin{align*}
 \lb \norm \fO_F^p + \RP L \fO \leq \inf_{f\in F} \bigl( \lb \norm f_F^p + \RP L f  \bigr)  + \e\,.
\end{align*}
Since $F\subset \sLx q \PX$ is dense, we have seen around \eqref{eq:risk-approx} that 
$\RPxB L F  = \RPB L$ holds true, and hence our choice of $f_0$ ensures \eqref{eq:approx-min-f0}.
By Lemma \ref{lem:embedding-improves-norm}  \ms{\cref{lem:embedding-improves-norm} yields the $P_X$-infty norm, not the pointwise!!}
we then obtain
a constant $C$ independent of $\lb, p, \e, \t,n$, and $f_0$ such that 
\begin{align*}
\inorm \fO  
\leq  C \cdot \Bigl( \frac {A_p(\lb) + \e}\lb  \Bigr)^{\zeta /p}\, .
\end{align*}
Using the growth condition \eqref{eq:loss-growth-order-q}, 
we thus find 
\begin{align*}
\inorm{L\circ \fO} 
\leq B + B\inorm \fO ^q
\leq B + B \cdot C^q \Bigl( \frac {A_p(\lb) + \e}\lb  \Bigr)^{\zeta q/p}
=: B_0\, .
\end{align*}
Let us now specify $\e$ by 
\begin{align*}
\e :=  \min\biggl\{ 2^{-1}\cdot    C^{-p/\zeta} \cdot \biggl( \frac {n}{15  c \t }  \biggr)^{p/(\zeta q)} \cdot |L|_{M,1}^{-p^2/(\zeta q)} \cdot  \lb ^{1+p/(\zeta q)},\, \, \frac \lb {9 |L|_{M,1}^{p}} \biggr\}\, .
\end{align*}
By considering only the first term in the minimum, this definition implies $(\frac {\e}\lb)^{\zeta q/p} \leq   \frac {2^{-\zeta q/p}n}{15   c \t C^{q}}  \lb |L|_{M,1}^{-p}  $, %
and therefore we obtain 
\begin{align} \nonumber
B_0
&\leq B + 2^{\zeta q/p} B \cdot C^q \biggl(\Bigl( \frac {A_p(\lb)}\lb  \Bigr)^{\zeta q/p} +  \Bigl( \frac {\e}\lb  \Bigr)^{\zeta q/p}\biggr) \\ \label{concentration:svm-analysis-final-h1-new} 
&\leq B + 2^{\zeta q/p} B \cdot C^q  \Bigl( \frac {A_p(\lb)}\lb  \Bigr)^{\zeta q/p} +    \frac { n}{15   \t} \lb |L|_{M,1}^{-p}  \, ,
\end{align}
where in the first step we used $(x+y)^r \leq 2^r (x^r+y^r)$ for all $x,y\geq 0$ and $r>0$.

With these preparations we now define 
\begin{align*}
 r &:= 
 \bigl(\lb \norm\fO_F^p +\RP L \fO  - \RPB L\bigr) 
 + 
 \firstt{\lb |L|_{M,1}^{-p}/2}
 +
 K_{\g,p,\vt,M,V,B} \cdot \biggl( \frac{a^{2\g p}}{\lb^{2\g}n^p}  \biggr)^{\frac{1}{2p-2\g-\vt p(1-\g)}} \\
 &\qquad + 
 \Bigl(\frac{72 V\t}n\Bigr)^{\!\frac 1{2-\vt}}  
 +
 \frac {5B_0\t }{n} \, .
\end{align*}
By the definition of $r^*$, see \eqref{eq:def-r-star-new}, 
we then have $r>r^*$. In addition, we have $r\geq  \lb |L|_{M,1}^{-p}/2$
and hence we have \eqref{eq:a-vs-B-new}. As mentioned above, the function 
$\p_n(r)$ thus satisfies \eqref{concentration:rerm2-rad-est}
as well as the doubling condition for the value of  $\a$ specified in 
\eqref{eq:doubling-alpha-new}.
In addition, \eqref{eq:extra-assump-on-r-new} is satisfied,
and therefore also the implicit inequality \eqref{concentration:svm-analysis-hxxx-new}. Finally, the remaining two conditions in 
\eqref{conc-adv:crerm-r-cond}, namely 
\begin{align*}
r > \biggl(\frac{72 V\t}n\biggr)^{\frac 1{2-\vt}} \myqquad \mbox{ and } \myqquad r >  \frac {5B_0\t }{n}\, , 
\end{align*}
are obviously also satisfied. 
 Theorem \ref{concentration:rerm2} then shows 
\begin{align*}
 \lb \norm {f_D}_F^p  + \RP L \fDc - \RPB L  \leq   6 \bigl(  \lb \norm {\fO}_F^p + \RP L {\fO} -\RPB L\bigr) 
+ 3 r  + 3\epsilon_D 
\end{align*}
with probability $\P^n$ not less than $1-3\eul^{-\t}$. It therefore
remains to verify that this inequality implies 
 the assertion.
 To this end, we first observe that our definitions yield 
 \begin{align*}
 6 \bigl(  \lb \norm {\fO}_F^p + \RP L {\fO} -\RPB L\bigr)  + 3r 
 &= 
 9 \bigl(  \lb \norm {\fO}_F^p + \RP L {\fO} -\RPB L\bigr) 
 + 
 \firstt{2}\lb|L|_{M,1}^{-p}\\
 &\qquad +
 3K_{\g,p,\vt,M,V,B} \cdot \biggl( \frac{a^{2\g p}}{\lb^{2\g}n^p}  \biggr)^{\frac{1}{2p-2\g-\vt p(1-\g)}} \\
 &\qquad + 
 3\Bigl(\frac{72 V\t}n\Bigr)^{\!\frac 1{2-\vt}}  
 +
 \frac {15B_0\t }{n} \, .
 \end{align*}
 \firstt{Since $F\subset \sLx q \PX$ is dense, we have seen around \eqref{eq:risk-approx} that 
$\RPxB L F  = \RPB L$ holds true. }
By the choice of $\fO$ 
 and the definition of $\e$ we thus have 
 \begin{align*}
  9 \bigl(  \lb \norm {\fO}_F^p + \RP L {\fO} -\RPB L\bigr) 
  \leq 
  9 A_p(\lb ) + 9 \e
  \leq 
  9 A_p(\lb) + \lb {|L|_{M,1}^{-p}} \, .
 \end{align*}
 In addition, \eqref{concentration:svm-analysis-final-h1-new}  implies 
 \begin{align*}
  \frac {15B_0\t }{n} 
  \leq 
   \frac {15B\t }{n} + 15 \cdot  2^{\zeta q/p} B \cdot C^q  \Bigl( \frac {A_p(\lb)}\lb  \Bigr)^{\zeta q/p} \cdot \frac \t n+  \lb|L|_{M,1}^{-p} \, .
 \end{align*}
Combining these considerations we then quickly obtain the assertion 
 for the constant $K:= 2B + 3K_{\g,p,\vt,M,V,B} +  15 \cdot  2^{\zeta q/p} B \cdot C^q\}$.
\end{proof}

The best rates in  \cref{concentration:bsf-svm-analysis-refined} are obtained by optimizing with respect to $\lb$.
The following lemma provides the result of this optimization up to constants.

\begin{lemma}\label{lem:hyper-param-opt-new}
Let 
 $F\subset \sLx \infty X$ be an RKBS such that 
   \assref{approx-error} is satisfied. 
 In addition, let $\vt\in [0,1]$, $p,q\geq 1$, $\zeta\in (0,1]$, and $\g\in (0,1)$. 
We define 
\begin{align}\label{lem:hyper-param-opt-new-rho}
\rho := \min\biggl\{\frac{p+\b(2-p)}{2\b + \zeta q(1-\b)}, \frac{p+\b(2-p)}{2\g   +   4\b   - 2\g \b     - 2\vt  \b + 2\vt \g \b    }  \biggr\}
\end{align}
and $\lb_n := n^{- \rho}$ for all $n\geq 1$. Then for all $\t\geq 1$ and $n\geq 1$ we have 
        \begin{align*}
        &\frac {\t}{{n}}
        + n^{-\!\frac 1{2-\vt}}
        + \inf_{\lambda\in (0,1]} 
        \bigg(
        \lambda 
        + A_p(\lambda) 
        +\frac {\t }{{n}} \Bigl(  \frac {A_p(\lambda)}{\lambda}  \Bigr)^{\zeta q/p}
        + (\lambda^{2\gamma}n^p)^{-\frac{1}{2p-2\g-\vt p(1-\g)}}\bigg)
        \\ &\le 
        \frac {\t}{{n}}
        + n^{-\!\frac 1{2-\vt}}
        + 
        \lambda_n 
        + A_p(\lambda_n) 
        +\frac {\t }{{n}} \Bigl(  \frac {A_p(\lambda_n)}{\lambda_n}  \Bigr)^{\zeta q/p}
        + (\lambda_n^{2\gamma}n^p)^{-\frac{1}{2p-2\g-\vt p(1-\g)}}
        \\ &\le         
        K n^{-\alpha  } \cdot  \tau  \, , 
    \end{align*}
    where $K$ is a constant independent of $\t$ and $n$ and  
     \begin{align*}
 \a :=  \min\biggl\{\frac{2\b}{\b(2-\zeta q) + \zeta q} ,  \frac{\b}{\g +   \b(2 - \g -\vt + \vt \g)   }  \biggr\} \, .
 \end{align*}
\end{lemma}

\begin{proof}
It suffices to prove the second inequality. Here 
 Lemma \ref{lem:reg-trafo} 
with $r=2$
shows
\begin{align*}
A_p(\lb) \leq c_{p,\b}\lb^{\frac{2\b}{p+\b(2-p)}} \, , \myqquad \lb\geq 0,
\end{align*}
where $c_{p,\b}$ is a constant only depending on $p$ and $\b$.
We   define $\b^\dagger := \frac{2\b}{p+\b(2-p)}$
and note that $\b\in (0,1]$ implies $\b^\dagger\in (0,1]$.
Moreover, we have 
\begin{align*}
\frac{p}{\zeta q+(p-\zeta q)\b^\dagger}
= 
\frac{p(p+\b(2-p))}{\zeta q p+\zeta q\b(2-p) + 2(p-\zeta q) \b   }
= 
\frac{p(p+\b(2-p))}{\zeta qp-\zeta q\b p + 2p \b   }
= 
\frac{p+\b(2-p)}{2\b + \zeta q(1-\b)}
\end{align*}
as well as 
\begin{align*}
\frac{p}{2\g + (2p-2\g-\vt p + \vt p \g)\b^\dagger} 
&= 
\frac{p(p+\b(2-p))}{2\g (p+\b(2-p)) + 2(2p-2\g-\vt p + \vt p \g)\b} \\
&=
\frac{p+\b(2-p)}{2\g   +   4\b  - 2\g \b      - 2\vt  \b + 2\vt \g \b    }\, .
\end{align*}
Consequently, we find 
\begin{align*}
 \rho = \min\biggl\{ \frac{p}{\zeta q+(p-\zeta q)\b^\dagger}, \frac{p}{2\g + (2p-2\g-\vt p + \vt p \g)\b^\dagger}     \biggr\}~.
\end{align*}

In view of the infimum, 
we now consider the function $h:(0,1]\to \R$ defined by
\begin{align*}
h(\lb) := \lb^{\b^\dagger} + n^{- \frac{p}{2p-2\g-\vt p(1-\g)}} \lb^{-\frac{2\g}{2p-2\g-\vt p(1-\g)}} + n^{-1} \lb^{\frac{(\b^\dagger-1)\zeta q}p}\, .
\end{align*}
Now, it is easy to verify that  \cite[Lemma A.1.7]{StCh08}  does not only hold for $r\geq 1$ but actually for all $r>0$.
Consequently, by this lemma and its proof 
 we then know that for all $n\geq 1$ we have 
\begin{align*}
  c n^{-\rho \b^\dagger} \leq \min_{\lb\in (0,1]} h(\lb) \leq h(\lb_n) \leq 
  3 n^{-\rho \b^\dagger} \, ,
\end{align*}
where we used the variables $s:= \lb$,
$t:= n^{-1}$, $\b':= \beta^\dagger$,  $\g':= \frac{p}{2p-2\g-\vt p(1-\g)}$, $\a':= 2\g/p$, $r':= \zeta q/p$, 
and $\rho' := \rho \b^\dagger$,
and $c$ is a constant 
independent of $n$.  
Using the definition of $\rho$, $\b^\dagger$,  $\a$, and $\t\geq 1$ we thus find 
\begin{align*}
 A_p(\lambda_n) 
        +\frac {\t }{{n}} \Bigl(  \frac {A_p(\lambda_n)}{\lambda_n}  \Bigr)^{\zeta q/p}
        + (\lambda_n^{2\gamma}n^p)^{-\frac{1}{2p-2\g-\vt p(1-\g)}}
    \leq
    K n^{-\rho \b^\dagger} \t = K n^{-\a} \t
\end{align*}
for all $n\geq 1$ and $\t\geq 1$ and a constant $K$ independent of $n$ and $\t$. Since the terms 
$ \frac {\t}{{n}}
        + n^{-\!\frac 1{2-\vt}}
        + 
        \lambda_n$ are dominated by  $n^{-\a} \t$ we then find the assertion for a suitably modified constant $K$.
\end{proof}

\begin{corollary}\label{cor:rates-unified-new-refined}
Let 
 $L:X\times Y\times \R\to [0,\infty)$ be a loss, $\P$ be a distribution on $X\times Y$, and 
 $F\subset \sLx \infty X$ be an RKBS such that 
Assumptions \assrefshort{approx-error} and \assrefshort{loss}
are satisfied. 
In addition, let
the embedding assumption \eqref{eq:embedding-by-norm}
 be satisfied for some $\zeta\in (0,1]$ and $\cemb\geq 1$, where 
 in the case $\zeta<1$ we further assume that the growth condition \eqref{eq:loss-growth-order-q} is asymptotically exact in the sense of \eqref{eq:loss-exact-growth-order-q-average}. Finally, assume that 
 the   entropy number bound \eqref{concentration:sva-analysis-entropy}
holds with $a\geq B$. 
We define $\rho$ by \eqref{lem:hyper-param-opt-new-rho} 
and $\lb_n := n^{- \rho}$ for all $n\geq 1$. Then for all $p\geq 1$ there exists
a constant $K\geq 1$  such that for all $n\geq 1$
and $\t\geq 1$ any learning method guaranteeing \eqref{eq:cr-crm-4-bsf}
with $\lb = \lb_n$
satisfies
\begin{align*}
\lb_n\norm{f_{D,\lb_n}}^p_F +\ \RP L {\clippfo_{D,\lb_n}} - \RPB L 
\leq 
K n^{- \a} \t + 3\epsilon_D
\end{align*}
 with probability $\P^n$ not less than $1-3\eul^{-\t}$, where
 \begin{align*}
 \a :=  \min\biggl\{\frac{2\b}{\b(2-\zeta q) + \zeta q} ,  \frac{\b}{\g +   \b(2 - \g -\vt + \vt \g)   }  \biggr\} \, .
 \end{align*}
\end{corollary}

\begin{proof}
The assertion directly follows from combining \cref{concentration:bsf-svm-analysis-refined} with \cref{lem:hyper-param-opt-new}.
\end{proof}

\begin{proof}[Proof of \cref{thm:learning_rates_tempered_C_minimal_general_form}]
For a fixed $n\geq 1$ we define $\varepsilon_n := 1/n$ and 
 $\lambda_n\equalDef n^{-\rho}$, where 
 \begin{align*}
\rho := \min\biggl\{\frac{1+\b}{2\b + \zeta q(1-\b)}, \frac{1+\b}{2\g   +   4\b   - 2\g \b     - 2\vt  \b + 2\vt \g \b    }  \biggr\}
\end{align*}
 is defined as in 
 \eqref{lem:hyper-param-opt-new-rho} for $p=1$. 
For   $\lb = \lb_n$ and $\varepsilon =\varepsilon_n$ 
we now choose 
a measurable map $D\mapsto t_D$ according to 
 \cref{thm:self-reg-implies-crem}.
By \eqref{thm:self-reg-implies-crem-eq} we then see that 
$D\mapsto f_{D,t_D}$ is an $\epsilon_D$-approximate RERM with respect to the regularization parameter $\lb_n$ and $p=1$, where 
\begin{align*}
\epsilon_D \equalDef \lambda_n\bigl( (\const -1) \norm {g_{D,\lb_n}}_F +  (\const + 1)  \constt + \varepsilon_n \bigr)  \,.
\end{align*}
Here, $D\mapsto g_{D,\lb_n}$ is the 
  $\lb_n\varepsilon_n$-approximate regularized 
empirical risk minimizer found in Lemma \ref{lem:nice-cr-erm} with $p=1$.
Now, by \Cref{cor:rates-unified-new-refined} we know that
\begin{align*}
\lb_n\norm{g_{D,\lb_n}}_F +\ \RP L {\clippgo_{D,\lb_n}} - \RPB L 
\leq 
K n^{- \a} \t + 3\lb_n\varepsilon_n
\end{align*}
holds with probability $\P^n$ not less than $1-3\eul^{-\t}$.
For such $D$ the  bound  $\a\leq 1 + \rho$ together with  the assumption $\t\geq 1$ implies 
$\lb_n\norm{g_{D,\lb_n}}_F \leq (K+3) n^{- \a} \t$,
 which in turn 
gives 
\begin{align*}
\epsilon_D  
\leq  
\const \cdot \lambda_n \norm {g_{D,\lb_n}}_F +  2\lambda_n\const   \constt + \lambda_n\varepsilon_n   
\leq 
\const (K+4) n^{- \a} \t + 2\lambda_n\const   \constt
\end{align*}
with probability $\P^n$ not less than $1-3\eul^{-\t}$.
Moreover, applying \Cref{cor:rates-unified-new-refined} to 
the $\epsilon_D$-approximate regularized 
empirical risk minimizer 
$D\mapsto f_{D,t_D}$ shows that
\begin{align*}
\lb_n\norm{ f_{D,t_D}}_F +\ \RP L {\clippfo_{D,t_D}} - \RPB L 
\leq 
K n^{- \a} \t + 3\epsilon_D
\end{align*}
holds with probability $\P^n$ not less than $1-3\eul^{-\t}$. Inserting the previous estimate on $\epsilon_D$ then gives the assertion with the algorithm independent  constant 
$16K$.

It remains to prove the improved rates in the cases \eqref{eq:s-range-exact}. Here we first note that we have 
\begin{align*}
\frac{2\beta}{\beta(2-\zeta q)+ \zeta q} 
\geq  \frac{\beta}{\beta(2-\gamma-\vartheta+\gamma\vartheta)+ \gamma}
\end{align*}
if and only if the first inequality in \eqref{eq:s-range-exact} holds.
The form of $\a$ in \eqref{eq:general-alpha-in-strong-rates} then yields the assertion. Finally, note that  $1 - \g -\vt + \vt \g\geq 0$ is true  for all $\g\in (0,1)$ and $\vt\in [0,1]$. From this it is straightforward to see that the second inequality in \eqref{eq:s-range-exact} implies the first one.
\end{proof}

\begin{funding}
    Max Sch\"olpple and Ingo Steinwart were funded by the Deutsche Forschungsgemeinschaft (DFG) in the project STE 1074/5-1, within the DFG priority programme SPP 2298 ``Theoretical Foundations of Deep Learning'', and  Max Sch\"olpple also by the International Max Planck Research School for Intelligent Systems (IMPRS-IS).
\end{funding}

\bibliographystyle{imsart-number}
\bibliography{references-used}

\begin{supplement}
\stitle{An Oracle Inequality for RERMs under Minimal Assumptions}
\sdescription{This supplement contains an oracle inequality for a generalized version of regularized empirical risk minimizers. The focus of this oracle inequality lies on using minimal assumptions on $L$, $\P$, and $F$, as well as on including a norm bound of the decision functions. This oracle inequality 
is the fundament of Theorem \ref{thm:simple-analysis}. Since its proof follows highly standard uniform deviation arguments, we moved it to a supplement.}
\end{supplement}

\begin{supplement}
\stitle{A Refined Oracle Inequality for Clipped RERMs}
\sdescription{This supplement contains an oracle inequality for a generalized version of clipped regularized empirical risk minimizers. Its focus lies on 
providing state-of-the-art bounds, and it serves as a basis for Theorem \ref{thm:learning_rates_tempered_C_minimal_general_form}. Since it is a generalization of \cite[Theorem 7.20]{StCh08} we also moved it to a supplement.}
\end{supplement}

\begin{supplement}
\stitle{Miscellaneous Material}
\sdescription{This supplement presents some straightforward generalizations of already known elementary material.}
\end{supplement}

\begin{supplement}
\stitle{Classical Function Spaces: Entropy numbers and Embeddings}
\sdescription{This supplement recalls the Besov and Triebel-Lizorkin scales of function spaces and how classical spaces such as (fractional) Sobolev spaces fall into these scales. In addition, it is shown, how these spaces 
satisfy the entropy number and the embedding assumptions used in our statistical analysis.}
\end{supplement}

\begin{supplement}
\stitle{Details on some Comparisons}
\sdescription{This supplement provides some details on the comparisons 
to learning rates obtained earlier in the literature.}
\end{supplement}

\newpage

\begin{appendix}

\setcounter{section}{0}
\renewcommand{\thesection}{\Alph{section}}
\numberwithin{equation}{section}
\numberwithin{theorem}{section}

\makeatletter
\@for\envname:=satz,lemma,proposition,Theorem,corollary,%
defenv,algorithmThm,remenv,exenv,outlook,notation,definition,%
optimization,notenv,remark,example\do{%
  \expandafter\def\csname the\envname\endcsname{\thetheorem}%
  \expandafter\def\csname theH\envname\endcsname{\theHtheorem}%
}
\makeatother

\section{An Oracle Inequality for RERMs under Minimal Assumptions}
\label{app:A}

In this supplement we investigate learning algorithms that 
approximately minimize a regularized empirical risk with high probability. 
To describe these algorithms, we fix a convex loss
$L:X\times Y\times \R\to [0,\infty)$, an RKBS $F$ over $X$, and some 
parameters $\lb>0$ and $p\geq 1$. Then a measurable learning algorithm 
$D\mapsto \gDl$ is called a \emph{probably approximately correct regularized
empirical risk minimizer (PAC-RERM)}, if there are $\epsilon>0$ and 
$\varrho \in [0,1)$ such that we have 
\begin{align}\label{eq:pac-rerm}
\lb \snorm {\gDl}_F^p + \RD L {\gDl}
\leq \inf_{f\in F}\bigl( \lb \snorm f_F^p + \RD L f \bigr)+ \epsilon
\end{align}
with probability $\P^n$ not less than $1-\varrho$.
The following theorem provides an oracle inequality for PAC-RERMs.

\begin{theorem}\label{concentration-hoeffding-applicat-simple-thm}
Let  
$L:X\times Y\times \R\to [0,\infty)$ be a convex, locally Lipschitz continuous loss satisfying 
the growth condition \eqref{eq:loss-growth-order-q}
with constants $B, q\geq 1$.
Moreover, let $F\subset \sLx \infty X$ be a separable RKBS satisfying both 
$\snorm{F\hookrightarrow \sL \infty X} \leq 1$
and the entropy number assumption
\eqref{eq:sup-entropy-numbers}
for some constants $a\geq 1$ and $\g> 0$. 
We define 
\begin{align*}
M_\lb:=(B+\epsilon)^{1/p}  \lb^{-1/p} \, , \myqquad \epsilon,\lb>0
\end{align*}
and $K:= 4(a^\g B + 2^{2+q}) $. Then, for all 
 probability measures $\P$ on $X\times Y$, all  $\lb\in (0,1]$, all measurable
learning algorithms   $D\mapsto \gDl$ satisfying \eqref{eq:pac-rerm}
and all $n\geq 1$ and $\t\geq 1$ we have 
\begin{align*}
\lb \snorm\gDl_{F}^{p} + \RP L\gDl-\RPxB LF 
\leq  A_{p}(\lb) + K M_\lb^{q} \cdot  n^{-\frac {1}{2+2\g}}  \sqrt{\t}
\end{align*}
with probability $\P^n$ not less than $1-\eul^{-\t}- \varrho$.
\end{theorem}

\begin{proof}
Let us  fix a $\gPl\in F$ such that 
$\lb \snorm \gPl_F^p + \RP L \gPl \leq A_p(\lb) + \RPxB LF + \epsilon$.
Moreover, we write 
\begin{align*}
\denseblo := \bigl \{ D\in (X\times Y)^n : \gDl \mbox{ satisfies } \eqref{eq:pac-rerm}\bigr\} \, .
\end{align*}
 By assumption, we then know $\P^n(\denseblo)\geq 1-\varrho$ and 
 for $D\in \denseblo$ we find 
\begin{align*} 
& 
\lb \snorm\gDl_{F}^{p} + \RP L\gDl-\RPxB LF - A_{p}(\lb) \\
&\leq 
\lb \snorm\gDl_{F}^{p} + \RP L\gDl- \lb \snorm \gPl_F^p - \RP L \gPl + \epsilon\\ 
& = 
 \RP L\gDl- \RD L\gDl\\  
&\qquad + \lb \snorm\gDl_F^{p} + \RD L\gDl- \lb \snorm \gPl_F^p - \RP L \gPl + \epsilon\\
& \leq 
\RP L\gDl - \RD L \gDl + \RD L\gPl - \RP L\gPl   + 2\epsilon\, , 
\end{align*}
where in the last step we used \eqref{eq:pac-rerm}.
Moreover, for $D\in \denseblo$ we have 
\begin{align*}
\lb \snorm\gDl_{F}^{p}   \leq \lb \snorm\gDl_{F}^{p} + \RP L\gDl \leq \RD L 0 +\epsilon \leq B+\epsilon
\end{align*}
and an exactly analogous estimate shows $\lb \snorm\gPl_{F}^{p}\leq  B+\epsilon$. In other words, we have found both $\gDl\in M_\lb B_F$ and $\gPl\in M_\lb B_F$.
Combining these considerations we find 
\begin{align}\label{concentration-hoeffding-applicat-simple-thm-h0}
\lb \snorm\gDl_{F}^{p} + \RP L\gDl-\RPxB LF - A_{p}(\lb) 
\leq 
2 \sup_{f\in M_\lb B_F} | \RP L f-\RD Lf |
\end{align}
for all $D\in \denseblo$. In the following, we bound the supremum on the right-hand side. To this end, we note that for 
$f\in F$ with $\snorm f_F\leq  M_\lb$ 
we have 
\begin{align*}
L(x,y,f(x)) \leq B+ B|f(x)|^q \leq B+B\inorm f^q 
\leq B + B \snorm f_F^q
\leq B + B  M_\lb^q 
\leq 2B  M_\lb^q\, ,
\end{align*}
where we used $\lb\in (0,1]$.
For a single $f\in M_\lb B_F$,
 Hoeffding's inequality, see e.g.~\cite[Theorem 8.1]{DeGyLu96} or
 \cite[Theorem 6.10]{StCh08},
thus ensures
\begin{align}\label{concentration-hoeffding-applicat-simple-thm-h1}
\P^n\biggl( D\in (X\times Y)^n: 
\bigl|\RD L f - \RP Lf| > 2 B  M_\lb^q \cdot \sqrt{ \frac {\t}{2n}}\,\, \biggr) \leq \eul^{-\t} \, .
\end{align}
To derive a bound that holds simultaneously over all $f\in M_\lb B_F$
we recall 
that the
covering numbers of a subset $A\subset \Lx \infty \PX$ are defined by 
\begin{align*}
\ca N(A,\inorm\mycdot ,\e) 
:=
\inf\biggl\{
m \geq 1: \exists\, f_{1},\dots,f_{m}\in A \mbox{ such that } A\subset \bigcup_{i=1}^{m} (f_i + \e    B_{\inorm\mycdot } )
\biggr\}  
\end{align*}
for all $\e>0$. Furthermore, recall that any such  collection $f_1,\dots,f_m$ is called an $\e$-net of $A$.
By our assumption \eqref{eq:sup-entropy-numbers} on the entropy numbers we then have 
\begin{align}\label{eq:log-cover-bound}
\ln \ca N(B_F,\inorm\mycdot ,\e) \leq 2 \Bigl(\frac a \e   \Bigr)^{2\gamma}\, , \myqquad \e>0
\end{align}
by a well-known relationship between covering and entropy numbers, 
see e.g.~\cite[Lemma 6.21]{StCh08}.
Now let $\ca F_{\e} = \{f_1,\dots,f_m \}\subset M_\lb B_{F}$ be an $\e$-net of $M_\lb B_{F}$ 
with cardinality 
\begin{align*}
\ln |\ca F_{\e}|= \ln \ca N\bigl(M_\lb  B_{F},\inorm {\cdot},\e\bigr)
= \ln \ca N\bigl(  B_{F},\inorm {\cdot}, M_\lb^{-1}\e\bigr)
\leq 
2 \Bigl(\frac {a M_\lb} \e   \Bigr)^{2\gamma} \, .
\end{align*}
Now, a well-known upper bound on the local Lipschitz constants of 
convex functions, see e.g.~\cite[Lemma A.6.5]{StCh08} or directly \cite[proof of Proposition 1.6]{Phelps93} gives
\begin{align*}
\bigl| L(x,y,t) - L(x,y,t')    \bigr|  
 \leq \frac 2 {M_\lb} \sup_{s \in [-2M_\lb, 2M_\lb]} L(x,y,s) \cdot |t-t'|
&\leq  \frac 2 {M_\lb} \bigl( B + |2M_\lb|^q\bigr) |t-t'| \\
&\leq 
2^{2+q} M_\lb^{q-1} \cdot |t-t'|
\end{align*}
for all $t,t'\in [-M_\lb, M_\lb]$. For $f\in  M_\lb B_{F}$ and $f_j\in \ca F_\e$ with $\inorm{f-f_j}\leq \e$ we thus find 
\begin{align*}
\bigl| L(x,y,f(x)) - L(x,y,f_j(x))    \bigr| \leq 2^{2+q} M_\lb^{q-1} \cdot \e
\end{align*}
for all $(x,y) \in X\times Y$. This in turn yields both 
$| \RP L f -  \RP L {f_j}| \leq 2^{2+q} M_\lb^{q-1} \cdot \e$
and $| \RD L f -  \RD L {f_j}| \leq 2^{2+q} M_\lb^{q-1} \cdot \e$, and therefore we obtain 
\begin{align*}
\sup_{f\in  M_\lb B_{F}} | \RP L f-\RD Lf | 
\leq 
\sup_{  f_j\in \ca F_\e } |\RP L {f_j}-\RD L{f_j} | + 2^{3+q} M_\lb^{q-1} \cdot \e
\end{align*}
Combining \eqref{concentration-hoeffding-applicat-simple-thm-h1} with a simple union bound, we thus find 
\begin{align*}
&\P^n\biggl( D\in (X\times Y)^n: \sup_{f\in  M_\lb B_{F}} | \RP L f-\RD Lf | > 2 B  M_\lb^q \cdot \sqrt{ \frac {\t}{2n}} + 2^{3+q} M_\lb^{q-1}  \cdot \e\biggr) \\
&\leq \P^n\biggl( D\in (X\times Y)^n: \sup_{  f_j\in \ca F_\e } |\RP L {f_j}-\RD L{f_j} | > 2 B  M_\lb^q \cdot \sqrt{ \frac {\t}{2n}}\,\,\biggr) \\
&\leq |\ca F_\e| \eul^{-\t} \, .
\end{align*}
By an elementary variable transformation for $\t$ we thus obtain
\begin{align}\label{concentration-hoeffding-applicat-simple-thm-h2}
 \sup_{f\in  M_\lb B_{F}} | \RP L f-\RD Lf | \leq  2 B  M_\lb^q \cdot \sqrt{ \frac {\t + \ln |\ca F_\e|}{2n}} + 2^{3+q} M_\lb^{q-1}  \cdot \e
\end{align}
with probability $\P^n$ not less than $1-\eul^{-\t}$. It remains to find a good value for $\e$. To this end, we first note that 
\begin{align} \nonumber
&
B  M_\lb^q \cdot \sqrt{ \frac {\t + \ln |\ca F_\e|}{2n}} + 2^{2+q} M_\lb^{q-1} 
\cdot \e \\ \nonumber
&\leq 
B  M_\lb^q \cdot \sqrt{ \frac {\t}{2n}} + B  M_\lb^q \cdot \sqrt{ \frac {\ln |\ca F_\e|}{2n}} + 2^{2+q} M_\lb^{q-1} 
\cdot \e \\ \label{concentration-hoeffding-applicat-simple-thm-hx}
&\leq B  M_\lb^q \cdot \sqrt{ \frac {\t}{2n}}  +  B  M_\lb^q \cdot \Bigl(\frac {a M_\lb} \e   \Bigr)^{\gamma} \cdot n^{-1/2}  + 2^{2+q} M_\lb^{q-1} 
\cdot \e \, .
\end{align}
If we now solve $M_\lb^{q+\g}\e^{-\g} n^{-1/2} = M_\lb^{q-1} 
\cdot \e$ for $\e$ we find 
 $\e=   n^{-1/(2+ 2\g)} M_\lb$, which in turn 
 gives 
 \begin{align}\label{concentration-hoeffding-applicat-simple-thm-hxx}
 M_\lb^{q-1} 
\cdot \e
= n^{-\frac {1}{2+2\g}} \cdot M_\lb^{q }\, .
 \end{align}
Let us now fix a data set $D\in \denseblo$ satisfying \eqref{concentration-hoeffding-applicat-simple-thm-h2}. By construction, the probability $\P^n$ of such $D$ is not less than $1-\varrho - \eul^{-\t}$. In addition, the estimate \eqref{concentration-hoeffding-applicat-simple-thm-hx}, the definition of $\e$, and    \eqref{concentration-hoeffding-applicat-simple-thm-hxx} in combination with \eqref{concentration-hoeffding-applicat-simple-thm-h2} show
\begin{align*}
\sup_{f\in  M_\lb B_{F}} | \RP L f-\RD Lf | 
&\leq  2B  M_\lb \cdot \sqrt{ \frac {\t}{n}}   + 2(a^\g B + 2^{2+q}) n^{-\frac {1}{2+2\g}} \cdot M_\lb^{q } \\
&\leq 
4(a^\g B + 2^{2+q}) n^{-\frac {1}{2+2\g}} \cdot M_\lb^{q} \sqrt{\t}\, , 
\end{align*}
where in the last step we used $M_\lb\geq 1$, $q\geq 1$, $\g>0$ $a\geq 1$, and $\t\geq 1$.
Combining this with \eqref{concentration-hoeffding-applicat-simple-thm-h0} we obtain the assertion.
\end{proof}

\section{A Refined Oracle Inequality for Clipped RERMs}\label{app:B}

In this appendix we generalize the notion of 
$\epsilon$-approximate clipped regularized
empirical risk minimization ($\e$-CR-ERMs)  of \cite[Chapter 7.4]{StCh08}
to include data-dependent optimization  accuracies. We then generalize 
the oracle inequality \cite[Theorem 7.20]{StCh08} to this new class of 
algorithms. 

Let us begin by introducing the generalization of $\e$-CR-ERMs in the following definition.

\begin{definition}\label{def:eps-cr-erm}
	Let $\loss:X\times Y\times \R \to [0,\infty)$ be a loss function that is clippable at $M$, 
	$\Fspace \subset \sLx 0 X$, and 
	 $\Upsilon:\Fspace \to [0,\infty)$. Then we call a learning method $D\mapsto f_D$ an $\epsilon_D$-CR-ERM
	 if for all $n\geq 1$, there exists a measurable map $(X\times Y)^n\to [0,\infty)$ denoted by $D\mapsto \epsilon_D$,
	 such that for all $D\in (X\times Y)^n$ we have both $f_D\in \Fspace$ and 
	\begin{align}
		\label{eq:eps_cr_erm}
		\Upsilon(f_D) + \risk \loss D {\fDc}  \le \inf_{f\in \Fspace} \Bigl(\Upsilon(f) +  \risk \loss D { f}  \Bigr) + \epsilon_D~.
	\end{align}
	In the following, the maps $D\mapsto \epsilon_D$ are called optimization accuracies.
\end{definition}

Note that the original definition in \cite[Definition 7.18]{StCh08} only considers constant 
 optimization  accuracies, that is $\epsilon_D = \epsilon$ for all data sets $D$.
Moreover, all remarks made after \cite[Definition 7.18]{StCh08} remain true: 
First, on the right-hand side of \eqref{eq:eps_cr_erm} the \emph{unclipped} loss is considered,
and hence an $\epsilon_D$-CR-ERM does not necessarily minimize the regularized clipped
empirical risk  $\risk \loss D {\clippdot} + \Upsilon(\mycdot)$ up to $\epsilon_D$. Second, 
 $\epsilon_D$-CR-ERMs do not need to minimize the regularized risk   $\risk \loss D {\mycdot} + \Upsilon(\mycdot)$ up to $\epsilon_D$
 either, because on the left-hand side
of \eqref{eq:eps_cr_erm} the \emph{clipped} function is considered. 
And third, if we have an $\epsilon_D$-approximate minimizer
of the unclipped regularized risk, then it automatically satisfies \eqref{eq:eps_cr_erm}.

As mentioned at the beginning our next and major  goal is to generalize 
the oracle inequality \cite[Theorem 7.20]{StCh08} to $\epsilon_D$-CR-ERMs. 
Let us begin by collecting some auxiliary results:
The following well-known and elementary lemma, see e.g.~\cite[Lemma 7.1]{StCh08}.

\begin{lemma}\label{concentration:model-selection-lemma}
For $q\in (1,\infty)$, define  $q'\in (1,\infty)$  by $\frac 1 q+\frac1{q'}=1$. Then we have 
$$
ab \leq \frac{a^q}q+\frac{b^{q'}}{q'}
$$
and $(qa)^{ 2 /q}(q'b)^{ 2/{q'}}\leq  (a+b)^2$ for all $a,b\geq 0$.
\end{lemma}

The next technical lemma establishes certain variance bounds for fractions. In a nutshell,
it outsources some of the tedious calculations done in the proof of  \cite[Theorem 7.20]{StCh08}.

\begin{lemma}\label{chapter3-lemma:fraction-of-risks}
 Let $(Z,\sA, \P)$ be a probability space, and $h:Z\to \R$ be a measurable
 map with $\E_\P h \geq 0$. Assume that there exist constants $B\geq 0$, 
$V\geq 0$ and $\vt\in [0,1]$
such that 
\begin{align*}
\inorm h \leq B 
\qquad\qquad 
\mbox{ and }
\qquad\qquad 
\E_\P h^2 \leq V \cdot (\E_\P h)^\vt\, ,
\end{align*}
where in the assumption on $V$ and $\vt$  we used the convention $0^0 := 1$.
 For $r>0$ we define 
 \begin{align}\label{chapter3-lemma:fraction-of-risks-eq1}
  g_{h,r}:= \frac{\E_\P h - h}{\E_\P h + r}\, .
 \end{align}
Then we have $\E_\P g_{h,r} = 0$ and $\inorm{g_{h,r}}\leq 2Br^{-1}$
as well as 
\begin{align}\label{chapter3-lemma:fraction-of-risks-eq2}
 \E_\P g_{h,r}^2 \leq V r^{\vt-2}\, .
\end{align}
 Finally, 
for all $c\geq 0$ we have 
\begin{align}\label{chapter3-lemma:fraction-of-risks-eq3}
 \sqrt{c \cdot  \E_\P h^2} \leq \E_\P h + (c\,V)^{\frac{1} {2-\vt}  }\, .
\end{align}
\end{lemma} 

\begin{proof}
The identity $\E_\P g_{h,r} = 0$ is obvious. Moreover, our assumptions on $h$ yield $-2B \leq \E_\P h - h \leq 2B$, and 
consequently, $\E_\P h \geq 0$ implies
\begin{align*}
 \inorm{g_{h, r}} & = \frac{\inorm{\E_\P h  - h }}{\E_\P h  + r} \leq \frac{2B}{r}\, .
\end{align*}

To establish the variance bound in \eqref{chapter3-lemma:fraction-of-risks-eq2}, we first note that in  the case $\vt = 0$
we have $\E_\P h^2 \leq V$, and hence 
we find $\E_\P g_{h,r}^2 \leq \E_\P h^2 \cdot r^{-2}  \leq V r^{\vt-2}$.
Let us therefore consider the case $\vt>0$. If $\E_\P h = 0$,
we find $\E_\P h^2 \leq V\cdot  (\E_\P h)^\vt = 0$, and hence we have 
\begin{align*}
 \E_\P g_{h,r}^2 =    \frac{\E_\P h^2 - (\E_\P h)^2}{(\E_\P h + r)^2} = 0\, .
\end{align*}
This implies \eqref{chapter3-lemma:fraction-of-risks-eq2}.  
Finally, 
for the (interesting) case $\E_\P h>0$, we first note that Lemma \ref{concentration:model-selection-lemma}
shows 
\begin{align*}
 (\E_\P h + r)^2 \geq (q\, \E_\P h)^{\frac  2{q}} \cdot (q'r)^{ \frac 2 {q'}}
\end{align*}
for all $q,q'\in (1,\infty)$ with $\frac 1 q+\frac 1{q'} = 1$.
Applying this inequality with $q:= \frac 2\vt$ and $q' := \frac{2}{2-\vt}$ 
then yields 
\begin{align*}
 (\E_\P h + r)^2 
 \geq 
 \frac{4}{\vt^\vt \cdot (2-\vt)^{2-\vt}} \cdot   (\E_\P h)^\vt r^{2-\vt}
 \geq 
 (\E_\P h)^\vt r^{2-\vt}\, ,
\end{align*}
where in the last step we used that the derivative of 
$f(\vt):= \vt^\vt \cdot (2-\vt)^{2-\vt}$ is $\vt^\vt \cdot (2-\vt)^{2-\vt} \cdot \ln \frac{\vt}{2-\vt}$, and therefore $f$ is decreasing on $(0,1]$, leading to $f(\vt) \in [1,4)$ on $(0,1]$.
With this information  we now obtain 
\begin{align*}
  \E_\P g_{h,r}^2 
  \leq     
  \frac{\E_\P h^2}{(\E_\P h + r)^2}
  \leq 
  \frac{\E_\P h^2 }{(\E_\P h)^\vt r^{2-\vt}}\, ,
\end{align*}
and using the assumed  $\E_\P h^2 \leq V \cdot (\E_\P h)^\vt$ we then obtain \eqref{chapter3-lemma:fraction-of-risks-eq2}.

Let us finally establish \eqref{chapter3-lemma:fraction-of-risks-eq3}. To this end, we start from the assumption and again invoke Lemma 
\ref{concentration:model-selection-lemma}
to find
\begin{align*}
 \sqrt{c \cdot  \E_\P h^2}
\leq 
 \sqrt{c \cdot V \cdot (\E_\P h)^\vt}
 &= 
 \sqrt{\frac{c \,\vt^\vt\cdot V}{2^\vt}} \cdot \biggl(\frac{2 \E_\P h}{\vt}\biggr)^{\vt/2} \\
 &\leq 
 \frac 1q\cdot  \biggl(\frac{c \,\vt^\vt\cdot V}{2^\vt}\biggr)^{q/2}
 +
 \frac 1 {q'} \cdot \biggl(\frac{2 \E_\P h}{\vt}\biggr)^{q'\vt/2}
\end{align*}
for all $q,q'\in (1,\infty)$ with $\frac 1 q+\frac 1{q'} = 1$.
Choosing $q := \frac{2}{2-\vt}$ and $q' := \frac 2\vt$   yields
\begin{align*}
 \frac 1q\cdot  \biggl(\frac{c \,\vt^\vt\cdot V}{2^\vt}\biggr)^{q/2}
 +
 \frac 1 {q'} \cdot \biggl(\frac{2 \E_\P h}{\vt}\biggr)^{q'\vt/2}
= 
  \frac {(2-\vt) \cdot \vt^{\frac{\vt}{2-\vt}}}{2 \cdot 2^{\frac{\vt}{2-\vt}}}\cdot  \bigl(c  \cdot V\bigr)^{\frac{1}{2-\vt}}
 +
 \E_\P h \, ,
\end{align*}
and using that $\vt \mapsto  f(\vt) := (2-\vt) \cdot \vt^{\frac{\vt}{2-\vt}} \cdot  { 2^{-\frac{\vt}{2-\vt}}}$
is monotonically decreasing on $(0,1]$, leading to $f(\vt) \leq 2$, and we then conclude \eqref{chapter3-lemma:fraction-of-risks-eq3}.
\end{proof}

Like the analysis conducted in \cite[Chapter 7.4]{StCh08} we also need some tools 
from empirical process theory. Fortunately, most of the required results can be taken 
directly from \cite[Chapter 7.3]{StCh08}. There is, however, one exception, namely \emph{peeling},
which needs to be slightly adapted. This is the goal of the next result, which generalizes 
\cite[Theorem 7.7]{StCh08}.

\begin{theorem}\label{concentration-peeling}
Let $(Z,\ca A,\P)$ be a 
probability space, $(T,d)$ be a sep\-arable metric space,  
$\Psi:T\to [0,\infty)$ be  a continuous  function,  
and $(g_t)_{t\in T}\subset \sL 0 Z$ be a
Carath\'eo\-dory family, that is, $t\mapsto g_t(z)$ is continuous for all $z\in Z$. 
We define 
$
r^{*}:= \inf \{ \Psi(t):t\in T\}
$. 
Moreover, let $\a\in (0,1)$ and 
$\p:(r^{*},\infty)\to [0,\infty)$ be a function such that  $\p(2 r)\leq  2^\a \p(r)$ and
$$
\E_{z\sim \P} \sup_{\substack{t\in T\\ \Psi(t)\leq r}} |g_t(z)| \leq \p(r) 
$$
for all $r>r^{*}$. Then,  for all $r>r^*$,  we  have 
$$
\E_{z\sim \P} \sup_{t\in T} \frac { g_t(z)}{\Psi(t)+r} \leq  \frac {2+2^\a}{2-2^\a} \cdot   \frac {\p(r)} r\, .
$$
\end{theorem}

\begin{proof}
For $z\in Z$  and  $r>r^{*}$,  we have 
\begin{align*}
\sup_{t\in T} \frac {g_t(z)}{\Psi(t)+r}
& \leq 
\sup_{\substack{t\in T\\ \Psi(t)\leq r}} \frac {|g_t(z)|} r 
+
\sum_{i=0}^\infty\, \sup_{\substack{t\in T\\ \Psi(t) \in[ r2^{i}, r2^{i+1}]}} \,   \frac{|g_t(z)|}{r2^i+r}\, ,
\end{align*}
where we used the convention $\sup\emptyset := 0$.
Consequently, 
we obtain 
\begin{align*}
\E_{z\sim \P} \sup_{t\in T} \frac {g_t(z)}{\Psi(t)+r}
& \leq 
\frac{\p(r)}{r} +  
\frac{1}{r}\sum_{i=0}^\infty \frac{1}{2^i+1}\,\, \E_{z\sim \P} \!\!\!\! \sup_{\substack{t\in T\\ \Psi(t) \leq r2^{i+1}}} |g_t(z)| \\
&
\leq
 \frac{1}{r}\biggl(\p(r)+\sum_{i=0}^\infty
\frac{\p(r2^{i+1})}{2^i+1}\biggr) \, .
\end{align*}
Moreover,  induction yields $\p(r2^{i+1})  \leq 2^{\a(i+1)}{\p}(r)$,  
$i\geq 0$, 
and hence we obtain
\begin{align*}
\E_{z\sim \P} 
\sup_{t\in T} \frac {g_t(z)}{\Psi(t)+r} 
\leq  
\frac{\p(r)}{r}\biggl(1+\sum_{i=0}^\infty \frac{2^{\a(i+1})}{2^i+1}\biggr) 
&\leq  
\frac{\p(r)}{r}\biggl(1+  2^\a \sum_{i=0}^\infty 2^{i(\a-1)}\biggr) \\
&= \frac{\p(r)}{r}\biggl(1+  \frac{2^\a}{1-2^{\a-1}} \biggr) \, .
\end{align*}
From this estimate, we easily obtain the assertion.
\end{proof}

Finally, we need to recall a few more notations from \cite[Chapter 7.4]{StCh08}.
To this end, let
$L:X\times Y\times \R\to [0,\infty)$ be a   loss that can be clipped at some $M>0$ 
in the sense of \eqref{eq:def-clippable}
and $\ca F\subset \sL 0 X$ be a subset
that is 
 equipped with a complete, separable metric $d$
dominating  the pointwise convergence in the sense of \cite[Lemma 2.11]{StCh08},
i.e.~convergence with respect to $d$ implies pointwise convergence.
Moreover, let
$\U:\ca F\to [0,\infty)$ be a function that is measurable with respect to the corresponding Borel $\sigma$-algebra on $\ca F$.
Finally,  let $\P$ be a distribution on $X\times Y$ for which
there exists a Bayes decision function $\fpb:X\to [-M,M]$.
For the oracle value
\begin{align}\label{concentration:def-rstar}
r^{*}:= \inf_{f\in \ca F}\Bigl( \U(f) + \RP L {\clippf}\Bigr) - \RPB L
\end{align}
and
$r>r^{*}$, we then write   
\begin{align}\label{concentration:def-Fr}
\ca F_{r} & :=  \bigl\{ f\in \ca F\,:\, \U(f) + \RP L{\clippf}-\RPB L\leq r \bigr\}\, , \\ \label{concentration:def-Hr}
\ca H_{r} & :=  \bigl\{ L\circ \clippf- L\circ \fpb\,:\, f\in \ca F_{r}\bigr\}\, .
\end{align}
Moreover, recall that 
for a given data set $D = ((x_1,y_1), \dots,(x_n,y_n))$ the  empirical Rademacher average of $\ca H_{r}$ is defined by 
\begin{align*}
\radd {\ca H_r}  := \E_\nu \sup_{h\in \ca H_r} \biggl| \frac 1 n \sum_{i=1}^n \e_i h(x_i,y_i)  \biggr|\, , 
\end{align*}
where $(\Theta, \ca C, \nu)$ is a probability space and 
$\e_1,\dots,\e_n:\Theta\to \{-1,1\}$ are i.i.d.~random variables with $\nu(\e_i=1) = \nu(\e_i= -1) = 1/2$.

With these preparations, the generalization of the oracle inequality \cite[Theorem 7.20]{StCh08}
reads as follows.

\begin{theorem}\label{concentration:rerm2}
Let $L:X\times Y\times \R\to [0,\infty)$ be a continuous loss  that can be clipped at $M>0$ and that satisfies 
\eqref{concentration:def-sup-bound-neu} for a constant $B>0$.
Moreover, let  $\ca F\subset \sL 0 X$ be a subset
that is 
 equipped with a complete, separable metric  
dominating the pointwise convergence, and let 
$\U:\ca F\to [0,\infty)$ be a continuous  function.
Given a distribution 
$\P$ on $X\times Y$ that has a Bayes decision function $\fpb:X\to [-M,M]$ and that
satisfies
 \eqref{concentration:def-var-bound}, we define 
$r^*$ and $\ca H_r$  by \eqref{concentration:def-rstar}
and \eqref{concentration:def-Hr}, respectively.
Assume that  for fixed $\a\in (0,1)$ and   $n\geq 1$ there exists a 
$\p_{n}:[0,\infty)\to [0,\infty)$ satisfying the doubling condition $\p_{n}(2 r)\leq  2^\a \p_{n}(r)$ as well as 
\begin{align}\label{concentration:rerm2-rad-est} 
\E_{D\sim \P^n} \radd {\ca H_r}     \leq \p_{n}(r) 
\end{align}
for all $r>r^{*}$.
Finally, fix an  $\fO \in \ca F$ and a   $B_0\geq B$ such that $\inorm {L\circ \fO}\leq B_0$.
Then, for all fixed %
$\t>0$  and $r>0$ satisfying 
\begin{align}\label{conc-adv:crerm-r-cond}
r> \max\biggl\{ 8 \cdot   \frac {2+2^\a}{2-2^\a}\cdot  \p_n(r), \biggl(\frac{72 V\t}n\biggr)^{\frac 1{2-\vt}},   \frac {5B_0\t }{n}, r^*\biggl\}\, ,
\end{align}
 every measurable  $\epsilon_D$-CR-ERM
satisfies  
\begin{align*}
\U (f_D) + \RP L \fDc - \RPB L  \leq   6 \bigl( \U(\fO) + \RP L {\fO} -\RPB L\bigr) 
+ 3 r  + 3\epsilon_D 
\end{align*}
with probability $\P^n$ not less than $1-3\eul^{-\t}$.
\end{theorem}

Before we prove this oracle inequality, we note that the original 
\cite[Theorem 7.20]{StCh08} only considered the case $\epsilon_D \equiv \epsilon$. 
In addition, the doubling condition there only considered the case $\a= 1/2$. Despite these differences, however,
the proof of Theorem \ref{concentration:rerm2} is almost identical to that of \cite[Theorem 7.20]{StCh08}.

\begin{proof}
We   define $h_f:= L\circ f - L\circ \fpb$ for all $f\in \sL 0 X$. 
By the definition of ${f_D}$,  we then have  
$$
\U(f_D)+\E_\D \hfc D
\leq \U(\fO) + \E_\D h_\fO+\epsilon_D\, ,
$$
and consequently we obtain 
\begin{align}\nonumber
&
\U (f_D) + \RP L \fDc - \RPB L  \\ \nonumber
& =  
\U(f_D) + \E_{\P} \hfc D\\ \nonumber
& \leq 
\U(\fO) + \E_\D h_\fO - \E_\D \hfc D   +\E_{\P} \hfc D +\epsilon_D \\ \label{concentration:rerm2-h0}
& = 
(\U(\fO) + \E_{\P} h_{\fO} )
+ (\E_\D h_\fO -  \E_{\P} h_{\fO}  )+(\E_{\P} \hfc D - \E_\D \hfc D)+\epsilon_D\qquad \qquad
\end{align}
for all $D\in (X\times Y)^n$. 
Let us first bound the term $\E_\D h_\fO - \E_\P h_\fO$. 
To this end, we further split this difference into
\begin{equation}\label{concentration:rerm2-g1}
\E_\D h_\fO - \E_\P h_\fO = \bigl( \E_\D (h_\fO-\hfc 0) - \E_\P (h_\fO-\hfc 0)\bigr) + \bigl(\E_\D \hfc 0 - \E_\P \hfc 0\bigr)\, .\,\, 
\end{equation}
We first bound the first term of RHS in \eqref{concentration:rerm2-g1} and then bound the second term.
Now observe that $L\circ \fO - L\circ \fOc\geq 0$ implies 
$h_\fO - \hfc 0  = L\circ \fO - L\circ \fOc\in [0,B_0]$, and hence we obtain 
$$
\E_\P\bigl((h_\fO-\hfc 0)-\E_\P(h_\fO-\hfc 0)\bigr)^2 \leq \E_\P (h_\fO-\hfc 0)^2 \leq B_0 \,\E_\P (h_\fO-\hfc 0)\, .
$$
Consequently, Bernstein's
 inequality, see e.g.~\cite[Theorem 6.12]{StCh08},
 applied to the function $h:= (h_\fO-\hfc 0)-\E_\P(h_\fO-\hfc 0)$
shows that
$$
 \E_\D (h_\fO-\hfc 0) - \E_\P (h_\fO-\hfc 0)<     \sqrt{\frac{2\t B_0\, \E_\P (h_\fO-\hfc 0)}{n}}+ \frac {2B_0\t }{3n}    
$$
holds  with probability $\P^n$ not less than $1-\eul^{-\t}$.
Moreover,
using  $\sqrt{ab}\leq \frac{a} 2+\frac{b}2$,
with $a:= 2\E_\P (h_\fO-\hfc 0)$
we 
find 
$$
\sqrt{\frac{2\t B_0\, \E_\P (h_\fO-\hfc 0)}{n}} \leq \E_\P(h_\fO-\hfc 0) + \frac {B_0\t }{2n}\, ,
$$
and consequently we have 
\begin{equation}\label{concentration:rerm2-h1}
\E_\D (h_\fO-\hfc 0) - \E_\P (h_\fO-\hfc 0)<       \E_\P(h_\fO-\hfc 0)  +  \frac {7B_0\t }{6n}     
\end{equation}
with probability $\P^n$ not less than $1-\eul^{-\t}$. 

To bound the remaining term in (\ref{concentration:rerm2-g1}), i.e.~$\E_\D \hfc 0 - \E_\P \hfc 0$,
we now observe that (\ref{concentration:def-sup-bound-neu}) implies  $\inorm{\hfc 0}\leq B$,
and hence we have 
\begin{align*}
\inorm{\hfc 0 - \E_\P \hfc 0}
\leq 
2\inorm {L\circ \fOc- L\circ \fpb}  \leq  2B \, .
\end{align*}
Bernstein's inequality in combination with \eqref{concentration:def-var-bound} and
\eqref{chapter3-lemma:fraction-of-risks-eq3}
then
shows 
that with probability $\P^n$ not less than $1-\eul^{-\t}$  we have 
\begin{equation*}%
 \E_\D \hfc 0 - \E_\P \hfc 0<   \E_\P \hfc 0 + \Bigl(\frac{2 V\t}n\Bigr)^{\frac 1{2-\vt}}  +   \frac {4B\t}{3n}     \, . 
\end{equation*}
By combining this estimate with \eqref{concentration:rerm2-h1} and  \eqref{concentration:rerm2-g1},     we now obtain 
that with probability $\P^n$ not less than $1-2\eul^{-\t}$ we have 
\begin{equation}\label{concentration:rerm2-g3}
\E_\D h_\fO - \E_\P h_\fO < \E_\P h_\fO + \Bigl(\frac{2 V\t}n\Bigr)^{\frac 1{2-\vt}}  +   \frac {4B\t}{3n}  + \frac {7B_0\t }{6n}  \, ,
\end{equation}
i.e., we have established a bound on the second term in \eqref{concentration:rerm2-h0}.

In the following, we consider the two case $n< 72\t$ and $n \ge 72\t$, separately.
For the case $n< 72\t$, the assumption $B^{2-\vt}\leq V$  implies $B<r$.
Combining \eqref{concentration:rerm2-g3} with \eqref{concentration:rerm2-h0} and using both $B\leq B_0$ 
and $\E_{\P} \hfc D - \E_\D \hfc D\leq 2B$
we hence find  
\begin{align} \nonumber
&
\U (f_D) + \RP L \fDc - \RPB L  \\  \label{eq:n-small-estimate}
& \leq 
\U(\fO) + 2 \E_{\P} h_{\fO} 
+ \Bigl(\frac{2 V\t}n\Bigr)^{\frac 1{2-\vt}}  +  \frac {5B_0\t }{2n}   +(\E_{\P} \hfc D - \E_\D \hfc D)+\epsilon_D\\   \nonumber
& \leq 
 6 \bigl( \U(\fO) + \RP L {\fO} -\RPB L\bigr)  + 3r + \epsilon_D
\end{align}
with probability $\P^n$ not less than $1-2\eul^{-\t}$. 

Consequently, it remains to consider the case 
 $n\geq 72\t$.
In order to establish a non-trivial bound on the term $\E_\P \hfc D - \E_\D \hfc D$ in \eqref{concentration:rerm2-h0},  we 
define  functions 
$$
g_{f,r}:= \frac {\E_\P \hfco- \hfco}{\U(f)+\E_\P \hfco +r}\, , \qquad \qquad f\in \ca F,\, r>r^{*}.
$$
By Lemma \ref{chapter3-lemma:fraction-of-risks} applied to $h:= \hfco\,$  
in combination with $\U(f) \geq 0$ we then find both $\inorm {g_{f,r}}\leq 2Br^{-1}$
and 
\begin{equation}\label{concentration:rerm2-g4}
\E_\P  g_{f,r}^2
\leq
 V{r^{\vt-2}} \, .
\end{equation} 
By simple modifications of \cite[Lemma 7.6]{StCh08}
and its proof, 
we further see that all families of maps considered below are so-called Carath\'eodory families in the sense of 
\cite[Definition 7.4]{StCh08},
and therefore all suprema considered below are indeed measurable, see the discussion following 
\cite[Definition 7.4]{StCh08}.
With these preparations we now observe that 
symmetrization, see e.g.~\cite[Proposition 7.10]{StCh08},   yields
$$
\E_{D\sim \P^n} \sup_{f\in \ca F_r} \bigl| \E_\D (\E_\P \hfco - \hfco  )      \bigr|  
\leq 
2 \E_{D\sim \P^n} \radd {\ca H_r} 
 \leq 2\p_n(r)\, .
$$
Peeling with the help of Theorem 
\ref{concentration-peeling} together with $\ca F_r = \{f\in \ca F: \U(f) + \E_\P \hfco\leq r \}$ hence gives 
$$
\E_{D\sim \P^n} \sup_{f\in \ca F} \bigl| \E_\D  g_{f,r}\bigr| \leq  2\cdot \frac {2+2^\a}{2-2^\a} \cdot \frac{ \p_n(r)}r\, .
$$
By Talagrand's inequality in the form of \cite[Theorem 7.5]{StCh08}  applied to 
$\g:= 1/4$, we therefore obtain 
$$
\P^n
\biggl(
 D\in (X\times Y)^n  : \sup_{f\in \ca F} \E_\D g_{f,r} < c_\a\cdot \frac{ \p_n(r)}r   + \sqrt{\frac{2V\t }{nr^{2-\vt}}}+ \frac {28B\t}{3n r}
\biggr)
\geq
1- \eul^{-\t}
$$
for $c_\a := \frac 5 2\cdot   \frac {2+2^\a}{2-2^\a}$ and all $r>r^{*}$. Using the definition of $g_{{f_D},r}$, we thus have with probability $\P^n$ not less than $1-\eul^{-\t}$ that 
\begin{align*}
\E_\P \hfc D - \E_\D \hfc D 
&< 
  \bigl(\U({f_D})+ \E_\P \hfc D  \bigr)\biggl(\frac{c_\a \p_n(r)}r + \sqrt{\frac{2V\t }{nr^{2-\vt}}}+ \frac {28B\t}{3n r}\biggr) \\
&\qquad 
  +     c_\a \p_n(r)+\sqrt{\frac{2V\t r^\vt}{n}}+ \frac {28B\t}{3n}\, .
\end{align*}
By applying this estimate and \eqref{concentration:rerm2-g3} to our initial estimate \eqref{concentration:rerm2-h0}
we then see that 
\begin{align}\nonumber
\U({f_D}) + \E_{\P} \hfc D   
& < 
\U(\fO) + 2\E_{\P} h_{\fO}  + \Bigl(\frac{2 V\t}n\Bigr)^{\frac 1{2-\vt}}  + \frac {7B_0\t }{6n}     + \epsilon_D\\ \nonumber
& \qquad +
  \bigl(\U({f_D})+ \E_\P \hfc D  \bigr)\biggl(\frac{c_\a \p_n(r)}r + \sqrt{\frac{2V\t }{nr^{2-\vt}}}+ \frac {28B\t}{3n r}\biggr) \\ \label{concentration:rerm2-g6}
&\qquad 
  +     c_\a \p_n(r)+\sqrt{\frac{2V\t r^\vt}{n}}+ \frac {32B\t}{3n}
\end{align}
holds with probability $\P^n$ not less than $1-3\eul^{-\t}$.
Consequently, it remains to bound the various terms. To this end, we first observe that 
\begin{align*}
r\geq 8 \cdot   \frac {2+2^\a}{2-2^\a} \cdot  \p_n(r)
\end{align*}
implies $c_\a\p_n(r)r^{-1} \leq 1/3$ and $c_\a \p_n(r) \leq r/3$. Moreover, 
$r\geq  \bigl(\frac{72 V\t}n\bigr)^{1/({2-\vt})}$ yields 
$$
\Bigl( \frac{2V\t }{nr^{2-\vt}} \Bigr)^{1/2}\leq \frac 16 \qquad \mbox{ and } \qquad \Bigl( \frac{2V\t r^\vt}{n} \Bigr)^{1/2}\leq \frac r6\, .
$$ 
In addition, $n\geq 72 \t$,
$V\geq B^{2-\vt}$,  and $r\geq  \bigl(\frac{72 V\t}n\bigr)^{1/({2-\vt})}$ imply 
$$
\frac {28B\t}{3n r} 
= 
\frac 7 {54} \cdot \frac{72\t}{n} \cdot \frac B r 
\leq 
\frac 7 {54} \cdot  \Bigl(\frac{72 \t}n\Bigr)^{\frac 1{2-\vt}} \cdot \frac{V^{\frac 1{2-\vt}}} r
\leq 
\frac 7 {54}
$$
and  $\frac {32B\t}{3n} \leq \frac {4r}{27}$.
Finally, using $6\leq 36^{1/(2-\vt)}$, we obtain $ \bigl(\frac{2 V\t}n\bigr)^{1/({2-\vt})} \leq \frac r 6$. 
Using these elementary estimates in \eqref{concentration:rerm2-g6}, 
we see that
$$
\U({f_D}) + \E_{\P} \hfc D   
 < 
\U(\fO) + 2\E_{\P} h_{\fO}   + \frac {7 B_0\t }{6n} + \epsilon_D
+
 \frac {17}{27}  \bigl(\U({f_D})+ \E_\P \hfc D  \bigr)
+ \frac {22 r}{27}
$$ 
holds with probability $\P^n$ not less than $1-3\eul^{-\t}$.
We now obtain the assertion by some simple algebraic transformations together with $r> \frac {5B_0\t }{n}$.
\end{proof}

\section{Miscellaneous Material}
\label{app:C}

\begin{comment}
We collect some straightforward generalizations of already known material, beginning with the following proposition which calculates the $\mdash$-smoothness of empirical risks based on $M$-smooth losses.
\begin{proposition}\label{prop:m-smoothness-of-risks}
    Let the  loss $\loss$ be $M$-smooth and $H\hookrightarrow\calL_\infty(X)$. 
    Then the empirical risk $\calR_{\loss,D}$ is $M'$-smooth with $M'= M \norm{H\hookrightarrow \calL_\infty(X)}^2$.
\end{proposition}
\textbf{Here, this should just be a special case of relative smoothness, where $\psi = \norm{\cdot}^2 /2$. Also, do in in an RKBS.}
This result straightforwardly generalizes to the population risk $\mathcal R_{\loss,P}$, as well as to  RKBSs instead of RKHSs.
\begin{proof}
    We have
    \begin{align*}
        \norm{ \nabla \calR_{\loss,D} (f) - \nabla \calR_{\loss,D}(g)} 
        &= \norm{\frac 1n \sum_{i=1}^n (\loss ' (x_i,y_i,f(x_i))- \loss ' (x_i,y_i,g(x_i))) \delta_{x_i}}
        \\ &\le \norm{H \hookrightarrow \calL_\infty(X)} \abs{\frac 1n \sum_{i=1}^n (\loss ' (x_i,y_i,f(x_i))- \loss ' (x_i,y_i,g(x_i)))}
        \\ & \le \norm{H \hookrightarrow \calL_\infty(X)} \frac 1n \sum_{i=1}^n M \abs{f(x_i)-g(x_i)}
        \\ & \le M\norm{H \hookrightarrow \calL_\infty(X)}^2 \norm{f-g}~. 
    \end{align*}
\end{proof}
\end{comment}

\begin{lemma}\label{lem:reg-trafo}
Let  $\loss:X\times Y\times \R\to [0,\infty)$ be a loss 
satisfying the growth condition \eqref{eq:loss-growth-order-q} for some $q\in [1,\infty)$ and $\P$ be a distribution on $X\times Y$. Moreover, let 
$F\subset \sL q\PX$ be an RKBS on $X$ and 
$p,r\geq 1$.
Then for all $\lb\geq 0$ and $\g> A_r(\lb)$ we have 
\begin{align*}
A_p(\lb^{p/r} \g^{1-p/r}) \leq 2\g\, .
\end{align*}
In particular, if there exist constants $c>0$ and $\b>0$ such that $A_r(\lb) \leq c\lb^\b$ for all 
$\lb >0$, then for all $\lb>0$ we have 
\begin{align*}
A_p(\lb) \leq 2 c^{\frac p {p+\b(r-p)}} \lb^{\frac{\b r}{p+\b(r-p)}}\, .
\end{align*}
\end{lemma}

\begin{proof}
We first note that \eqref{eq:loss-growth-order-q}  in combination with $F\subset \sL q\PX$ shows
$\RP \loss f< \infty$ for all $f\in F$, and thus also $\RPxB \loss F<\infty$.

Now, our assumption $\g> A_r(\lb)$  ensures the existence of an $f\in F$ with 
\begin{align*}
\lb \norm f_F^r + \RP L f - \RPxB LF \leq  \g.
\end{align*}
Since $\RP L f - \RPxB LF\geq 0$, we then have $\lb \norm f_F^r \leq \g$, that is $\norm f_F\leq (\g/\lb)^{1/r}$. For 
$\kappa := \lb^{p/r}\g^{1-p/r}$ we thus find 
\begin{align*}
A_p(\kappa)
\leq 
\kappa \norm f_F^p + \RP L f - \RPxB LF
\leq 
\lb^{r/p}\g^{1-p/r} \lb^{-p/r}\g^{p/r} + \g
= 
2\g\, .
\end{align*}
To show the second assertion we again define $\kappa := \lb^{p/r}\g^{1-p/r}$. For $\e>0$ and 
$\g:= (1+\e)c \lb^\b$ the first assertion then yields 
\begin{align*}
A_p(\kappa) \leq 2\g =  2 (1+\e)c \lb^\b\, ,
\end{align*}
while our definitions of $\kappa$ and $\g$ implies $\kappa =\lb^{p/r} (1+\e)^{1-p/r}c^{1-p/r} \lb^{\b(1-p/r)}$. Solving this for $\lb$, inserting the solution into the previous estimate 
and letting $\e\to 0$ then yields the assertion with the 
help of  some simple algebraic transformations.
\end{proof}

\approvedd{
\begin{lemma}\label{lem:basic-measurability}
Let $X$ be a measurable space, $Y\subset \R$ measurable, and $F$ be a separable BSF on $X$ with $F\subset \sLx 0 X$. In addition, let $D\mapsto f_D$ be a learning method with $f_D\in F$ for all data sets $D \in (X\times Y)^n$ and all $n\geq 1$. Then the following statements are equivalent:
\begin{enumerate}
\item $D\mapsto f_D$ is a measurable learning method.
\item For all $n\geq 1$ the map $D\mapsto f_D$ is measurable with respect to the Borel $\sigma$-algebra on $F$.
\end{enumerate}
\end{lemma}

\begin{proof}
We write $\sigma(F') := \sigma(\{y': y'\in F'\})$ for the $\sigma$-algebra on $F$ that is generated by all bounded linear functionals on $F$. \ms{formulation a bit unclear: I think you mean the sigma algebra on $F$ such that all $y'\in F'$ are $F-\bbR$-measurable? }
In addition, $\sborel (F)$ denotes the Borel $\sigma$-algebra on $F$ that is generated by the norm topology. 
As $F$ is separable, it is well-known that $\sigma(F') = \sborel (F)$, see  e.g.~\cite[Proposition 1.1.1]{HyvNVeWe16}.
Moreover,   $\delta_x:F\to\R$ denotes the point evaluation functional at $x\in X$. Since $F$ is a BSF, we have $\delta_x\in F'$.
Finally, throughout the proof we fix $n\geq 1$ and only consider data sets $D \in (X\times Y)^n$.

\atob ab By assumption $(D,x)\mapsto f_D(x)$ is measurable. If we fix an $x\in X$, we thus obtain the measurability of 
\begin{align*}
D\mapsto f_D(x) = \langle \delta_x, f\rangle_{F',F}\, . 
\end{align*}
Consequently, $D\mapsto f_D$ is $\sigma(\{\delta_x: x\in X\})$-measurable, see e.g.~\cite[Corollary 1.82]{Klenke14}.
Now, $\spann \{\delta_x: x\in X\}$ is $w^*$-dense in $F'$, see e.g.~\cite[Proposition 4.3]{Steinwart26a}, and hence we have 
\begin{align*}
\sigma(\{\delta_x: x\in X\}) = \sigma(F')
\end{align*}
by e.g.~\cite[Proposition 1.1.1]{HyvNVeWe16}. Since $\sigma(F') = \sborel (F)$ we then obtain the assertion.

\atob ba  We first show that \ms{I find this notation really weird when only the domains are given. did you forget a name, or is it intended?} $X\times F\to \R$ given by $(x,f)\mapsto f(x)$ is $\ca A\otimes \sborel(F)$-measurable, where $\ca A$ denotes the $\sigma$-algebra on $X$. 
To this end, we note that for fixed $x\in X$ the map $f\mapsto f(x)$ is continuous since $F$ is a BSF. In addition, for fixed $f\in F$ the map 
$x\mapsto f(x)$ is measurable since we assumed $F\subset \sLx 0 X$. Moreover, $F$ is a Polish space by the assumed separability of $F$. Consequently, 
$(x,f)\mapsto f(x)$ is measurable with respect to the $\sigma$-algebra $\ca A\otimes \sborel(F)$ on the domain $X\times F$, see e.g.~\cite[Lemma III.14]{CaVa77}.
Moreover, the map $(x,D)\mapsto (x, f_D)$ is measurable with respect to the $\sigma$-algebra $\ca A\otimes \sborel(F)$ on the codomain \ms{why codomain, not domain as before?} $X\times F$
as $x\mapsto x$ is obviously $\ca A$-measurable and $D\mapsto f_D$ is 
$\sborel(F)$-measurable by assumption. The composition $(x,D)\mapsto f_D(x)$ of the two maps
 $(x,D)\mapsto (x, f_D)$ and $(x,f)\mapsto f(x)$ is thus measurable.
\end{proof}
}

\section{Classical Function Spaces: Entropy numbers and Embeddings}
\label{app:D}

In this supplement, we briefly recall in which sense the entropy number assumptions 
of Section \ref{sec:statistical_analysis} and the embedding assumptions \eqref{eq:embedding-by-norm} are satisfied for  classical BSFs. 

\subsection{Besov and Triebel-Lizorkin Spaces}

To this end, we will begin by considering the two scales of Besov spaces 
$\besov spq\Rd$ for $s\in [0,\infty)$, $p,q\in [1,\infty]$
and 
Triebel-Lizorkin spaces $\triebel spq\Rd$ for $s\in [0,\infty)$, $p\in [1,\infty)$ and 
$q\in [1,\infty]$, where the basic definition can be found in e.g.~\cite[Chapter 2.3.1]{Triebel83} or \cite[Chapter 2.2.1]{EdTr96}. For both scales, we speak of ``allowed values'' for $s$, $p$, or $q$, if they fall into the respective ranges specified above.
As in \cite{EdTr96} we use the notation $\besovtriebel spq\Rd:=F$ if either $F= \besov  spq\Rd$ or $F= \triebel spq\Rd$.

Recall that both scales are ordered in $q$, that is for $1\leq q_1\leq q_2\leq \infty$ and all allowed values for 
$s,p$ we have 
\begin{align}\label{eq:bf-space-mono-q}
\besov sp{q_1} \Rd \hookrightarrow \besov sp{q_2} \Rd \myqquad \mbox{and} \myqquad \triebel sp{q_1} \Rd \hookrightarrow \triebel sp{q_2} \Rd \, , 
\end{align}
see \cite[Proposition 2 in Chapter 2.3.2]{Triebel83}. Similarly, these scales are ordered in the ``smoothness index'', that is,
for $0\leq s_1\leq s_2<\infty$ and all allowed values of $p,q$ we have 
\begin{align}\label{eq:bf-space-mono-s}
\besov {s_2}pq \Rd \hookrightarrow \besov {s_1}pq \Rd \myqquad \mbox{and} \myqquad \triebel {s_2}pq \Rd \hookrightarrow \triebel {s_1}pq \Rd \, , 
\end{align}
see again   \cite[Proposition 2 in Chapter 2.3.2]{Triebel83}. In addition, the latter proposition shows 
\begin{align}\label{eq:bf-space-b-vs-f}
\besov sp{\min\{p,q\}}\Rd \hookrightarrow \triebel spq\Rd \hookrightarrow  \besov sp{\max\{p,q\}}\Rd
\end{align}
for all  $s\in [0,\infty)$, $p\in [1,\infty)$ and $q\in [1,\infty]$. Finally, \cite[Theorem 2.7.1]{Triebel83}
shows that for all allowed values $s_1,p_1,q$ and $s_2,p_2,q$ with  $s_1 < s_2$ and  $s_1 - d/p_1 = s_2 - d/p_2$     we have 
\begin{align}\label{eq:besov-triebel-s-vs-p}
\besov {s_2}{p_2}q\Rd \hookrightarrow \besov {s_1}{p_1}q \Rd \myqquad \mbox{and} \myqquad \triebel {s_2}{p_2}q\Rd \hookrightarrow \triebel {s_1}{p_1}1\Rd\, . 
\end{align}

\subsection{Classical Spaces within the Besov and Triebel-Lizorkin Scales}
 
 Let us now recall some results that compare the spaces of the two scales to classical function spaces. Here, we first mention
 the embeddings
 \begin{align}\label{eq:b-space-vs-lp}
 \besov 0p1\Rd & \hookrightarrow \Lx p \Rd \hookrightarrow \besov 0p\infty \Rd \, , \\ \label{eq:b0-space-vs-c}
  \besov 0\infty1\Rd & \hookrightarrow \cbuc\Rd \hookrightarrow \besov 0\infty\infty \Rd
 \,, 
 \end{align}
where $p\in [1,\infty)$ and $\cbuc\Rd$ denotes the space of all bounded and uniformly continuous functions, 
see \cite[(1) and (2) in Chapter 2.5.7 and Chapter 2.2.2]{Triebel83}. 
On the other hand, on the Triebel-Lizorkin scale the identity
\begin{align}\label{eq:fp-as-lp}
\triebel 0p2\Rd = \Lx p \Rd\, , \myqquad 1<p<\infty
\end{align}
holds, 
see e.g.~\cite[(5) in Chapter 2.5.6]{Triebel83}.
In addition, we have the 
 ``Sobolev embeddings''   
\begin{align}\label{eq:bs-space-vs-c}
\besov spq\Rd \hookrightarrow \cbuc\Rd\, , 
\end{align}
where we either have $p,q\in [1,\infty]$ and $s>d/p$, or $p\in [1,\infty]$, $q=1$, and $s= d/p$, 
see \cite[Chapter 2.8.3]{Triebel83}. 
Finally, combining \eqref{eq:bf-space-b-vs-f} with \eqref{eq:bs-space-vs-c} we find 
\begin{align}\label{eq:fs-space-vs-c}
\triebel spq\Rd \hookrightarrow  \besov sp{\max\{p,q\}}\Rd \hookrightarrow   \cbuc\Rd
\end{align}
for all $p\in [1,\infty)$, $q\in [1,\infty]$, and $s>d/p$.  Note that 
$\besovtriebel spq\Rd$  can be viewed as a BSF if $\besovtriebel spq\Rd \hookrightarrow \cbuc\Rd$.
 The embeddings 
\eqref{eq:b0-space-vs-c}, \eqref{eq:bs-space-vs-c}, and \eqref{eq:fs-space-vs-c} thus guarantee that the 
involved spaces are BSFs.

For $m\in \N$ and $p\in [1,\infty)$ we now write $\sobol mp\Rd$ for the usual Sobolev space on $\Rd$. 
In addition, $\bessel sp \Rd$ denotes the usual Bessel potential space for $s\geq 0$ and  $p\in (1,\infty)$. 
By \cite[Chapter 2.5.6]{Triebel83}
we then have 
\begin{align}\label{eq:sobol-is-triebel}
\sobol mp\Rd =  \bessel mp\Rd =  \triebel mp2\Rd \myqquad \mbox{ and } \myqquad \bessel sp\Rd = \triebel sp2\Rd
\end{align}
for all $p\in (1,\infty)$, $m\in \N$ and $s\geq 1$. In addition, \cite[(10) and (11) in Chapter 2.5.7]{Triebel83}
 shows 
 \begin{align}\label{eq:besov-vs-sobol}
 \besov mp1 \Rd & \hookrightarrow \sobol mp\Rd \hookrightarrow \besov mp\infty \Rd\, , \\ \label{eq:besov-vs-cm}
  \besov m\infty1 \Rd & \hookrightarrow \cmbuc m\Rd \hookrightarrow \besov m\infty\infty \Rd
 \end{align}
 for all $m\in \N$ and $p\in [1,\infty)$, where $\cmbuc m\Rd$ denotes the space of all 
 $m$-times continuously differentiable functions with $D^\a f \in \cbuc\Rd$ for all $\a\in \N_0^d$ with $|\a|\leq m$.
Setting $\sobol 0p\Rd := \Lx p\Rd$ and $\cmbuc 0\Rd := \cbuc\Rd$, we further see by 
\eqref{eq:b-space-vs-lp} and \eqref{eq:b0-space-vs-c} that \eqref{eq:besov-vs-sobol} and 
\eqref{eq:besov-vs-cm} actually hold for all $m\in \N_0$ and $p\in [1,\infty)$.
Finally, if  for $p\in [1,\infty)$ and $s>0$ with $s\neq \N$ we denote the Slobodeckij space by $\sobol sp\Rd$, then  
we have 
\begin{align}\label{eq:slob-vs-bf}
\sobol sp\Rd = \besov spp\Rd = \triebel spp\Rd , 
\end{align}
where the first identity follows from  \cite[(9) in Chapter 2.5.7 and (21) in Chapter 2.5.12]{Triebel83} and the second one directly follows from \eqref{eq:bf-space-b-vs-f}.
 
Finally, we follow \cite[Chapter 2.2.2]{EdTr96} by introducing the H\"older-Zygmund spaces by
\begin{align}\label{eq:zygmund-vs-besov}
\zygmund s\Rd := \besov s\infty\infty \Rd\, , \myqquad s\geq 0.
\end{align}
In the case $s\not\in \N_0$ and $s = k + r$ with $k\in \N_0$ and $r\in (0,1)$ these spaces consist of all 
$k$-times partially differentiable functions whose $k$-th derivatives are $r$-H\"older continuous, see 
\cite[Remark 2 in Chapter 2.2.3]{EdTr96},
while for $s\in \N$ this is not the case, see \cite[Remark 3 in Chapter 2.2.2]{Triebel83}.

\subsection{Restrictions}

Let us now consider open and bounded  $X\subset \Rd$. 
Following \cite[Chapter 3.2.2]{Triebel83} and \cite[Chapter 2.5.1]{EdTr96} we then define for 
   all allowed values of $s,p,q$ the spaces 
 $\besovtriebel spqX$ by restriction, that is 
\begin{align*}
\besovtriebel spqX & := \bigl\{ f_{|X}: f\in \besovtriebel spq\Rd    \bigr\} \, , \\
\snorm{g}_{\besovtriebel spqX} & := \inf\bigl\{  \snorm{f}_{\besovtriebel spq\Rd}: f\in \besovtriebel spq\Rd \mbox{ with } f_{|X} = g  \bigr\} \, . 
\end{align*} 
Analogously, we define the Sobolev spaces $\sobol mpX$, Bessel potential spaces $\bessel spX$, the
Slobedeckij spaces $\sobol spX$, and the Zygmund spaces $\zygmund sX$.

Note that   for some of the restricted spaces there is also a ``natural intrinsic''
definition for open and bounded $X\subset \Rd$, but in general these intrinsic definitions are not equivalent 
to our ``extrinsic''  definition by restriction. For a discussion on this topic we refer to \cite[Remarks 1--3 in Chapter 2.5.1]{EdTr96}
and to a more detailed discussion in \cite[Chapters 3.1.1 and 3.1.2]{Triebel83}. To add to those discussions
assume for moment that we have 
a BSF $F$ on $\Rd$ and an ``intrinsically defined'' BSF $F(X)$ on $X$ for which 
the restriction map $\mycdot_{|X}:F\to F(X)$ is well-defined and surjective. Then the BSFs $F_{|X} := \ran \mycdot_{|X}$
and $F(X)$ are equal as sets and thus have equivalent norms, see \cite[Lemmas 23 and 24]{ScSt25a}. In other words, 
if $f_{|X}\in F(X)$ for all $f\in F$ and 
all $g\in F(X)$ have an extension $\hat g\in F$, then the intrinsic and extrinsic definitions coincide up to equivalent norms.
In any case, to avoid potential technical problems, we 
 follow \cite{Triebel83,EdTr96} by additionally assuming that  $X$ has a  
 $C^\infty$-boundary in the sense of \cite[Chapter 3.2.1]{Triebel83}.
 
Now, it  
 is straightforward to verify that the embeddings
\eqref{eq:bf-space-mono-q}, \eqref{eq:bf-space-mono-s}, and \eqref{eq:bf-space-b-vs-f}
also hold if we replace $\Rd$ by  $X$, see also 
\cite[Chapter 3.2.4]{Triebel83}. For such $X$ we additionally have  an ordering in $p$, that is,
for all allowed values of $s,p_1,q$ and $s,p_2,q$ with $p_1\leq p_2$ we have 
\begin{align}\label{eq:bf-space-mono-p}
\besov s{p_2}q X \hookrightarrow \besov s{p_1}q X \myqquad \mbox{and} \myqquad \triebel s{p_2}q  X \hookrightarrow \triebel s{p_1}q  X \, , 
\end{align}
 see \cite[Chapter 3.3.1]{Triebel83}. Moreover, \eqref{eq:besov-triebel-s-vs-p}   holds for 
 all allowed values $s_1,p_1,q$ and $s_2,p_2,q$ with  $s_1 < s_2$ and  $s_1 - d/p_1 = s_2 - d/p_2$ and $\Rd$ replaced by $X$.
 Thanks to \eqref{eq:bf-space-mono-p}, however, even more is possible. Namely, 
for Besov spaces we have 
\begin{align}\label{eq:besov-s-vs-p-X}
\besov {s_2}{p_2}{q_2}X \hookrightarrow \besov {s_1}{p_1}{q_1} X
\end{align}
 for all  allowed values $s_1,p_1,q_1$ and $s_2,p_2,q_2$ with  $s_1 < s_2$ and  $s_1 - d/p_1 < s_2 - d/p_2$, and 
 for Triebel-Lizorkin spaces we  have 
\begin{align}\label{eq:triebel-s-vs-p-X}
\triebel {s_2}{p_2}qX \hookrightarrow \triebel {s_1}{p_1}1X 
 \end{align}
  for all allowed values $s_1,p_1,q$ and $s_2,p_2$ with  $s_1 < s_2$ and  $s_1 - d/p_1 \leq s_2 - d/p_2$.
For both embeddings we refer to \cite[Chapter 3.3.1]{Triebel83}.

 In addition,   \eqref{eq:b-space-vs-lp}, \eqref{eq:b0-space-vs-c}, 
 \eqref{eq:fp-as-lp}, \eqref{eq:bs-space-vs-c},  and \eqref{eq:fs-space-vs-c}   hold when $\Rd$ is replaced by $X$, see also the last paragraph on embeddings in \cite[Chapter 2.5.1]{EdTr96}.
In particular,  $\besovtriebel spqX$ is a BSF if   $\besovtriebel spq\Rd$ is one.
Moreover, the identification   \eqref{eq:sobol-is-triebel} of Sobolev and Bessel potential spaces,
 the embeddings \eqref{eq:besov-vs-sobol} for Sobolev spaces in the Besov scale, and the identification
 of Slobodeckij spaces \eqref{eq:slob-vs-bf} are also true for $X$ as all 
 involved spaces on $X$ are defined extrinsically by restriction.
 Finally, the identification \eqref{eq:zygmund-vs-besov} of Zygmund spaces also hold true for the extrinsically defined $\zygmund sX$,
where we note that 
 for $s\not \in \N_0$ we also have a natural intrinsic definition on $\overline X$
 in terms of the partial derivatives, see \cite[Chapter 2.5.2]{EdTr96}.

\subsection{Entropy Numbers} 
 Entropy numbers for embeddings  between Besov- and Triebel-Lizorkin spaces are well-understood. 
 To recall the corresponding results from \cite[Chapter 3.3.2]{EdTr96} we consider the 
 embeddings
 \begin{align}\label{eq:bfs-vs-bfs}
 \id: \besovtriebel {s_2}{p_2}{q_2}X    \to  \besovtriebel {s_1}{p_1}{q_1}X
 \end{align}
 where $X\subset \Rd$ is  open and bounded with $C^\infty$-boundary, and 
 $s_1,p_1,q_1$ and $s_2,p_2,q_2$ are allowed values with 
 \begin{align}\label{eq:embedding-assumptions}
 s_1 < s_2 \myqquad \mbox{ and } \myqquad s_2 - s_1  -  d \Bigl( \frac 1 {p_2} - \frac 1 {p_1} \Bigr)_+  >0\, . 
 \end{align}
 By \eqref{eq:besov-s-vs-p-X}, \eqref{eq:triebel-s-vs-p-X}, and the restricted versions of 
 \eqref{eq:bf-space-mono-q}  and 
 \eqref{eq:bf-space-b-vs-f}
 we then see that the embedding \eqref{eq:bfs-vs-bfs}  exists.
 Moreover,   it is shown in \cite[Chapter 3.3.2]{EdTr96} that there exists a constant $c\geq 1$ with 
 \begin{align}\label{eq:besovtriebel-entropy-general}
 e_i \bigl(\id: \besovtriebel {s_2}{p_2}{q_2}X    \to  \besovtriebel {s_1}{p_1}{q_1}X \bigr) \leq c i^{-(s_2-s_1)/d}
 \end{align}
 for all $i\geq 1$. 
 In fact, there is a matching lower bound proven in \cite[Chapter 3.3.3]{EdTr96}, but since
 we do not need it, we skip the details. Instead, the following lemma shows how 
 \eqref{eq:besovtriebel-entropy-general} can be used to establish entropy number bounds of the form 
 \eqref{eq:sup-entropy-numbers}.

  \begin{lemma}\label{lem:entropy-for-besov-triebel}
 Let $X\subset \Rd$ be open and bounded with $C^\infty$-boundary and 
 $F$ be a BSF on $X$ such that 
 $F\hookrightarrow \besovtriebel spq X$
 for some allowed values $s,p,q$ with $s>d/p$. We define 
 $\g := \frac d{2s}$. Then, 
 there exists a constant $a>0$ such that 
 \begin{align}\label{eq:general-entropy-by-bf}
e_i(\id:F\to \sLx \infty X )  
 \leq  a i^{-\frac 1 {2\g}}  \, , \myqquad i\geq 1. 
\end{align}
In particular, \eqref{eq:general-entropy-by-bf} holds in the following cases:
\begin{enumerate}
\item $F  \hookrightarrow \bessel spX $ with  $p\in (1,\infty)$  and $s> d/p$. 
\item $F\hookrightarrow \sobol spX$ with $p\in [1,\infty)$ and $s>d/p$.
\item $F\hookrightarrow \zygmund sX$ with $s>0$ and $s\not \in \N$.
\end{enumerate}
 \end{lemma}

 Before we prove Lemma \ref{lem:entropy-for-besov-triebel} 
 we note that the spaces in \emph{a)} to \emph{c)} can be 
 extrinsically or intrinsically defined as discussed above.
Moreover, if for $s\in \N$ we  consider the restricted version 
 of $\cmbuc s\Rd$, we again obtain \eqref{eq:general-entropy-by-bf} 
 for all $F \hookrightarrow \cmbuc sX$ by using  \eqref{eq:besov-vs-cm}.
 If, instead, we have $F \hookrightarrow C^s(X)$, where 
 $C^s(X)$ denotes the intrinsically defined  space of  all H\"older continuous functions 
 of order $s>0$, then \eqref{eq:general-entropy-by-bf} is 
 again true if $X$ is bounded and convex with nonempty interior,
 see e.g.~\cite[Theorem 2.7.1]{VaWe96}.
 We refer to  \cite[Appendix A.5.6]{StCh08} for some additional, related results and corresponding references.

 \begin{proof}[Proof of Lemma \ref{lem:entropy-for-besov-triebel}]
  Writing $s_2 := s$, $p_2:= p$, $q_2 := q$ as well as $s_1 := 0$, $p_1 := \infty$, $q_1:= 1$, we  
 obtain 
  the following commutative diagram of embeddings

 \quadiasn F{\sLx \infty X}{\besovtriebel spqX}{\besov 0\infty 1X}  {\id}{I_1}{I_3}{I_2}

where $I_1$ exists by assumption, $I_2$ exists as \eqref{eq:embedding-assumptions} is satisfied for the specified
indices, and $I_3$ exists as \eqref{eq:b0-space-vs-c} implies the embedding $\besov 0\infty 1\Rd \to \sLx \infty \Rd$,
which in turn can be restricted. As $I_1$ to $I_3$ are embeddings, their composition then gives the existence of the 
embedding $\id:F\to \sLx \infty X$.
By the multiplicativity of entropy numbers, see \cite[Chapter 1.3.1]{EdTr96}, 
and \eqref{eq:besovtriebel-entropy-general} we then conclude  with the help of the diagram that 
\begin{align*} 
e_i(\id:F\to \sLx \infty X ) 
\leq \snorm {I_3} \cdot  e_i(I_2:\besovtriebel spqX\to \besov 0\infty 1X )  \cdot  \snorm {I_1} 
 \leq  a i^{-s/d}  
\end{align*}
for all $i\geq 1$, where $a\geq 1$ is some constant independent of $i$. 
Using the definition of $\g$, we then find \eqref{eq:general-entropy-by-bf}.

 Finally,
 the special cases \emph{a)} to \emph{c)} directly follow 
 from \eqref{eq:general-entropy-by-bf} and the restricted versions of 
 \eqref{eq:sobol-is-triebel},  \eqref{eq:besov-vs-sobol},    \eqref{eq:slob-vs-bf}, and \eqref{eq:zygmund-vs-besov}.
 \end{proof}

\subsection{Embeddings}

 In this part we establish the embedding assumption \eqref{eq:embedding-by-norm} for Besov and Triebel-Lizorkin spaces. Here, we will be mostly 
 interested in the second inequality in  \eqref{eq:s-range-exact},
 as this 
 simplifies the rates in Theorem \ref{thm:learning_rates_tempered_C_minimal_general_form} for all $\b$ and $\vt$.
As one might also be interested in the first inequality in 
\eqref{eq:s-range-exact}, 
we present, however, the 
required techniques in some  detail.

 In view of the sufficient condition 
 \eqref{eq:interpol-embed} for \eqref{eq:embedding-by-norm}
 we begin by recalling some results on the interpolation spaces of 
 the Besov and Triebel-Lizorkin scales. To this end, let $\zeta\in (0,1)$
 and $\upsilon\in [1,\infty]$ and $E$ and $F$ be compatible 
 Banach spaces in the sense of interpolation space theory, see 
 e.g.~\cite{BeLo76,BeSh88}. We then denote the 
 real interpolation space between $E$ and $F$ by $[E,F]_{\zeta,\upsilon}$
 and the complex interpolation space by  $[E,F]_\zeta$.
 Recall from e.g.~\cite[Theorem 4.7.1]{BeLo76} that these spaces are nested 
 in the sense of 
 \begin{align}\label{eq:real-vs-complex-interpol}
 [E,F]_{\zeta,1}\hookrightarrow [E,F]_{\zeta}\hookrightarrow [E,F]_{\zeta,\infty}.
 \end{align}
 Moreover, if $E_0\hookrightarrow E$ and $F_0\hookrightarrow F$, then 
the definition of real interpolation spaces immediately ensures
 \begin{align}\label{eq:inklusions-by-inclusions}
 [E_0,F_0]_{\zeta,\upsilon} \hookrightarrow [E,F]_{\zeta,\upsilon}\, , 
 \end{align}
 see also \cite[Remark 4 in Chapter 2.4.1]{Triebel83}. Finally, we have
 $[E,F]_{\zeta,\upsilon_1} \hookrightarrow [E,F]_{\zeta,\upsilon_2}$
 whenever $\upsilon_1 \leq \upsilon_2$.
 In particular, for a given $\zeta$, the interpolation space considered in 
  \eqref{eq:interpol-embed} is the smallest possible in the scale of 
  real interpolation spaces between the involved spaces.
 
 Fortunately, the interpolation spaces within the Besov and Triebel-Lizorkin scales are well understood. To present the 
 corresponding results, we assume that $X\subset \Rd$ is open and bounded with $C^\infty$-boundary. For allowed values of  $s_1,p_1,q_1$ and  $s_2,p_2,q_2$ and $\zeta\in (0,1)$, 
 we  define 
 \begin{align*}
s^* = (1-\zeta)s_1 + \zeta s_2\, , \myqquad 
\frac 1{p^*} = \frac{1-\zeta}{p_1} + \frac \zeta {p_2}\, , \myqquad 
\frac 1{q^*} = \frac{1-\zeta}{q_1} + \frac \zeta {q_2} \, . 
\end{align*}
 By e.g.~\cite[Chapter 3.3.6]{Triebel83} the complex interpolation spaces 
 are then given by 
 \begin{align}\label{eq:interposl-b-complex}
 [\besov {s_1}{p_1}{q_1}X, \besov {s_2}{p_2}{q_2}X ]_\zeta  &= 
 \besov   {s^*}{p^*}{q^*}X\, , \\ \label{eq:interposl-f-complex}
 [\triebel {s_1}{p_1}{q_1}X, \triebel{s_2}{p_2}{q_2}X ]_\zeta & = 
 \triebel   {s^*}{p^*}{q^*}X\, , 
 \end{align}
 where for the Triebel-Lizorkin spaces we need to additionally assume 
 $q_1,q_2<\infty$.  Similarly, if 
 $s_1,p,q_1$ and  $s_2,p,q_2$ are allowed values with $s_1\neq s_2$
 and for $\zeta\in (0,1)$ and $\upsilon\in [1,\infty]$ we define 
 $s^*$ as above, then 
  \begin{align}\label{eq:interposl-b-real}
 [\besov {s_1}{p}{q_1}X, \besov {s_2}{p}{q_2}X ]_{\zeta,\upsilon}  &= 
 \besov   {s^*}{p}{\upsilon}X\, , \\ \label{eq:interposl-f-real}
 [\triebel {s_1}{p}{q_1}X, \triebel{s_2}{p}{q_2}X ]_{\zeta,\upsilon} & = 
 \triebel   {s^*}{p}{\upsilon}X\, .
 \end{align}
 With these preparations we can now establish  
 \eqref{eq:embedding-by-norm} for $\zeta$ satisfying the second inequality in 
 \eqref{eq:s-range-exact}
and both scales of spaces.

 \begin{lemma}\label{lem:embedding-for-besov-triebel}
 Let $X\subset \Rd$ be open and bounded with $C^\infty$-boundary, $r\geq 1$, and $\mu$ 
 be a probability measure on $X$ such that $\Lx r\mu \hookrightarrow \Lx r X$. Moreover, let $F$ be a BSF on $X$ such that 
 $F\hookrightarrow \besovtriebel spq X$
 for some allowed values $s,p,q$ with $s>d/p$ and $p\geq r$. We define 
 $\g := \frac d{2s}$ and $\zeta := \min\{1, \frac {2\g}r\}$. Then 
 there exists a constant $\cemb>0$ such that  
 \begin{align}\label{eq:embedding-by-norm-appendix}
\snorm f_{\sLx \infty X} \leq \cemb \snorm f_{\Lx r \mu}^{1-\zeta} \snorm f_F^\zeta\, , \myqquad f\in F.
\end{align}
In particular, this is true in the following cases:
\begin{enumerate}
\item $F  \hookrightarrow \bessel spX $ with  $p\in (1,\infty)$, $p\geq r$, and $s> d/p$. 
\item $F\hookrightarrow \sobol spX$ with $p\in [r,\infty)$ and $s>d/p$.
\item $F\hookrightarrow \zygmund sX$ with $s>0$, $s\not \in \N$, and $r\geq 1$.
\end{enumerate}
 \end{lemma}
 
 Before we prove this lemma, we note that its proof actually also gives
 hints how to find suitable values of $\zeta$ for 
 \eqref{eq:embedding-by-norm-appendix} to hold in the case $p<r$
 or $\Lx r\mu \hookrightarrow \Lx {r'} X$ for some $r'\neq r$.
 We skip the details, as we are mainly interested  
 in the second inequality in  \eqref{eq:s-range-exact}
 as mentioned at the beginning of this part.

 \begin{proof}[Proof of Lemma \ref{lem:embedding-for-besov-triebel}]
 We first note that $s>d/p$ ensures by the restricted versions of
 \eqref{eq:bs-space-vs-c}, respectively \eqref{eq:fs-space-vs-c} 
 and $\cbuc\Rd\hookrightarrow \sLx \infty \Rd$
 that 
 \begin{align*}
 F\hookrightarrow  \besovtriebel spq X \hookrightarrow \sLx \infty X\, . 
 \end{align*}
 In the case $\frac {2\g}r \geq 1$, which leads to 
 $\zeta = 1$, we may thus choose $\cemb := \snorm{\id:F\to \sLx \infty X}$.
 
 It therefore remains to consider the case $\frac {2\g}r < 1$, that is 
 $\zeta = \frac {2\g}r  < 1$ and $\zeta = \frac d {rs}$.
 Here, we will show that $[\Lx r \mu, F]_{\zeta,1}$ can be viewed 
 as a BSF with 
  \begin{align}\label{eq:interpol-embed-appendix}
[\Lx r\mu, F]_{\zeta,1}\hookrightarrow \sLx \infty X\, .
\end{align}
As discussed around \eqref{eq:interpol-embed} this embedding then implies 
\eqref{eq:embedding-by-norm-appendix}. For the proof of 
\eqref{eq:interpol-embed-appendix},
 we first note that 
our assumptions together with \eqref{eq:inklusions-by-inclusions}
give 
 \begin{align}\label{lem:embedding-forbesov-triebel-h1}
 [\Lx r X, F]_{\zeta,1}
 \hookrightarrow 
  [\Lx r X, \besovtriebel spqX]_{\zeta,1}
  \hookrightarrow
  [\besov 0r\infty X, \besovtriebel spqX]_{\zeta,1}\, , 
 \end{align}
 where in the second step we used the restricted version of \eqref{eq:b-space-vs-lp}.

 In the next part of the proof we consider the case 
  $\besovtriebel spqX = \besov spqX$
 with $p>r$.
 If we write $\frac 1{p^*} = \frac{1-\zeta}{r} + \frac \zeta {p}$, we then 
 obtain
 \begin{align}\label{lem:embedding-forbesov-triebel-h2}
  [\besov 0r\infty X, \besovtriebel spqX]_{\zeta,1}
  \hookrightarrow
   [\besov 0r\infty X, \besov spqX]_{\zeta}
   =
   \besov {\zeta s}{p^*}{q/\zeta}X
 \end{align}
 by the relationship between real and complex interpolation \eqref{eq:real-vs-complex-interpol} and the complex interpolation formula \eqref{eq:interposl-b-complex}
 for Besov spaces.
  Since we assumed $p>r$ we obviously have
 \begin{align*}
 \frac d {p^*}
 = d\Bigl( \frac{1-\zeta}{r} + \frac \zeta {p}  \Bigr) 
 < d\Bigl( \frac{1-\zeta}{r} + \frac \zeta {r}  \Bigr) 
 = \frac dr
 = \zeta s\, , 
 \end{align*}
 and therefore \eqref{eq:bs-space-vs-c} 
 shows 
 \begin{align*}
\besov {\zeta s}{p^*}{q/\zeta}\Rd \hookrightarrow \cbuc\Rd \hookrightarrow \sLx \infty \Rd\, . 
 \end{align*}
 Taking the restrictions onto $X$ on both sides and combining
 the result with \eqref{lem:embedding-forbesov-triebel-h1}
 and  \eqref{lem:embedding-forbesov-triebel-h2} we obtain
 \begin{align*}
  [\Lx r X, F]_{\zeta,1}
 \hookrightarrow  
 \besov {\zeta s}{p^*}{q/\zeta}\Rd
 \hookrightarrow
 \sLx \infty X \, . 
 \end{align*}
 In other words, we have found \eqref{eq:interpol-embed-appendix}.

  Let us now consider the case 
 $\besovtriebel spqX = \besov spqX$
 with $p=r$. 
By the real interpolation formula \eqref{eq:interposl-b-real}
for Besov spaces we then find 
 \begin{align}\label{lem:embedding-forbesov-triebel-h3}
   [\besov 0r\infty X, \besovtriebel spqX]_{\zeta,1}
   = 
     [\besov 0r\infty X, \besov srqX]_{\zeta,1}
  = \besov {\zeta s} r1X \, ,
 \end{align}
 see again \cite[Chapter 3.3.6]{Triebel83}. Since $\zeta s = d/r$
 the Sobolev embedding  
 \eqref{eq:bs-space-vs-c} gives 
 \begin{align*}
 \besov {\zeta s}{p^*}{q/\zeta}\Rd \hookrightarrow \cbuc\Rd \hookrightarrow \sLx \infty \Rd
 \end{align*}
 Taking the restrictions onto $X$ on both sides and combining
 the result with \eqref{lem:embedding-forbesov-triebel-h1}
 and  \eqref{lem:embedding-forbesov-triebel-h3} we again find  
 \eqref{eq:interpol-embed-appendix}.

Next, we consider the case  
 $\besovtriebel spqX = \triebel spqX$ for allowed
 values  $s,p,q$ with $s>d/p$ and $p\geq r$.
 In this case, the restricted version of \eqref{eq:bf-space-b-vs-f}
 yields 
 \begin{align*}
 F\hookrightarrow \triebel spqX \hookrightarrow \besov s p{\max\{p,q\}}X\, . 
 \end{align*}
 Applying the already shown assertion for the Besov scale then gives the assertion.
 
 Finally,
 the special cases \emph{a)} to \emph{c)} directly follow 
 from the already shown assertions and the restricted versions of 
 \eqref{eq:sobol-is-triebel},  \eqref{eq:besov-vs-sobol},    \eqref{eq:slob-vs-bf}, and \eqref{eq:zygmund-vs-besov}.
 \end{proof}

\section{Details on some Comparisons and Rates}
\label{app:E}

In this supplement, we provide some details on the comparisons 
to learning rates obtained earlier in the literature. 
Here we focus on those papers mentioned at the end of Section \ref{sec:statistical_analysis}
for which the comparison is not straightforward.
In addition, we provide some details on our rates for $F$ being a Besov space.

\subsection{Blanchard et al.~2003}\label{sec:comparison-blanchard}

In \cite{BlLuVa03a}, the authors consider
a form of structured risk minimization that is based upon 
IERM over scaled, absolutely convex hulls $\lb^{-1} \aco \ca C$ of a uniformly bounded base class $\ca C \subset \sLx \infty X$, where the invoked loss is margin-based and its  
  representing function $\p$ is
twice continuously differentiable, strictly increasing, and strictly convex.
Moreover, some additional technical assumptions 
on $\p$ are imposed, which ensure $\vt = 1$, see \cite[Lemma 19]{BlLuVa03a}.
In addition, they use 
an entropy number bound of the form 
\begin{align}\label{eq:entropy-blanchard}
\sup_{D\in (X\times Y)^n} e_i (\aco (\ca C), \Lx 2 D) \leq c i^{-\frac 1 {2\g}}\, , \myqquad  i\geq 1, n\geq 1,
\end{align}
where both $c$ and $\g$ are independent of $i$ and $n$, see \cite[Lemma 20 and the discussion afterwards]{BlLuVa03a}\footnote{The main results are actually formulated for $V:= \vcdim \ca C < \infty$, but we followed \cite[Remark 2]{BlLuVa03a} to replace this assumption by the weaker entropy number assumption with $\g = V/(V+2)$.}.
In addition, their approximation error assumption $A_\infty(\lb) \leq c \lb^{\rho}$ translates into $A_2(\lb) \leq \tilde c \lb^\b$ with $\b:= \rho / (\rho + 2)$, see e.g.~\cite[Exercise 5.11]{StCh08} for a related argument. Under the 
growth condition \eqref{eq:loss-growth-order-q} they then obtain the rate $n^{-\a}$ with 
\begin{align*}
\a = \frac{1}{1+\g} \cdot \frac{2\b}{\b(2-q) + q}
\end{align*}
for the \emph{unclipped} predictor. To compare this result to ours, we assume that $\ca C$ is countable\footnote{We believe that this assumption is actually also needed in \cite{BlLuVa03a} as otherwise the measurability of their algorithm may not be ensured.} and that 
  $(h_i)_{i\in \N}$ is an enumeration of it, that is  
$\ca C = \{h_i: i\in \N\}$. Furthermore, we consider the 
 vector space 
\begin{align}\label{eq:ell_1-F-space}
F:= \biggl\{ \sum_{i\in \N}   w_i h_i \, : (w_i) \in \ell_1      \biggr\}\, , 
\end{align}
and equip it with the norm 
\begin{align*}
\snorm f_F := \inf\biggl\{ \sum_{i\in \N} |w_i| :   (w_i)_{i\in I} \in \ell_1  \mbox{ with }  f   =\sum_{i\in \N} w_i h_i  \biggr\}\, .
\end{align*}
Then, it is straightforward to verify that $(F,\snorm \cdot_F)$ is a separable BSF, see e.g.~\cite[Lemma 2.1]{Steinwart09a}  satisfying $F\subset \sLx \infty X$. 
In the following we derive learning rates for RERM with respect to $F$, as this algorithm closely resembles the 
structured risk minimization approach of \cite{BlLuVa03a}.
To this end, we first note that 
our entropy number assumption  \eqref{concentration:sva-analysis-entropy} is weaker than \eqref{eq:entropy-blanchard}, and hence we may apply Theorem \ref{thm:learning_rates_tempered_C_minimal_general_form} with $\zeta=\vt = 1$ to the setting 
of \cite{BlLuVa03a}. This leads to 
the rate $n^{-\a}$ with 
\begin{align*}
\alpha \equalDef \min \left\{ \frac{2\beta}{\beta(2-q)+q }, \frac{\beta}{\beta + \gamma} \right\}
\end{align*}
For $q\geq 1$, $\b\in (0,1)$, and $\g\in (0,1)$ a routine calculation then shows that our rates are 
\emph{always} better than those in \cite{BlLuVa03a}. However, it is fair to mention that 
 \cite{BlLuVa03a} do not clip the predictors.  Moreover, in the case $\b=1$ the rates in  \cite{BlLuVa03a}  coincide 
 with ours and lead to rates for binary classification that match the minmax optimal rates, see 
 \cite[Corollary 9]{BlLuVa03a} and \cite{Yang99a}.

\subsection{Loustau 2009}\label{sec:Loustau}
 
 In \cite{Loustau09a} RERM over Besov spaces $\besov spq\Rd$ with $s>d/p$  and $p\in [1,2]$ is considered.
While the considered Besov spaces are defined on $\Rd$, they are implicitly restricted to compact $X\subset \Rd$,
as the marginals $\PX$ of the considered distributions $\P$ are assumed to have compact support $X$. In addition,
it is assumed that $\PX$ is Lebesgue absolutely continuous and the corresponding density is bounded away from 
$0$ and $\infty$. In particular, we have $\Lx p \PX = \Lx p X$ for all $p\in [1,\infty]$ with equivalent norms.

Now,  \cite{Loustau09a} considers binary classification with 
the hinge loss, and hence the  growth condition \eqref{eq:loss-growth-order-q} holds for $q=1$.
We keep this value in mind to avoid confusion with the fine index $q$ of Besov spaces.
  Moreover, it is assumed that 
  the distributions $\P$ satisfy a restricted version of Tsybakov's 
  noise condition, see \cite[(18)]{Loustau09a}, \firstt{namely
  \begin{align}\label{eq:tsybakov-restricted}
  \Bigl| \eta(x) - \frac 1 2   \Bigr| \geq \eta^* \myqquad \mbox{ and } \myqquad \min\bigl\{\eta(x), 1-\eta(x)\bigr\} \geq \eta^*
  \end{align}
  for all $x\in X$, where $\eta(x) := \P(Y=1|x)$ and $\eta^*\in (0,1]$ is some constant.
   It is easy to check that the first condition implies a variance bound \eqref{concentration:def-var-bound} with $\vt = 1$, while 
   the second condition essentially means, that the noise level of the labeling process is bounded away from zero. It is 
   easy to check that the latter implies \eqref{eq:loss-exact-growth-order-q-average} for $c:= \eta^*$.}
  Moreover, two regularization 
  exponents $p$ in \eqref{eq:rerm-opt-prob-eps} with $\epsilon = 0$ are considered in his rates 
  \cite[Corollary 1]{Loustau09a}, namely $p=1$ and $p=2$, but since his
  rates are better in the case $p=1$, see \cite[Remark 5]{Loustau09a}, we restrict our comparison to this case.
  This restriction has the additional benefit that we are not getting confused with the integration index $p$ of Besov spaces.

 Moreover, \cite{Loustau09a} controls the approximation error by assuming  $\fpb\in B_{p,\infty}^r(\Rd)$
 for some $r\in (0,s)$, where we note that the upper bound on $r$ is forgotten in the formulation of 
   \cite[Corollary 1]{Loustau09a} but used in its proof, see   \cite[Lemma 5]{Loustau09a}. Indeed,
   the assumption  $\fpb\in B_{p,\infty}^r(\Rd)$ is  used to show 
   \begin{align}\label{eq:A1-loustau}
   A_1(\lb) \leq c \lb^{r/s}\, , \myqquad \lb>0
   \end{align}
 for some constant $c$ independent of $\lb$, see the second to  last estimate on \cite[page 848]{Loustau09a},
 and \eqref{eq:A1-loustau}  is then used in the proof of  \cite[Corollary 1]{Loustau09a}.
To translate  \eqref{eq:A1-loustau} into our language we recall that   \eqref{eq:A1-loustau} is equivalent to 
 $A_2(\lb) \leq \tilde c \lb^\b$ with $r/s= 2\b / (1+\b)$, see Lemma \ref{lem:reg-trafo}.
 For the space $\besov spq\Rd$,
 the rates $n^{-\a}$ from \cite[(21) in  Corollary 1]{Loustau09a} have then the exponent
 \begin{align}\label{eq:rate-loustau}
 \a = \frac {2\b}{1+\b} \cdot   \frac {2u_p}{2u_p+d}   \, , 
 \end{align}
 where $p\in [1,2]$ and 
 \begin{align*}
  u_p := s + d \Bigl( \frac 12 - \frac 1p\Bigr) \, .
  \end{align*}

To compare these findings to the rates obtained by 
Theorem \ref{thm:learning_rates_tempered_C_minimal_general_form}, we first recall from Lemma \ref{lem:entropy-for-besov-triebel}
that for $F:= \besov spq X$ we have the entropy number bound 
\eqref{eq:sup-entropy-numbers} for $\g := \frac d {2s}$. Consequently, 
\eqref{concentration:sva-analysis-entropy}
in Theorem \ref{thm:learning_rates_tempered_C_minimal_general_form} 
also holds for this $\g$ if $s>d/2$.
Moreover, Lemma \ref{lem:embedding-for-besov-triebel} 
shows that the embedding assumption \eqref{eq:embedding-by-norm}
is satisfied for $\zeta = \min\{1, 2\g\}$ and all $p,q\in [1,\infty]$ and $s>d/p$. Consequently, for the hinge loss \firstt{and $\P$ satisfying \eqref{eq:tsybakov-restricted} we have 
$\vt = 1$ and \eqref{eq:loss-exact-growth-order-q-average} with $q=1$,} and therefore the rates of 
Theorem \ref{thm:learning_rates_tempered_C_minimal_general_form} 
have the exponent 
  $$
 \a = \frac{\beta}{\beta + \gamma} 
 $$
for \emph{all} Besov spaces $\besov spqX$ with $s>\max\{d/p, d/2\}$ and all 
$\fpb\in B_{p,\infty}^r(\Rd)$. 
In particular for the range $p\in [1,2]$ considered in \cite{Loustau09a}
this is true for all $s>d/p$.

To finish the comparison we note that 
$p\mapsto u_p$ is increasing, and thus so is  $p\mapsto \frac {2u_p}{2u_p+d}$. Consequently, for the expression of  $\a$ in \eqref{eq:rate-loustau} we find 
\begin{align*}
\frac {2\b}{1+\b} \cdot   \frac {2u_p}{2u_p+d}  
\leq 
\frac {2\b}{1+\b} \cdot   \frac {2u_2}{2u_2+d}
= \frac {2\b}{1+\b} \cdot   \frac {2s}{2s+d} = \frac {2\b}{1+\b} \cdot \frac 1 {1+\g}
\leq  \frac{\beta}{\beta + \gamma}\, , 
\end{align*}
where in the first estimate we used the constraint $p\leq 2$ 
in \cite{Loustau09a} and in the last step we used 
$1-\b-\g+\b\g \geq 0$ for all $\b\in (0,1]$ and $\g\in (0,1)$.
Moreover, the above inequality is strict unless $\b=1$ and $p=2$. 
Consequently, in the latter case the rates of 
\cite{Loustau09a}  coincide with those of 
Theorem \ref{thm:learning_rates_tempered_C_minimal_general_form} 
while for all other cases, our rates are faster.
 In particular, the conjectured optimality of the rates of \cite[Corollary 1]{Loustau09a}, see \cite[Section 3.3]{Loustau09a},
 is not true.

\subsection{Own Rates for Besov Spaces}\label{sec:besov-rates}

 Let us now derive in detail, how  Theorem \ref{thm:learning_rates_tempered_C_minimal_general_form} provides the rates for Besov spaces mentioned at the end of Section 
 \ref{sec:statistical_analysis}. To this end, we consider $F:= \besov spq X$ with $p,q\in [1,\infty]$ and $s > d/p$, where the latter inequality ensures 
by the restricted version of \eqref{eq:bs-space-vs-c} that $F$ is a BSF. In addition, we assume that $X$ has a $C^\infty$-boundary
and in view of  Lemma \ref{lem:embedding-for-besov-triebel}
we further assume that $\PX$ has a Lebesgue density, which is bounded away from $0$ and $\infty$. Thus we have   $\Lx p \PX = \Lx p X$ with equivalent norms.
By Lemma \ref{lem:entropy-for-besov-triebel}    and   Lemma  \ref{lem:embedding-for-besov-triebel}
 we then know 
 \begin{align}\label{eq:besov-gamma-and-zeta}
 \g = \frac d {2s} \myqquad \mbox{ and } \myqquad \zeta  = \min\Bigl\{1, \frac {2\g}p\Bigl\} = \min\Bigl\{1, \frac {2\g}p\Bigl\} = \frac d {sp}\, , 
 \end{align}
 where in the last  identity we used $s>d/p$. As Theorem \ref{thm:learning_rates_tempered_C_minimal_general_form} requires $\g\in (0,1)$, we thus have to assume 
 \begin{align}\label{besov-rates-s-cond}
 s > \max\Bigl\{ \frac d2 , \frac dp  \Bigr\}\, . 
 \end{align}
 
 Let us first assume that $L$ is the least squares loss. Here we pick an $r\in (0,s]$ and 
 define $\b := r/s$ as well as $s^* := (1-\b) \cdot 0 + \b s = r$ and 
 \begin{align*}
\frac 1{p^*} := \frac{1-\beta}{2} + \frac \beta {p} = \frac{ps - pr + 2r}{2ps} \myqquad \mbox{ and } \qquad 
\frac 1{q^*} := \frac{1-\beta}{1} + \frac \beta {q} = \frac {qs - qr + r }{qs} \, . 
 \end{align*}
 Let us now assume that $\fpb \in \besov r {p^*}{q^*} X = \besov {s^*} {p^*}{q^*}X$. If $r=s$, we then have $\b=1$ and $ \besov r {p^*}{q^*} X =  F$,
 and consequently Assumption \assrefshort{approx-error} is indeed satisfied for $\b=1$.
 Moreover, in the case $r<s$ we have $\b< 1$, and by 
 using \eqref{eq:interposl-b-complex}, \eqref{eq:real-vs-complex-interpol}, and \eqref{eq:b-space-vs-lp} %
 we thus find 
 \begin{align*}
 \fpb 
 \in \besov r {p^*}{q^*} X
 = \bigl[\besov 021X, \besov spqX\bigr]_{r/s}
 &\hookrightarrow \bigl[\besov 021X, \besov spqX\bigr]_{r/s,\infty} \\
 &\hookrightarrow  \bigl[\Lx 2 X, \besov spqX\bigr]_{r/s,\infty} \\
 &= \bigl[\Lx 2 \PX, F\bigr]_{\b,\infty}\, . 
 \end{align*}
 As the $K$-functional defining the latter interpolation space is equivalent to $A_2(\mycdot)$, see e.g.~\cite{SmZh03a} where this observation was made in the RKHS case, we then 
conclude that Assumption \assrefshort{approx-error} is indeed satisfied for $\b$.
 Theorem  \ref{thm:learning_rates_tempered_C_minimal_general_form} then establishes rates $n^{-\a}$ with 
 \begin{align*}
  \a = \min \left\{ \frac{\beta}{\beta(1-\zeta)+ \zeta  }, \frac{\beta}{\beta + \gamma} \right\}
   = \min \left\{ \frac{rsp}{r(sp-d)  + sd   }, \frac{2r}{2r + d} \right\}\, , 
 \end{align*}
 where in the first step we used both the quadratic growth behavior of the least squares loss and the automatically satisfied $\vt = 1$.
 Now in the case $p\geq 2$, Condition \eqref{besov-rates-s-cond} reduces to $s>d/2$, 
 and $\a$ reduces to 
 \begin{align}\label{besov-rates-optim-exp}
 \a = \frac{2r}{2r + d} \, . 
 \end{align}
These  rates match the minmax optimal rates, see e.g.~\cite[Theorem 2.2]{Temlyakov04a} for   general $d$ and $r> d/p^*$, that is 
$r > \frac {dps}{2ps + dp -2d}$ or \cite[Theorem 3.1]{Chesneau07a} for $d=1$ and general $r$. 
Moreover, in the case $p\in [1,2)$, Condition \eqref{besov-rates-s-cond} reduces to $s>d/p$, 
and using \eqref{eq:general-alpha-in-strong-rates}  we obtain \eqref{besov-rates-optim-exp} 
whenever
\begin{align}\label{besov-rates-r-range}
r\geq \frac {2-p}2 \cdot s\, . 
\end{align}

Besides the least squares loss, Theorem \ref{thm:learning_rates_tempered_C_minimal_general_form} also applies to other loss functions. 
To illustrate the consequences in the Besov space case, we now consider Lipschitz continuous $L$, that is $|L|_1 := \sup_{a\geq 0} |L|_{a,1} < \infty$,
with asymptotically exact growth behavior for $q=1$ in the sense of 
\eqref{eq:loss-growth-order-q} and \eqref{eq:loss-exact-growth-order-q-average}. 
In addition, we assume that $\fpb$ exists.
Examples of such $L$ include the pinball loss used for quantile regression, and the hinge loss \firstt{provided 
that $\P$ satisfies the second assumption in  \eqref{eq:tsybakov-restricted}.}
 Let us now assume that $F:= \besov spq X$ with $s$ satisfying \eqref{besov-rates-s-cond} and $p,q\in [1,\infty]$.
Moreover, we pick an $r\in (0,s]$ and 
 define $\theta := r/s$ as well as $s^* := (1-\theta) \cdot 0 + \theta s = r$ and 
 \begin{align*}
\frac 1{p^*} :=  \frac{1-\theta}{1} + \frac \theta {p} = \frac {ps - pr + r }{ps}\myqquad \mbox{ and } \qquad 
\frac 1{q^*} := \frac{1-\theta}{1} + \frac \theta {q} = \frac {qs - qr + r }{qs} \, .
 \end{align*}
 In the following we assume $\fpb \in \besov r {p^*}{q^*} X$, where we note that $p^*$ is in general \emph{not} equal to the value we used for the least squares loss. 
 In the case $r=s$, we then have $\theta=1$ and $ \besov r {p^*}{q^*} X =  F$,
 and consequently Assumption \assrefshort{approx-error} is   satisfied for $\b=1$.
 Moreover, in the case $r<s$ we have $\theta< 1$, and by 
 using \eqref{eq:interposl-b-complex}, \eqref{eq:real-vs-complex-interpol}, and \eqref{eq:b-space-vs-lp} %
 we thus find 
 \begin{align*}
 \fpb 
 \in \besov r {p^*}{q^*} X
 = \bigl[\besov 011X, \besov spqX\bigr]_{r/s}
 &\hookrightarrow \bigl[\besov 011X, \besov spqX\bigr]_{r/s,\infty} \\
 &\hookrightarrow  \bigl[\Lx 1 X, \besov spqX\bigr]_{r/s,\infty} \\
 &= \bigl[\Lx 1 \PX, F\bigr]_{\theta,\infty}\, . 
 \end{align*}
 Moreover, a quick calculation shows
\begin{align}\label{eq:pinball-lipschitz}
A_1(\lb) = 
\inf_{f\in F} \bigl(\lb \snorm f_F + \RP Lf - \RPB L\bigr)
&\leq \inf_{f\in F} \bigl(\lb \snorm f_F + |L|_1 \snorm{f-\fpb}_{\Lx 1 \PX}\bigr) \\ \nonumber 
&\leq |L|_1 \cdot K(\fpb, \lb, \Lx 1 \PX, F) \\ \nonumber
&\leq |L|_1 \cdot \snorm \fpb_{[\Lx 1 \PX, F]_{\theta,\infty}}  \lb^\theta
\end{align}
 where the first step uses the Lipschitz continuity, see e.g.~\cite[Lemma 2.19]{StCh08},  $K$ denotes the so-called $K$-functional of the real interpolation method,
 and the last step used the observed $ \fpb \in [\Lx 1 \PX, F]_{\theta,\infty}$ together with the definition of this interpolation space in terms of the 
 behavior of the $K$-functional, see e.g.~\cite[Definition 1.7 on page 299]{BeSh88}.
 By Lemma \ref{lem:reg-trafo} we thus find $\b = \theta/(2-\theta) = \frac r {2s-r}$. Moreover a quick calculation shows that $\g$ and $\zeta$ given by 
 \eqref{eq:besov-gamma-and-zeta} always satisfy $\zeta \leq 2\g$, and therefore the rates obtained by Theorem \ref{thm:learning_rates_tempered_C_minimal_general_form}
 have the exponent 
 \begin{align*}
 \a =    \frac{\beta}{\beta(2-\gamma-\vartheta+\gamma\vartheta)+ \gamma}
 &=   \frac{r}{r(2-\gamma-\vartheta+\gamma\vartheta)+ \gamma(2s-r)} \\
 & =   \frac{r}{r(2 -\vartheta)+ \gamma(2s-2r+ \vt r)} \\
 &=  \frac{2sr}{2sr(2 -\vartheta)+ (2s-2r+ \vt r)d}\, .
 \end{align*}
 To illustrate these rates, we assume that $L$ is the $\t$-pinball loss for some $\t\in (0,1)$  and that the distribution $\P$ on $X\times [-1,1]$ has 
 $\t$-quantiles of $\infty$-average type $2$ in the sense of \cite[Definition 2.6]{StCh11a}. The latter is, for example, satisfied if 
 each conditional distribution $\P(\mycdot|x)$ has a Lebesgue density that is bounded away from $0$ and $\infty$ with a constant independent of $x$.
Then \cite[Theorem 2.8]{StCh11a} gives a variance bound for $\vt = 1$, and the expression for $\a$ reduces to 
\begin{align*}
 \a =      \frac{2sr}{2sr + (2s-r)d} =  \frac{2r}{2r + d + (1-r/s)d}\, .
\end{align*}
Moreover, for such $\P$, the calibration inequality of  \cite[Theorem 2.7]{StCh11a}
shows 
\begin{align}\label{eq:pinpabll-calibration}
\snorm{f-\fpb}_{\Lx 2 \PX}^2 \leq c \bigl( \RP Lf - \RPB L \bigr) 
\end{align}
for all measurable $f:X\to [-1,1]$ and a constant $c$ independent of $f$. In the case $r=s$, we thus obtain a rate that matches 
the one we found for least squares regression, while for $r<s$ the obtained rates depend on both the smoothness $r$ of the target function $\fpb$
and the smoothness $s$ of the used space $F=\besov spq X$. 
In particular, they are not minmax optimal. Our conjecture is that this looseness comes from the mismatch of the estimates in \eqref{eq:pinball-lipschitz} and 
\eqref{eq:pinpabll-calibration}. Indeed, together they read as 
\begin{align*}
c^{-1} \snorm{f-\fpb}_{\Lx 2 \PX}^2 \leq  \RP Lf - \RPB L  \leq  |L|_1 \cdot  \snorm{f-\fpb}_{\Lx 1 \PX}\, , 
\end{align*}
and the left and right sides behave asymptotically different for $f$ approaching $\fpb$. 
Moreover, this difference has an impact on the rates. To illustrate this, assume that, 
under potentially additional assumptions on $\P(\cdot|x)$, we have
\begin{align*}
 \RP Lf - \RPB L \leq C \snorm{f-\fpb}_{\Lx 2 \PX}^2 \, .
\end{align*}
Then the arguments used in the least squares case can be literally copied resulting in $\b = r/s$ for
Assumption \assrefshort{approx-error} rather than the current $\b = \theta/(2-\theta) = \frac r {2s-r}$. This in turn would give rates matching the minmax ones.

\firstt{

\subsection{Supportive material for mirror descent}\label{sec:appendix--mirror-descent}
Begin with some standard definitions.
Let $F$ be an RKBS, $f,g,\psi:F\to\R$ be functions and $\psi:F\to \R$ be furthermore Fr\'echet-differentiable.
The function $f$ is called \emph{convex} if 
\begin{align*}
    f(sa + (1-s) b) \le s f(a) +(1-s)f(b)
\end{align*}
holds for all $a,b\in F, s\in [0,1]$. 
If the inequality is strict whenever $a\ne b$ and $s\in(0,1)$, then $f$ is called \emph{strictly convex}.
If $f,g$ are convex and $Mg- f$ is convex for $M\ge 0$, then $f$ is called $M$-smooth relative to $g$. 

Following \cite[Chapter~1]{ButnariuIusem2000}, we define the Bregman-divergence $D_\psi :F\times F\to \R$ as 
\begin{align*}
    \BM \psi fg \equalDef \psi(f) - \psi(g) - \scal{\nabla \psi(g)}{f-g}
\end{align*}
and collect some elementary properties. 
Firstly, the Bregman divergence is linear in $\psi$. 
Most importantly, $\psi$ is convex if and only if
\begin{align*}
    \BM \psi fg \ge 0 \qquad \forall f,g\in F~.
\end{align*}
Strict convexity of $\psi$ is characterized by strict inequality whenever $f\ne g$. \ms{we still need to include your comment about kurs auf M-smooth finden from the corrections - skipped for now.}
Noteworthy algebraic identities are the \emph{three point identity}
        \begin{align}
            \label{eq:bregman-three-point}
            \BM{\psi} a c = \BM \psi ab + \BM \psi bc + \scal{\nabla \psi(b) -\nabla \psi(c)}{a-b}\qquad \forall a,b,c\in F~.
        \end{align}
        and the \emph{symmetrization identity}
    \begin{align}
        \label{eq:bregman-symmetrization}
        \BM \psi fg + \BM \psi gf = \scal{\nabla \psi(f)-\nabla \psi(g) }{f-g}\qquad \forall f,g\in F~.
    \end{align}
}
\secondd{
Finally, note that if a loss function $\loss$ is differentiable, then it is $M$-smooth, cf.~\Eqref{eq:loss-m-smooth}, if and only if $\loss$ is $M$-smooth relative to $(\R\ni r\mapsto \abs{r}^2/2)$. So, our double use of notation is indeed justified.
}
\firstt{
The following lemma clarifies the connection between $\MMbeta$-smoothness of the empirical risk $\riskk\cdot $ relative to a mirror map $\psi$, the key ingredient in \cref{lem:mirror-descent-fejer}. Furthermore, it extends to Hilbert space settings by choosing $\psi(f)\equalDef \norm{f}^2/2$, yielding the constant $d_\psi=1$ in \cref{prop:m-smoothness-of-risks}.

\begin{lemma}\label{prop:m-smoothness-of-risks}
    Let $F$ be an RKBS fulfilling $F\hookrightarrow \ell_{\infty}(X)$ and let  $\psi:F\to\R$ be convex and Fr\'echet differentiable with $d_\psi>0$ such that all $f,g\in F$ satisfy
    \begin{align*}
        \BM \psi fg \ge \frac {d_\psi}2 \norm{f-g}_F^2~.
    \end{align*}
    Let $\loss$ be  an $\Mbeta$-smooth loss function, cf.~\Eqref{eq:loss-m-smooth}.
    Then, $\riskk\cdot$ is $\MMbeta$-smooth relative to $\psi$ 
    with 
    \begin{align*}
        \MMbeta\equalDef \Mbeta d_\psi^{-1} \norm{F\hookrightarrow \ell_\infty(X)}^2 
    \end{align*}
    for all data sets $D$.
\end{lemma}
\begin{proof}
    Let $D$ be a data set. We need to show $\MMbeta \BM \psi fg  \ge  \BM {\riskk\cdot}fg$ for all $f,g\in F$.

    Firstly, note that we have
    \begin{align*}
        \norm{ \nabla \calR_{\loss,D} (f) - \nabla \calR_{\loss,D}(g)}_{F'}
        &= \norm{\frac 1n \sum_{i=1}^n (\loss ' (x_i,y_i,f(x_i))- \loss ' (x_i,y_i,g(x_i))) \delta_{x_i}}_{F'}
        \\ &\le  \frac 1n \sum_{i=1}^n \underbrace{\abs{(\loss ' (x_i,y_i,\delta_{x_i}(f))- \loss ' (x_i,y_i,\delta_{x_i}(g)))}}_{\le M \abs{\delta_{x_i}(f)-\delta_{x_i}(g)}}\norm{ \delta_{x_i}}_{F'}
        \\ & \le \Mbeta\norm{F \hookrightarrow \calL_\infty(X)}^2 \norm{f-g}_F~, 
    \end{align*}
    since trivially $\norm{\delta_{x_i}}_{F'}\le \norm{F \hookrightarrow \ell_\infty(X)}$.

    By Fr\'echet differentiability of $\riskk\cdot$ we have
    \begin{align*}
        \riskk{f} = \riskk{g} +\int_{0}^1 \scal{\nabla\riskk {s(f-g)+g}}{f-g}_{F',F} \d s
    \end{align*}
    which yields
    \begin{align*}
        \BM {\riskk\cdot} {f}{g} &= \int_{0}^1 \scal{\nabla\riskk {s(f-g) +g}}{f-g}_{F',F} \d s - \scal{\nabla\riskk{g}}{f-g}_{F',F}
        \\ &\le 
        \norm{f-g}_F \int_0^1 \norm{\nabla \riskk{s(f-g) + g} -\nabla\riskk g}_{F'} \d s
        \\ &\le \norm{f-g}_F \int_0^1 \Mbeta \norm{F\hookrightarrow \ell_\infty(X)}^2 s \norm{f-g}_F   \d s
        \\ &= \frac 12  \Mbeta \norm{F\hookrightarrow \ell_\infty(X)}^2 \norm{f-g}_F^2 
        \\ &\le d_\psi^{-1} \Mbeta \norm{F\hookrightarrow \ell_\infty(X)}^2  \BM\psi fg~.
    \end{align*}
\end{proof}

}

\firstt{
    Following common notation, we call a Banach space $F$ \emph{Fréchet-smooth} if $f\mapsto \norm f_F$ is Fréchet differentiable for all $f\ne 0$, and we call $F$ \emph{strictly convex} if for all $f\ne g\in F$ with $\norm f_F =\norm g_F =1$ and all $\lambda\in (0,1)$ we have
    \begin{align*}
        \norm{\lambda x + (1-\lambda) y } <1~.
    \end{align*}
\begin{lemma}\label{lem:checking-mirror-map-requirements}
    Let $F$  be reflexive and let $F,F'$ be Fréchet-smooth. 
    Let $A:F\to F$ be a linear isomorphism.
    Define $$\psi(f)\equalDef \norm{Af}^2/2~.$$ 
        Then, $\psi$ is  a mirror map.
\end{lemma}
\begin{proof}
    Clearly, $\psi(0)=0$ and $\nabla\psi(0)=0$ hold. By \cite[Theorem 8]{dragomir2004semi}, Fréchet-smoothness and reflexivity of $F$ and $F'$ imply strict convexity of $F'$ and $F$. Strict convexity of $\psi$ follows, as $A$ is an isomorphism.

    For $f\in F$, consider the following equations for $\xi\in F'$:
    \begin{align}
        \label{eq:normalized-dual-mapping-1}
        \scal{\xi}{f}_{F',F} &= \norm{\xi}_{F'} \norm{f}_F~,\\ 
        \label{eq:normalized-dual-mapping-2}
        \norm{\xi}_{F'}&=\norm{f}_{F}~.
    \end{align}
    \cite[Theorem 1]{dragomir2004semi} shows that the smoothness of $F$ ensures that the solution $\xi$ of the above equations exists uniquely, and $J_F(f)\equalDef \xi$ is called the \emph{normalized duality mapping} $J_F:F\to F'$, cf.\ \cite[Definition 1]{dragomir2004semi}.
    \cite[Theorems 2 and 4]{dragomir2004semi} show that by reflexivity and strict convexity of $F$, $J_F$ is a bijection. 
    Likewise, our assumptions ensure that $J_{F'}$ has analogous properties.

    Now, we observe that by \cite[Prop.\ 8.7]{defigueiredo1989ekeland} the Fr\'echet-differential of $f\mapsto \norm{f}_F^2/2$ equals $J_F$. Analogously, $J_{F'}$ equals the Fréchet-differential of $F'\ni \xi \mapsto \norm{\xi}^2_{F'}/2$. Furthermore, both $J_F$ and $J_{F'}$ are continuous by \cite[Prop.\ 8.5]{defigueiredo1989ekeland}.
    Let $A'$ denote the adjoint operator of $A$. The chain rule yields 
    \begin{align}
        \label{eq:psi-chain-rule}
        \nabla \psi(f) 
        = A' J_F(A f)~,
    \end{align}
    that is, $\psi$ is continuously Fréchet-differentiable. Furthermore,  $\nabla\psi$ is bijective since $A$ is an isomorphism.

It remains to show that the inverse $(\nabla \psi)^{-1}$ is continuous.
To this end, let $\iota:F\to F''$ be the canonical isometric isomorphism.
For $\xi=J_F(f)$, the defining identities \eqref{eq:normalized-dual-mapping-1}, \eqref{eq:normalized-dual-mapping-2}  of $J_F$ give
\[
    \|\iota f\|_{F''}=\|f\|_F=\|\xi\|_{F'},
    \qquad
    \langle \iota f,\xi\rangle_{F'',F'}
    =\langle \xi,f\rangle_{F',F}
    =\norm{\xi}_{F'}\norm{f}_F 
    =\norm {\iota f}_{F''} \norm{\xi}_{F'}~.
\]
We obtain $\iota f = J_{F'}(\xi)$, which yields
\[
    J_{F'}\circ J_F=\iota,
    \qquad
    J_F^{-1}=\iota^{-1}\circ J_{F'}~,
\]
and from the continuity of $J_{F'}$ we see that $J_F^{-1}$ is continuous.
Now, from 
\eqref{eq:psi-chain-rule} 
we obtain 
\begin{align*}
    (\nabla\psi)^{-1} = A^{-1} J_F^{-1} (A')^{-1}~,
\end{align*}
which is continuous as a composition of continuous maps.
\end{proof}
}
\end{appendix}

\end{document}